\documentclass[11pt]{article}

\usepackage{graphicx, amsmath, fullpage, amssymb, amsthm, amsfonts, color, algorithm,mathtools}
\usepackage{algorithmicx, algpseudocode,float}

\usepackage{enumerate}

\newtheorem{assumption}{Assumption}

\numberwithin{equation}{section}
\theoremstyle{definition}

\usepackage[toc,page]{appendix} 
\usepackage[svgnames]{xcolor}
\theoremstyle{remark}
\usepackage{mathtools}
\usepackage{parskip}

\numberwithin{equation}{section}
\usepackage[colorlinks,
            linkcolor=red,
            anchorcolor=red,
            citecolor=blue
            ]{hyperref}
 
\newcommand{\R}{\mathbb R}

\newcommand{\1}{\mathbf 1}

\newcommand{\bQ}{\mathbf Q}
\newcommand{\bK}{\mathbf K}
\newcommand{\bD}{\mathbf D}
\newcommand{\bC}{\mathbf C}

\algrenewcommand\algorithmicrequire{\textbf{Input:}}
\algrenewcommand\algorithmicensure{\textbf{Output:}}

\usepackage{mathtools}
\usepackage{graphicx}
\usepackage{etoolbox}
\usepackage{amsmath, amssymb, amsthm}

\newcommand{\dR}{\mathbb{R}}

\newcommand{\dQ}{\mathbb{Q}}

\newcommand{\dE}{\mathbb{E}}
\newcommand{\dP}{\mathbb{P}}

\newcommand{\cA}{\mathcal{A}}

\newcommand{\cF}{\mathcal{F}}

\newcommand{\cH}{\mathcal{H}}

\newcommand{\cL}{\mathcal{L}}

\newcommand{\cN}{\mathcal{N}}

\newcommand{\cS}{\mathcal{S}}

\newcommand{\ind}{\mathbf{1}}

\providecommand{\given}{}
\DeclarePairedDelimiterXPP{\Pb}[1]{\mathbb{P}}{\lparen}{\rparen}{}{\renewcommand{\given}{\nonscript{}\:\delimsize\vert\nonscript{}\:\mathopen{}} #1}
\DeclarePairedDelimiterXPP{\E}[1]{\mathbb{E}}[]{}{\renewcommand{\given}{\nonscript{}\:\delimsize\vert\nonscript{}\:\mathopen{}} #1}
\DeclarePairedDelimiterX{\Set}[1]\lbrace\rbrace{\renewcommand{\given}{\nonscript{}\:\delimsize\vert\nonscript{}\:\mathopen{}} #1}

\DeclareMathOperator{\diag}{diag}
\DeclareMathOperator{\tr}{tr}
\DeclareMathOperator{\Var}{Var}
\DeclareMathOperator{\Cov}{Cov}
\DeclareMathOperator{\img}{im}

\DeclarePairedDelimiterX{\norm}[1]\lVert\rVert{\ifblank{#1}{\: \cdot \:}{#1}}

\DeclareMathOperator{\Poi}{Poi}
\DeclareMathOperator{\Bin}{Bin}

\newcommand{\bilingualcommand}[3]{%
	\newcommand{#1}[1][\ ]{%
		##1%
		\iflanguage{english}{\text{#2}}{%
			\iflanguage{french}{\text{#3}}{}%
		}%
		##1%
	}%
}

\bilingualcommand{\where}{where}{où}
\bilingualcommand{\textif}{if}{si}
\bilingualcommand{\textand}{and}{et}
\bilingualcommand{\textiff}{if and only if}{si et seulement si}
\bilingualcommand{\otherwise}{otherwise}{sinon}

\newcommand{\eps}{\varepsilon}

\newcommand{\quand}{\quad \textand{} \quad}

\theoremstyle{definition}

\newtheorem{theorem}{Theorem}
\newtheorem{lemma}{Lemma}
\newtheorem{proposition}{Proposition}
\newtheorem{definition}{Definition}

\newtheorem{remark}{Remark}
\newtheorem{example}{Example}

\begin{document}

\title{Achieving the Kesten–Stigum bound in the non-uniform hypergraph stochastic block model}



\author{ Manuel Fernandez V\thanks{Department of Mathematics, University of Southern California, manuelf7@usc.edu} \and Ludovic Stephan\thanks{Univ Rennes, Ensai, CNRS, CREST-UMR 9194, F-35000 Rennes, France,  ludovic.stephan@ensai.fr} \and Yizhe Zhu\thanks{Department of Mathematics, University of Southern California, yizhezhu@usc.edu}}

%
\date{}

\maketitle

    \begin{abstract}
We study the community detection problem in the non-uniform hypergraph stochastic block model (HSBM), where hyperedges of varying sizes coexist. This setting captures higher-order and multi-view interactions and raises a fundamental question: can multiple uniform hypergraph layers below the detection threshold be combined to enable weak recovery?
We answer this question by establishing a Kesten--Stigum-type bound for weak recovery in a general class of non-uniform HSBMs with $r$ blocks, generated according to multiple symmetric probability tensors. In the case $r=2$, we show that weak recovery is possible whenever the sum of the signal-to-noise ratios across all uniform hypergraph layers exceeds one, thereby confirming the positive part of a conjecture in \cite{chodrow2023nonbacktracking}. 
Moreover, we provide a polynomial-time spectral algorithm that achieves this threshold via an optimally weighted non-backtracking operator. For the unweighted non-backtracking matrix, our spectral method attains a different algorithmic threshold, also conjectured in \cite{chodrow2023nonbacktracking}.

Our approach develops a spectral theory for weighted non-backtracking operators on non-uniform hypergraphs, including a precise characterization of outlier eigenvalues and eigenvector overlaps. We introduce a novel Ihara--Bass formula tailored to weighted non-uniform hypergraphs, which yields an efficient low-dimensional representation and leads to a provable spectral reconstruction algorithm. 
Taken together, these results provide a principled and computationally efficient approach to clustering in non-uniform hypergraphs, and highlight the role of optimal weighting in aggregating heterogeneous higher-order interactions.
\end{abstract}

  \tableofcontents

\section{Introduction}

Community detection in network data is a central problem in statistics, probability, and machine learning. The stochastic block model (SBM) provides a canonical framework: vertices are partitioned into latent communities, and edges are generated according to their community memberships. Over the past decade, a sharp understanding of detection thresholds has emerged for sparse graph SBMs \cite{abbe2018community}. In particular, the Kesten--Stigum (KS) threshold \cite{kesten1966limit} characterizes the fundamental computational boundary for weak recovery \cite{decelle2011asymptotic,massoulie2014community,bordenave2018nonbacktracking,mossel2015reconstruction,mossel2018proof,abbe2018proof}.

Many real-world networks, however, involve higher-order interactions among more than two entities, such as co-authorship networks, biochemical systems, and group communications. These systems are more naturally modeled by hypergraphs \cite{zhou2007learning,abiad2026hypergraphs}, leading to the hypergraph stochastic block model (HSBM) \cite{ghoshdastidar2017consistency}. 
When all hyperedges have the same size \(q\), the \(q\)-uniform HSBM has been extensively studied. Existing results include exact recovery \cite{kim2018stochastic,cole2020exact,zhang2021exact,gaudio2022,bresler2024thresholds,ahn2018hypergraph,lee2020robust}, weak consistency and partial recovery \cite{ghoshdastidar2017consistency,ghoshdastidar2017uniform,dumitriu2025partial}, as well as spectral methods for weak recovery and hardness results below the Kesten--Stigum threshold \cite{Pal_2021,stephan2024sparse,gu2023weak,gu2024community,kunisky2024low,mossel2025weak}. 
Recent works have also explored higher-order spectral and message-passing methods for hypergraph community detection \cite{ruggeri2024message,li2026higher,schmidhuber2025quartic}. A common approach, both in theory and in practice, is to project hypergraphs onto weighted graphs \cite{kim2018stochastic,dumitriu2021spectra,cole2020exact,gaudio2022,alaluusua2023multilayer,dumitriu2025partial,morgan2026achievability}, thereby reducing higher-order interactions to pairwise ones; however, such projections may discard important structural information \cite{ke2019community,zhang2021exact,valimaa2025consistent,dumitriu2026optimal}.

In contrast, non-uniform HSBMs allow hyperedges of varying sizes, resulting in a richer and more complex structure. A natural perspective is to view a non-uniform HSBM as the union of multiple independent uniform hypergraphs \(G^{(1)},\dots,G^{(K)}\), all sharing the same vertex set and community assignment. This viewpoint raises a fundamental question: 

\begin{center}
 \textit{Can multiple hypergraph layers, each below the KS threshold, be combined to enable weak recovery?}   
\end{center}

Equivalently, how should one optimally aggregate information across hypergraph layers of different sizes? This question is closely related to multi-view clustering and correlated spiked models \cite{ma2023community,lei2023bias,yang2025fundamental,mergny2025spectral,gong2026fundamental,li2025algorithmic}, where a common latent structure is inferred from multiple noisy observations.
Despite recent progress on partial recovery, weak consistency, and exact recovery in non-uniform settings \cite{ghoshdastidar2017consistency,dumitriu2025partial,alaluusua2023multilayer,wang2023strong,dumitriu2026optimal}, the fundamental limits of detection remain less understood. In particular, \cite{chodrow2023nonbacktracking} conjectured that, in the binary case, a spectral method based on a belief-propagation Jacobian matrix achieves a Kesten--Stigum-type threshold determined by the combined signal-to-noise ratios across all hypergraph layers.

In this work, we resolve this conjecture and develop a more efficient spectral method based on a new weighted non-backtracking operator for non-uniform hypergraphs, which achieves the optimal detection threshold.  In the assortative case (all eigenvalues of the signal matrix are positive), without the knowledge of model parameters, an unweighted non-backtracking method can achieve weak recovery above another algorithmic threshold conjectured in \cite{chodrow2023nonbacktracking}.
See Section~\ref{sec:discuss} for further discussion.

\subsection{Our contribution}
Our contributions can be summarized as follows:
\begin{enumerate}
    \item We develop a spectral theory for weighted non-backtracking operators on non-uniform random hypergraphs, including bounds on their bulk eigenvalues, precise locations of outlier eigenvalues, and alignment of the corresponding leading eigenvectors with the planted community structure in general non-uniform HSBMs. We do not assume that the signal matrices of the individual layers \(G^{(k)}\) share a common eigenbasis, and our results apply to both assortative and disassortative settings. 
In particular, when some hypergraph layer \(G^{(k)}\) is disassortative, the corresponding optimal weight \(w^{(k)}\) may be negative, which introduces additional technical challenges in the analysis.
    
    \item We derive a new Ihara--Bass formula for weighted non-uniform hypergraphs (Proposition~\ref{thm:IharaBass}). This yields a reduced operator of dimension \(2Kn \times 2Kn\), leading to an efficient procedure for computing informative eigenvalues and constructing estimators of the community labels.
    
    \item In the two-block case, we reach the conjectured detection threshold for non-uniform HSBMs proposed in \cite{chodrow2023nonbacktracking}. More precisely, we show that weak recovery is possible whenever the combined signal-to-noise ratio across all layers exceeds the Kesten--Stigum threshold, and we provide a polynomial-time spectral algorithm that achieves this bound. 
Building on our new Ihara--Bass formula for weighted hypergraphs, we derive an efficient spectral method based on a reduced non-backtracking matrix of smaller dimension. In contrast, the belief-propagation Jacobian matrix in \cite{chodrow2023nonbacktracking} is derived using methods from statistical physics and operates on a larger space. 
We also establish the conjectured algorithmic threshold of \cite{chodrow2023nonbacktracking} for the unweighted non-backtracking operator; see Examples~\ref{example:weight_1} and \ref{example:weight_2} for more details.
\end{enumerate}

\subsection{Related work}

The non-uniform HSBM is closely connected to several lines of work in hypergraph clustering, multi-layer network models, and broader high-dimensional inference problems. We highlight some of these connections below.

\paragraph{Multi-layer and multi-view models}
The non-uniform HSBM can be naturally viewed as a multi-layer or multi-view model, where each layer corresponds to hyperedges of a fixed size. This perspective connects our setting to the literature on multilayer stochastic block models and multi-view clustering, where multiple noisy observations of a common latent partition are combined. Representative works include \cite{lei2023bias,gong2026fundamental,ma2023community,yang2025fundamental}. Related models include the censored SBM \cite{saade2015spectral}, the labeled SBM \cite{lelarge2015reconstruction,stephan2020non}, and SBMs with side information \cite{gaudio2024exact}, all of which incorporate additional sources of signal beyond a single network layer. When all \(G^{(1)},\dots,G^{(K)}\) are graph SBMs with a binary partition, \(G\) reduces to a special case of the multi-layer stochastic block model in which all layers share the same vertex labels. In this setting (corresponding to \(\rho=1,\mu=0\) in \cite{gong2026fundamental}), our spectral algorithm attains the information-theoretic lower bound established in \cite[Theorem 1.5]{gong2026fundamental}, as well as the detection threshold for labeled SBMs in \cite{lelarge2015reconstruction,stephan2020non}. We believe that an adaptation of our weighted non-backtracking operator could also attain the fundamental limit in the contextual multi-layer SBM \cite{yang2025fundamental,gong2026fundamental}.

\paragraph{Detection vs.\ estimation}
An important question in the non-uniform setting is whether combining multiple subcritical layers can enable recovery when each individual layer lies below its own detection threshold. This raises subtle issues of parameter estimation and identifiability. In particular, when all \(G^{(k)}\) are below the KS threshold, it may be impossible to consistently estimate the model parameters or to achieve non-trivial detection \cite{mossel2015reconstruction}. Such impossibility phenomena have been studied in related contexts \cite{mossel2025weak}. In our setting, finding the optimal weight for the non-backtracking spectral method requires knowledge of the model parameters of each $G^{(k)}$ is they are below the KS threshold. The assumption that these parameters are known is often referred to as the \textit{Nishimori condition} in statistical physics \cite{decelle2011asymptotic}.

\paragraph{Spectral methods with non-backtracking operators}
The Non-backtracking operators have proved to be a powerful tool for community detection in sparse graphs and hypergraphs \cite{bordenave2018nonbacktracking,stephan2020non,stephan2024sparse}, as well as in matrix and tensor completion problems \cite{bordenave2020detection,stephan2024non}. In the non-uniform hypergraph setting, \cite{chodrow2023nonbacktracking} considered an \emph{unweighted} non-backtracking spectral clustering method and conjectured a Kesten–Stigum-type detection threshold that is weaker than that of the belief-propagation Jacobian matrix. Our results also recover this conjectured threshold when all weights are taken to be equal (see \cite[Conjecture 4.3]{chodrow2023nonbacktracking}), and further provide a refined spectral analysis together with a dimension-reduction procedure beyond the two-block setting.

\paragraph{Spiked and factor models}
There are also conceptual connections between non-uniform HSBMs and spiked random matrix and tensor models. In particular, combining multiple hyperedge sizes can be viewed as aggregating multiple noisy observations of a low-rank signal, analogous to recent work on correlated spiked Wigner and Wishart models \cite{keup2025optimal,mergny2025spectral,li2025algorithmic,hong2023optimally}. More broadly, our sparse random hypergraph setting is related to general factor models and inference on trees, where signal propagation and reconstruction thresholds play a central role; see \cite{mezard2006reconstruction,liu2022statistical,mossel2025weak}.

\paragraph*{Organization of the paper} 
The rest of the paper is organized as follows. 
Section~\ref{sec:model} introduces the non-uniform HSBM, together with the weighted non-backtracking operator and the associated signal-to-noise parameters. 
Section~\ref{sec:main_result} presents the main results, including the spectral theory of the weighted non-backtracking matrix, an Ihara--Bass formula, eigenvector overlaps, and algorithms for weak recovery. 
Section~\ref{sec:proof_main} contains the proofs of Theorems~\ref{thm:main} and \ref{th:main_reduced} using perturbation analysis of non-Hermitian random matrices. 
Section~\ref{sec:matrix_decomp} decomposes the weighted non-backtracking matrix into several components and bounds their spectral norms via the high-trace method.  
Section~\ref{sec:local_hypertree} establishes the connection between the non-uniform HSBM and non-uniform Galton--Watson hypertrees and analyzes the associated branching processes. Section~\ref{sec:functional_hypertree} 
combines these ingredients to approximate hypergraph functionals by their analogues on the hypertree process, and translates quantities on the hypertree process into information about the pseudoeigenvectors of \(B\).   Section~\ref{sec:proof_prop_main} concludes the proof of Proposition~\ref{prop_main}.
Additional proofs are deferred to Appendix~\ref{appendix:proof}.

\section{Preliminaries}\label{sec:model}

\subsection{Model parameters and assumptions}
\begin{definition}[Hypergraph]
 A  \textit{hypergraph} 
 	$G$ is a pair $G=(V,H)$ where $V$ is a set of vertices and $H$ is the set of non-empty subsets of $V$ called \textit{hyperedges}. If any hyperedge $e\in H$ is a set of $q$ elements of $V$, we call $G$ \textit{$q$-uniform}.   In particular, a $2$-uniform hypergraph is an ordinary graph.
  The \textit{degree} of a vertex $x\in V$ is the number of hyperedges in $G$ that contain $x$. A $q$-uniform hypergraph is complete if any set of $q$ vertices is a hyperedge. 
\end{definition}

\begin{definition}[$q$-uniform hypergraph stochastic block model]
Consider  an order-$q$  \textit{symmetric probability tensor} $\mathbf P\in \mathbb R^{r^q}$ such that
$p_{i_1,\dots,i_q}=p_{s(i_1),\dots, s(i_q)}$
for any permutation $s$ on $[q]$.  We sometimes use 
$p_{\underline i}$,  $p_{i, \underline j}$ to specify the index for an entry in $\mathbf P$.
Let the vertex set of a hypergraph be  $V = [n]$, and assign each vertex $x$ a \emph{type} $\sigma(x) \in [r]$. Then, each hyperedge $e$ of size $q$ is included in $H$ with probability
        \[ \Pb{e \in H} = \frac{p_{\underline\sigma(e)}}{\binom{n}{q - 1}}, \]
for any hyperedge $e=\{x_1,\dots, x_q \}$, where $\underline \sigma(e)=\underline\sigma(\Set{x_1, \dots, x_q}):= (\sigma(x_1), \dots, \sigma(x_q))$.
The  \textit{hypergraph stochastic block model} is defined as the distribution of a random hypergraph  $G=(V,H)$ generated according to $\mathbf{P}$ and $\sigma$, where all entries in $\mathbf P$ are constant independent of $n$.  
\end{definition}

Now we are able to define our non-uniform hypergraph SBM as follows.  
\begin{definition}[Non-uniform hypergraph SBM]
For each $1\leq k\leq K$, generate a $q^{(k)}$-uniform    HSBM with probability tensor  $\mathbf P^{(k)}\in \R^{r^{q^{(k)}}}$. Let $G^{(1)},\dots, G^{(K)}$ be the random hypergraphs sampled from the $q^{(k)}$-uniform   HSBM with $q^{(k)}\geq 2$ for all $k\in [K]$. 
\end{definition}
The hyperedges in each \(G^{(k)}\) are endowed with labels (colors). In particular, even when some \(q^{(k)}\) coincide, hyperedges originating from different layers remain distinguishable.

We work under the following two assumptions on the degrees of the HSBM:
\begin{assumption}[Constant average degree]\label{assumption:degree}
For each $k \in [K]$,
    \begin{equation}\label{eq:constant_degree}
    \sum_{\underline j \in [r]^{q^{(k)}-1}} p_{i, \underline j}^{(k)} \prod_{\ell \in \underline j} \pi_{\ell} = d^{(k)}, \quad i \in [r],
\end{equation}
where the $\pi_i$ are the proportions of each type in $V$:
 $\pi_{i} = \frac{\#\Set{x \in [n] \given \sigma(x) = i}}{n}$. 
\end{assumption}
This condition ensures that each vertex has the same expected degree.
\begin{assumption}[Average degree lower bound]
   \begin{align}
    \sum_{1\leq k\leq K}(q^{(k)}-1)d^{(k)}>1. \label{eq:avg_deg_bound}
\end{align} 
\end{assumption}
  This is a natural assumption since otherwise there are no giant components in the HSBM and detection becomes impossible \cite{schmidt1985component}. 
For each $k$, we define the two-type average degree matrix
\begin{align} \label{eq:defD2}
\bD_{ij}^{(k)} = \sum_{\underline k \in [r]^{q-2}} p_{ij, \underline k}\prod_{\ell \in \underline k}\pi_\ell.
\end{align}
By our symmetry assumption on $\mathbf P^{(k)}$, $\bD^{(k)}$ is a symmetric matrix, and we have
\begin{equation}\label{eq:perron_eigenvalue_Q}
    d^{(k)} = \sum_{j \in [r]} \bD_{ij}^{(k)} \pi_j \quad \text{for all } i \in [r].
\end{equation}

The \emph{signal matrix} $\bQ^{(k)}$ is given by 
\begin{align}\label{eq:Q=DPi}
 \bQ^{(k)}_{ij} =  \bD^{(k)}_{ij}\pi_j    
\end{align}
so $\bQ^{(k)}=\bD^{(k)}\Pi$ with $\Pi = \diag(\pi)$. Since $\bQ^{(k)}$ is similar to the symmetric matrix $ \Pi^{1/2} \bD^{(k)}\Pi^{1/2}$, its eigenvalues are all real. If all eigenvalues of $\bQ^{(k)}$ are positive, we call $G^{(k)}$ an \textit{assortative} model; otherwise, it's called \textit{disassortative}.  Since $\bQ^{(k)}$ is a contraction from an order-$q^{(k)}$ symmetric probability tensor, Lemma~\ref{lem:deterministic_Q} implies that its eigenvalues satisfy certain constraints:
\begin{lemma} \label{lem:deterministic_Q}
    All eigenvalues of $\bQ^{(k)}$ are bounded below by $-\frac{d^{(k)}}{q^{(k)}-1}$.
\end{lemma}
Although the matrices $\bQ^{(k)}$ encode the same community assignment, they do not necessarily share a common eigenspace; see Example~\ref{example:weight_4}.

\paragraph{Weighting matrices}
Throughout this paper, we will need to consider weighted sums of the matrices $\bQ^{(k)}$. For a vector $a \in \R^K$, we define
\begin{equation}\label{eq:def_Q_weighted}
    \bQ_a = \sum_{k=1}^K a^{(k)} (q^{(k)}-1) \bQ^{(k)}.
\end{equation}

For simplicity, we will treat all operations appearing in the weights (e.g. $Q_{(q-2)},Q_{w^2}, Q_{(q-1)/(q-2)}$) as being applied elementwise to the $q^{(k)}, w^{(k)}$.

In particular, we will fix a reweighting $w$, and define the \emph{signal} and \emph{variance} matrices as
\begin{align}
    \bQ=\bQ_{w}=\sum_{k=1}^K w^{(k)} (q^{(k)}-1) \bQ^{(k)}, \quad \bK=\bQ_{w^2} = \sum_{k=1}^K \left(w^{(k)}\right)^2 (q^{(k)}-1)\bQ^{(k)}. \notag
\end{align}
Denote  the weighted two-type average degree matrix as 
\begin{align}
    \bD=\sum_{k=1}^K w^{(k)}(q^{(k)}-1) \bD^{(k)}. \notag
\end{align}
Again, $\bQ = \bD \Pi$ is similar to the symmetric matrix $\Pi^{1/2} \bD \Pi^{1/2}$, and hence its eigenvalues are real; we order them by absolute value:
\[ \mathrm{Sp}(\bQ) = \{|\mu_r| \leq \dots \leq |\mu_1| \} \]
Additionally, the eigenvectors $\phi_i$ of $\bQ$ are equal to $\Pi^{-1/2}\psi_i$, where $\{\psi_i\}_{1\leq i\leq r}$ is a set of orthonormal eigenvectors of $ \Pi^{1/2} \bD\Pi^{1/2}$, and hence
\begin{align}\label{eq:pi_orthogonal}
    \langle \phi_i,\phi_j\rangle_{\pi}:=\sum_{k\in [r]}\pi_k\phi_i(k)\phi_j(k)=\langle \psi_i,\psi_j\rangle =\delta_{ij}.
\end{align}
Further, $\bK$ is a nonnegative matrix, and thus by the Perron-Frobenius theorem its top eigenvalue is also nonnegative. We define
\begin{equation}\label{eq:def_vartheta}
    \vartheta = \lambda_1(\bK) = \sum_{k=1}^K \left(w^{(k)}\right)^2 (q^{(k)}-1)d^{(k)},
\end{equation}
where we use the fact that $\phi_1=(1,\dots, 1)^\top\in \R^r$ is the shared top eigenvector for all $\bQ^{(k)}, 1\leq k\leq K$ due to \eqref{eq:constant_degree}.
Define the \textit{inverse signal-to-noise ratios} for the weighted HSBM as 
\begin{align}\label{eq:SNR}
    \tau_i = \frac{\lambda_1(\bK)}{ \lambda_i(\bQ)^2} = \frac{\vartheta}{\mu_i^2}.
\end{align}

Note that $w^{(k)}$ can be negative and \eqref{eq:SNR} is scale-invariant in terms of $w^{(k)}$. 
Let $r_0$ be the only integer such that 
\begin{align}
    \tau_{r_0}<1,  \quad \tau_{r_0+1}\geq  1. \notag
\end{align}
Here, $r_0$ is the number of eigenspaces of $\bQ$ that can be recovered above the \textit{Kesten–Stigum threshold}.

\paragraph{Average and weighted degrees}

From \eqref{eq:perron_eigenvalue_Q} and the Perron-Frobenius theorem, the degree $d^{(k)}$ is the largest eigenvalue of $\bQ^{(k)}$, associated to the all-one eigenvector. As a result, all matrices of the form $\bQ_a$ admit the all-one vector as an eigenvector, with associated eigenvalue the weighted degree
\begin{equation}\label{eq:def_d_weighted}
    d_a := \sum_{k=1}^K a^{(k)} (q^{(k)}-1) d^{(k)}.
\end{equation}
In particular, when $a = \mathbf{1}$, we get the average degree of the graph that we denote with $d_{\mathbf{1}} = d$. We note that contrary to the typical hypergraph setting, $d_w$ is not necessarily the top eigenvalue of $\bQ_w$, and hence in some cases the top eigenvector of $\bQ_w$ can contain useful information on the community structure. On the other hand, $d_a$ is the top eigenvector of $\bQ_a$ as long as $a$ has only positive entries, and in particular $\vartheta = d_{w^2}$

\paragraph{Parameter scaling} In the proof, we treat all parameters $\pi_i$, $\mathbf{P}^{(k)}, q^{(k)}$ as absolute constants, in particular with respect to $\lesssim$. However, the interested reader can check that the implicit constants are always polynomials in $r, Q^{(k)}, K, (1 - \tau_{r_0})^{-1}$ (not $q^{(k)}$ !), and hence our results all hold under the growth condition
\begin{equation}
    r, d^{(k)}, K, (1 - \tau_{r_0})^{-1} = O(\operatorname{polylog}(n))
\end{equation}

\subsection{Non-backtracking operators for weighted   hypergraphs}

 For a given hypergraph $G=(V,H)$, let $\vec{H}$ be the \textit{oriented hyperedges} of $G$ such that 
\begin{align*}
    \vec{H}=\{ (x, e) : x\in e\cap V, e\in H \}.
\end{align*}
We shall sometimes use the notation $x\to e$ instead of  $(x,e)$ to emphasize that it is an oriented hyperedge. We assign a hyperedge coming from $G^{(k)}$ with weight $w^{(k)}$.
Now we can define the hypergraph non-backtracking operator as follows.

\begin{definition}[Non-backtracking operator for  hypergraphs] For a given non-uniform hypergraphs $(V, E_k, \{w^{(k)}\}, 1\leq k\leq K)$, let  $B$ be a matrix indexed by $\vec{H}$ such that
\begin{align} \label{eq:defB}
B_{(x \to e), (y \to f)} = \begin{cases}
    w^{(k)} & \text{if }  y \in e \setminus \{x\}, f \neq e, f\in E_k \\
    0 & \mathrm{otherwise}.
  \end{cases}
\end{align}
\end{definition}

Define the \textit{edge reversal operator} $J$ such that
\begin{equation}\label{eq:def_P}
    J_{(x\to e), (y\to f)} = \ind_{e = f, x \neq y}
\end{equation}
Equivalently, for any $u \in \dR^{\vec H}$,
\[ [{J}u](x\to e) = \sum_{x\in e, y\neq x} u_{y \to e}. \]

We can explicitly compute the spectrum of $J$ as follows.
\begin{lemma}\label{lem:J_eigenvalue}
    $J$  has eigenvalues $q^{(k)}-1$ with multiplicity $|E_k|$ for $1\leq k\leq K$ and eigenvalue $-1$ with multiplicity $\sum_{k=1}^K (q^{(k)}-1)|E_k|$. Consequently, $J$ is invertible.
\end{lemma}
The matrix {$J$} is useful in the study of the non-backtracking matrix $B$ in part due to the following formula, known as the \emph{parity-time} symmetry:
\begin{lemma}[parity-time symmetry]\label{lem:BPPB}
Let $D_w$ be a diagonal matrix with $(D_w)_{x\to e, x\to e}=w_e$.  For any $k\geq 0$,
  \[D_w B^kJ=J{B^\star }^kD_w.\]
\end{lemma}

 Define the \emph{start matrix}  {$S\in \mathbb R^{V\times \vec H}$} and the terminal matrix $T\in R^{V\times \vec H}$ such that
\begin{equation}\label{eq:def_S_T}
    S_{x, y\to e} = \ind_{x = y} , \qquad T_{x,y\to e}=\1_{x\in e,x\not=y}.
\end{equation} 
Define $\tilde \phi_i \in \R^n$ such that
\begin{equation}\label{eq:ndimlifting}
    \tilde \phi_i(x) = \phi_i(\sigma(x)).
\end{equation}
The corresponding lifted eigenvectors of $\chi_i\in \mathbb R^{\vec{H}}$ are defined {for $i \in [r]$} as
\begin{align}\label{eq:defchi}
  {\chi_i = S^\star \tilde \phi_i}, \quad \text{or equivalently} \quad \chi_i(x \to e) =  \phi_i(\sigma(x)).
\end{align}

\section{Main results}\label{sec:main_result}

\subsection{Spectrum of the non-backtracking matrix}
Our first result characterizes the outlier eigenvalues of $B$.
\begin{theorem}[Spectrum of $B$]\label{thm:main}
Let $B$ be the non-backtracking matrix of the non-uniform HSBM with weights $w^{(k)}, k\in [K]$. With probability $1-n^{-c}$, the following holds:
\begin{enumerate}
    \item (Outliers) For $i\in [r_0]$,
    \begin{align}
        \lambda_i(B)= \mu_i+O(n^{-c}). \notag
    \end{align}
    \item  (Bulk spectrum) For all $i>r_0$, 
    \begin{align}
        |\lambda_i(B)|\leq  \sqrt{\vartheta}+o(1). \notag
    \end{align}
\end{enumerate}
\end{theorem}

Theorem~\ref{thm:main} shows that the outliers of $B$ are close to the top $r_0$ eigenvalues of $\mathbf{Q}$,  enabling parameter estimation above the Kesten--Stigum threshold.  However, $B$ has dimension $|\vec H|\times |\vec H|$, making its spectral computation potentially costly.  We therefore construct a reduced non-backtracking matrix of smaller dimension that shares the same nontrivial spectrum as $B$.
\subsection{A reduced non-backtracking matrix and the Bethe-Hessian matrix}\label{sec:reducedNB}

Let $A_k, D_k$ be the adjacency and degree matrices of $G^{(k)}$, respectively, where 
\begin{align}
    D_{k}(x,x)=\frac{1}{q^{(k)}-1}\sum_{y\in [n]} A_k(x,y) \notag
\end{align}
counts the number of hyperedges in $G^{(k)}$ containing $x$.

We define the  \textit{reduced non-backtracking matrix} $\tilde B$ of size $2Kn$ with blocks
\begin{align} \label{eq:tildeB}
\tilde B_{k\ell} = w_\ell \begin{pmatrix}
    0 & D_k -\delta_{k\ell}I \\
    - (q_\ell -1) \delta_{k\ell} & A_k - (q_\ell - 2)\delta_{k\ell}I
\end{pmatrix}. 
\end{align}
Each main block matrix $\tilde B_{kk}$ corresponds to the reduced non-backtracking operator for a $q^{(k)}$-uniform hypergraph considered in \cite{stephan2024sparse}, up to the scalar $w^{(k)}$.

With the relation above, we derive an Ihara-Bass formula for weighted non-uniform hypergraphs. 
\begin{proposition}[Ihara-Bass formula]\label{thm:IharaBass}
Let $G$ be a weighted non-uniform hypergraph  where $G^{(k)}=(V, E_k)$ is $q^{(k)}$-uniform with $|V|=n$  and
hyperedge set $E_k$ with $m_k=|E_k|$. Each hyperedge in $E_k$ is weighted by $w^{(k)}$. Then the following identity holds for any $\lambda \not\in \{w^{(k)},-w^{(k)}(q^{(k)}-1)\}$: 
    \begin{equation}\label{eq:ihara-bass-lambda}
\det(\lambda I - B)
=
\det(\lambda I - \tilde B)\,
\prod_{k=1}^K
(\lambda-w^{(k)})^{(q^{(k)}-1)m_k-n}\;
(\lambda+w^{(k)}(q^{(k)}-1))^{m_k-n}.
\end{equation}
\end{proposition}
This shows $B$ and $\tilde B$ have the same eigenvalues up to deterministic eigenvalues $\{w^{(k)}, -w^{(k)}(q^{(k)}-1), 1\leq k\leq K\}$. Proposition~\ref {thm:IharaBass} generalizes the classical Ihara-Bass formula for graphs \cite{bass_iharaselberg_1992} and for uniform hypergraphs in \cite{stephan2024sparse}.  

\begin{remark}[Trivial eigenvalues of $B$] 
 Note that Theorem~\ref{thm:main} implies that if 
\(|w^{(k)}(q^{(k)}-1)| > \sqrt{\vartheta}\) for some \(k\), then 
\(-w^{(k)}(q^{(k)}-1)\) is not an eigenvalue of \(B\) with high probability. 
This is consistent with the Ihara--Bass identity~\eqref{eq:ihara-bass-lambda}, 
since the condition \(|w^{(k)}(q^{(k)}-1)| > \sqrt{\vartheta}\) implies 
\(d^{(k)} < q^{(k)} - 1\). 
In the non-uniform HSBM, we have 
$m_k = \frac{d^{(k)}_{\mathrm{avg}}\, n}{q^{(k)} - 1}$,
where the average degree satisfies \(d^{(k)}_{\mathrm{avg}} = d^{(k)} + o(1)\) with high probability. 
Hence \(m_k < n\), and in particular the exponent \(m_k - n\) in 
\eqref{eq:ihara-bass-lambda} is negative. Since from Theorem~\ref{thm:main}, the value 
\(-w^{(k)}(q^{(k)}-1)\) does not correspond to an eigenvalue of \(B\) (otherwise it will be an outlier outside the bulk), 
while it appears as an eigenvalue of \(\widetilde B\) with multiplicity exactly \(n - m_k\).
\end{remark}

The proof of Proposition~\ref{thm:IharaBass} indicates the following  \textit{Bethe-Hessian matrix}:
\begin{align}\label{eq:def_BH}
    H(\lambda):=I_n
-\sum_{k=1}^K \frac{w^{(k)}\lambda}{(\lambda-w^{(k)})\bigl(\lambda+w^{(k)}(q^{(k)}-1)\bigr)}A_k
+\sum_{k=1}^K \frac{{w^{(k)}}^2(q^{(k)}-1)}{(\lambda-w^{(k)})\bigl(\lambda+w^{(k)}(q^{(k)}-1)\bigr)}D_k.
\end{align}
 This definition generalizes the Bethe-Hessian matrices in graphs \cite{saade2014spectral,stephan2024community} and uniform hypergraphs \cite{stephan2024sparse}.

\subsection{Eigenvector relations}
Similar to \eqref{eq:def_S_T},  define the \emph{start} and \emph{terminal} matrices {$S_k,T_k\in \mathbb R^{V\times \vec H}$} such that
\begin{equation}\label{eq:def_S_T_k}
    S_k(x, y\to e) = \ind_{x = y} \ind_{e \in E_k} \quand T_k(x, y\to e) = \ind_{x\in e, x\neq y} \ind_{e \in E_k},
\end{equation} 
and the \emph{edge inversion} matrices $J_k(x \to e, y \to f) = \ind_{e = f} \ind_{x \neq y} \ind_{e \in E_k}$. 

The following lemma generalizes the eigenvector relation between the Bethe-Hessian matrix and the non-backtracking matrix \cite{dall2019revisiting,stephan2024community} to weighted non-uniform hypergraphs:

\begin{proposition}[Eigenvector relation between $B$ and $\tilde B$, and $H(\lambda)$]\label{prop:Bethe_hessian_relation}
    Let $v$ be an eigenvector of $B$ with associated eigenvalue $\lambda$ with $\lambda\not= w^{(k)}, -w^{(k)}(q^{(k)}-1)$ for all $k\in [K]$. Define
    \begin{align*}
        v_k^{\leftarrow} (x \to e) &= S_k J_k^{-1} v, \qquad
        v_k^{\rightarrow}(x \to e) = S_k v,
    \end{align*}
    where $J_k^{-1}$ denotes the inversion of $J_k$ on the $E_k$ block.  Then the following holds:   
 
    \begin{enumerate}
        \item Define $\vec v\in \R^{2Kn}$ such that \begin{align} \label{eq:def_tildev}
    \tilde v = \begin{pmatrix}
        v_1^{\leftarrow} \\
        v_1^{\rightarrow} \\
         \vdots \\
        v_K^{\leftarrow} \\
        v_K^{\rightarrow}
    \end{pmatrix}.
    \end{align} Then $\tilde v$ is an eigenvector of $\tilde B$, with eigenvalue $\lambda$.
        \item Define the aggregated vector $y\in \R^n$:
\begin{align}\label{eq:def_y}
    y=\sum_{k=1}^K w^{(k)} v_k^{\rightarrow}.
\end{align}Then $H(\lambda)y=0$.
    \end{enumerate}

\end{proposition}

We will use the aggregated vectors $y_1,\dots,y_{r_0}$ obtained from top eigenvectors $\tilde v_1,\dots, \tilde v_{r_0}$ from $\tilde B$ as estimators of the eigenspace of $\bQ$.

\subsection{Eigenvector overlap}

With the reduced non-backtracking matrix characterization, the next theorem characterizes the overlap between $y_i$ and $\tilde \phi_i$ defined in \eqref{eq:ndimlifting}:
\begin{theorem}[Eigenvector overlaps]\label{th:main_reduced}
Let $\tilde B$ be the reduced non-backtracking matrix associated with the matrix $B$. Let $y_1, \dots, y_{r_0}$ be the aggregated vectors defined in \eqref{eq:def_y}, associated to the top $r_0$ eigenvectors of $\tilde B$. With probability $1 - n^{-c}$, the following holds:
\begin{enumerate}[(i)]
    \item (Overlap) for any $i \in [r_0]$, there exists an eigenvector $\phi'_i$ of $Q$ associated to $\mu_i$ such that
    \begin{align} 
    \frac{\langle y_i, \tilde \phi_i' \rangle}{ \|y_i\| \cdot \|\tilde \phi_i'\|} = \sqrt{\frac1{\gamma_i}} + O(n^{-c}),  \notag
    \end{align}
    where $\tilde \phi'$ is defined in~\eqref{eq:ndimlifting} and
    \begin{equation}\label{eq:def_gamma_i}
        \gamma_i = \frac{1 + \tau_i \phi_i^\top \Pi \bQ_{(q-2)w^2/\vartheta} \phi_i}{1 - \tau_i}.
    \end{equation} 
    \item (Orthogonal to  $\mathbf{1}$)
If $\mu_i \neq d_w$, then
\begin{align}\label{eq:near-orthogonoal}
\frac{\langle y_i, \1 \rangle}{ \sqrt n\|y_i\| } =O(n^{-c}).
\end{align} 
\end{enumerate}
\end{theorem}

\begin{figure}[ht!]
    \centering
    \includegraphics[width=0.45\linewidth]{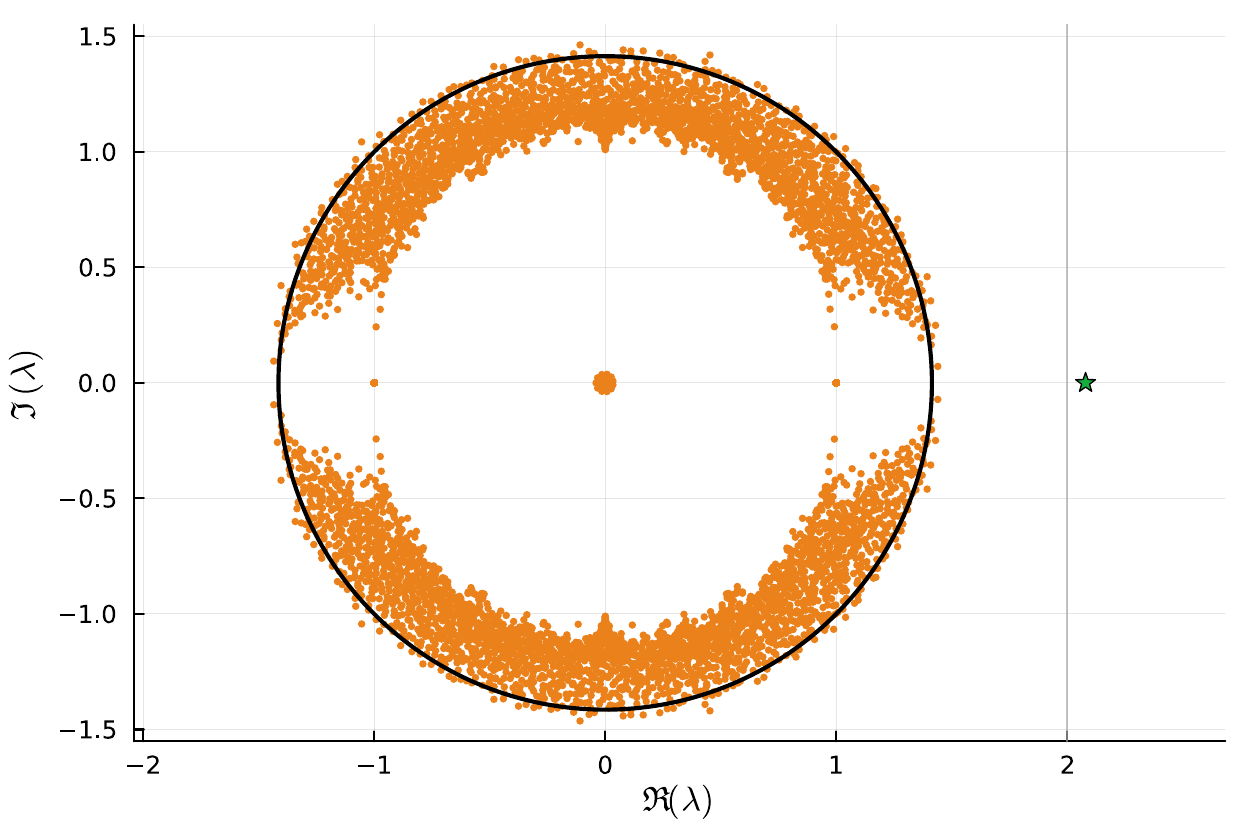} \hspace{1em}
        \includegraphics[width=0.45\linewidth]{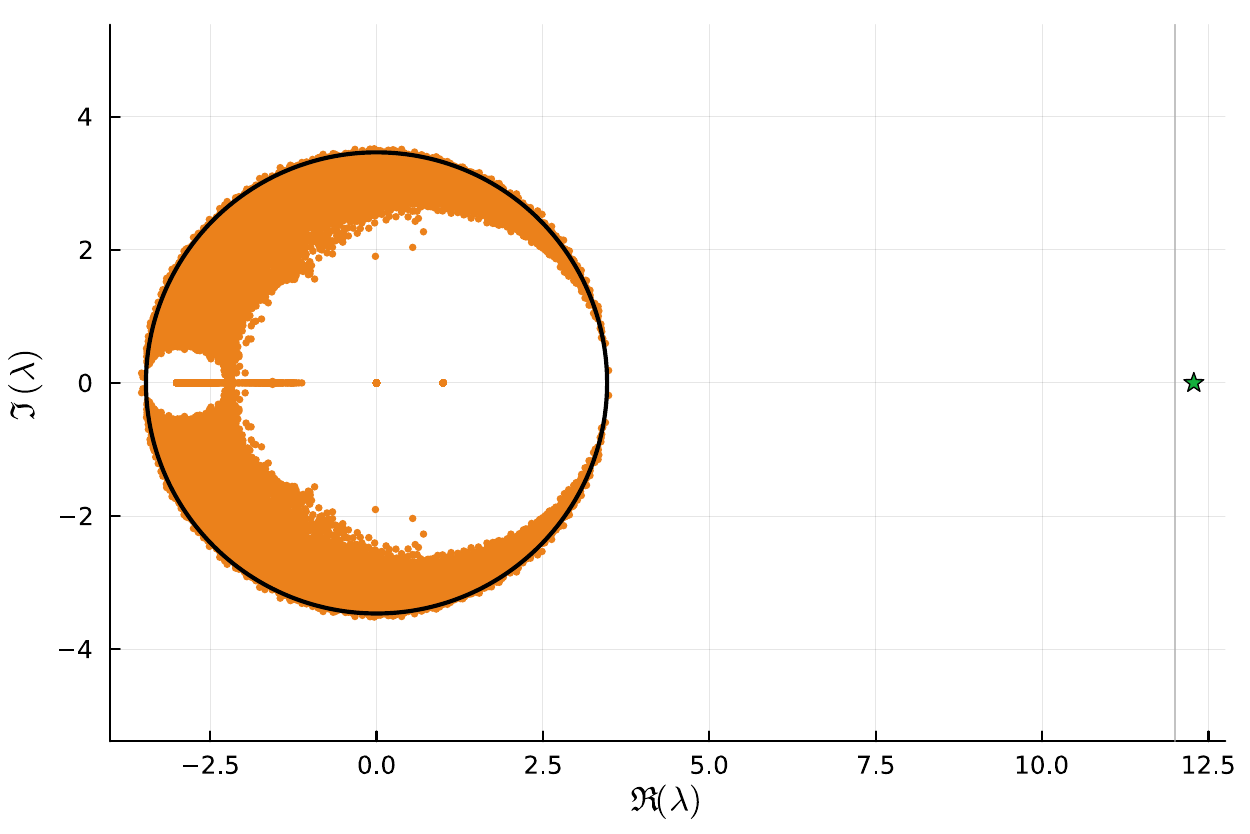}
        
        \vspace{1em}
        \includegraphics[width=0.45\linewidth]{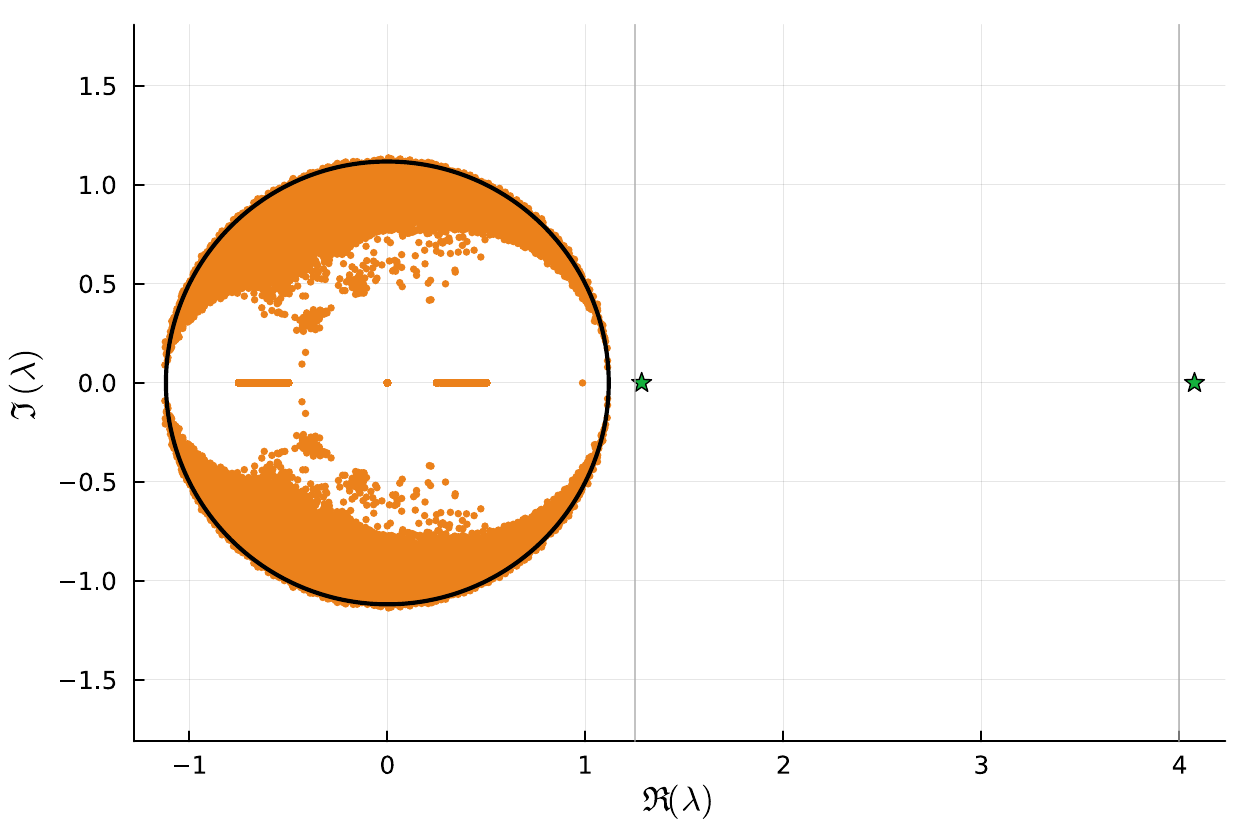} \hspace{1em}
        \includegraphics[width=0.45\linewidth]{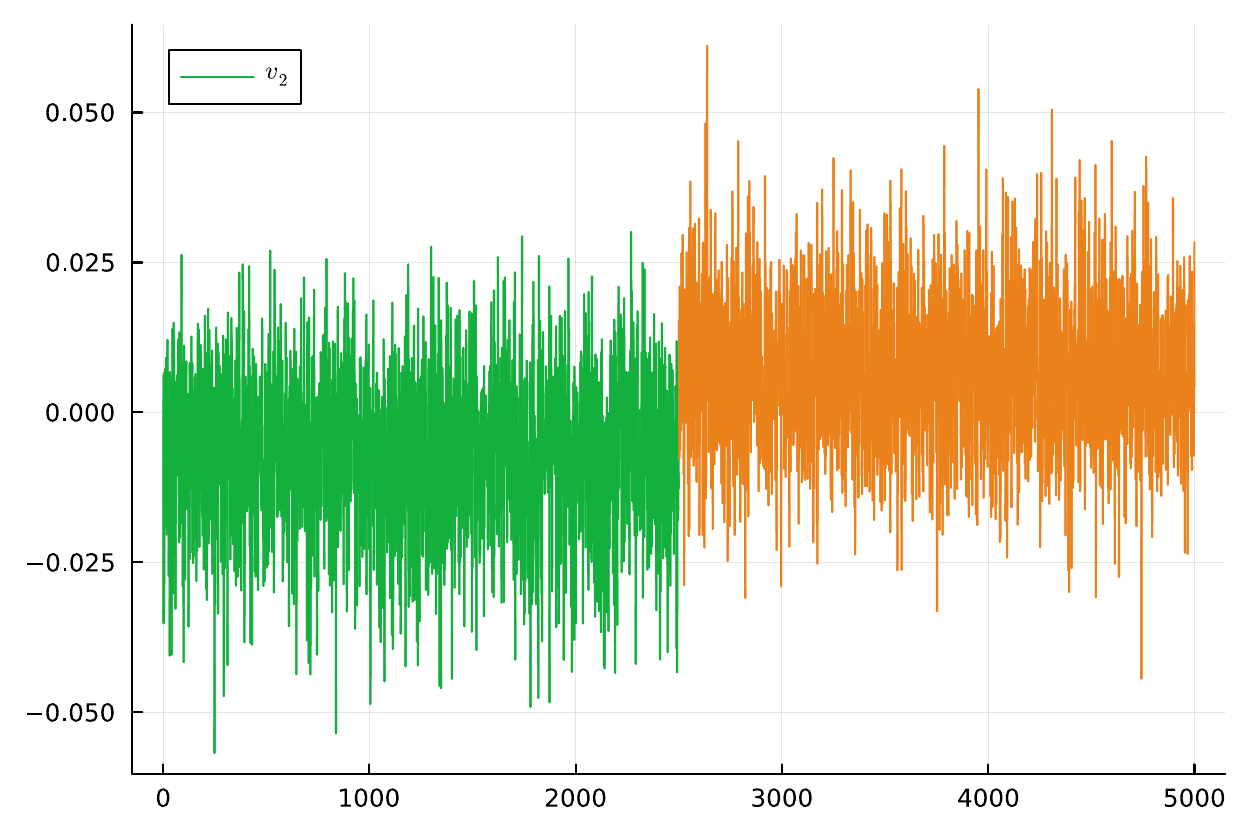}
    \caption{Top: Spectrum of the non-backtracking matrix of two HSBMs with $n = 5000, r=2$. Left has parameters ($q^{(1)}=2, d^{(1)} = 2, \mu_2^{(1)} = 1$), right has parameters ($q^{(2)}=4, d^{(2)} = 4, \mu_2^{(2)} = 1$). Both are below the KS threshold, and hence there is only one uninformative outlier in their spectrum. \newline
    Bottom: Spectrum of the combined non-backtracking matrix, using the optimal values from Example~\ref{example:weight_2}. A new outlier emerges from the bulk, whose aggregated eigenvector (pictured right) is correlated with the community structure. Assigning vertices to communities depending on the sign of $y_2$ yields a 70.5\% accuracy.}
    \label{fig:simulation}
\end{figure}
We present a simulation illustrating the effectiveness of the weighted spectral method in Figure~\ref{fig:simulation}.

\subsection{A weak reconstruction algorithm}

Given an estimate $\hat\sigma$ of the community assignment $\sigma$, the \emph{overlap} between the two vectors is defined as
\begin{align} \label{eq:ov}
\operatorname{ov}(\sigma, \hat\sigma) = \max_{\mathfrak{p} \in \mathfrak{S}_r} \frac1n \sum_{x=1}^n \mathbf{1}\{\hat\sigma(x) = \mathfrak{p} \circ \sigma(x)\}, 
\end{align}
where the maximum is over all permutations of $[r]$. 

Theorem~\ref{th:main_reduced} implies the following guarantee for a weak reconstruction algorithm for $r$ communities of equal size, stated in Algorithm~\ref{alg:provable_algorithm}. Note that Algorithm~\ref{alg:provable_algorithm} outputs only 2 communities, but we can show the ratio of correct label assignment is strictly larger than $1/r$.

\begin{algorithm}
\caption{Provably efficient reconstruction algorithm when $\pi_i=\frac{1}{r}, i\in [r]$}\label{alg:provable_algorithm}

\begin{algorithmic}[1]
\Require uniform hypergraphs $G^{(k)}, k\in [K]$, thresholding parameter $T$, weights $w^{(k)}, k\in [K]$.
\State Form the $2Kn\times 2Kn$ matrix $\tilde B$ as in \eqref{eq:tildeB}.
\State Compute the top two eigenvectors $\tilde v_1, \tilde v_2$ of $\tilde B$, and form two vectors $\tilde y_1,\tilde y_2$ from $\tilde v_1, \tilde v_2$  using \eqref{eq:def_y}.  
\State If  $\frac{|\langle \tilde y_1,\1\rangle|}{\sqrt{n} \|\tilde y_1 \|} > \frac{1}{\log n}$, take $u=\tilde y_2$. Otherwise, take $u=\tilde y_1$. Normalize $u$ such that $\norm{u}^2 = n$.
 
\State Partition the vertex set $[n]$ in two sets $(V^+, V^-)$ randomly, such that
\[ \Pb{x \in V^+ \given u} = \frac{1}{2} + \frac1{2T} u(x) \ind_{|u(x)| \leq T}. \]
\State Assign vertices in $V^+$ to the community $1$ and vertices in $V^-$ to the community $2$.
\end{algorithmic}
\end{algorithm}

\begin{theorem}[Weak reconstruction algorithm]\label{thm:weakreconstruction}  Under Assumption \ref{assumption:degree}  and the extra condition that the communities are balanced, i.e., $\pi_i = 1/r$ for all $i\in [r]$, there exists a deterministic threshold $T=2\sqrt{2}r\gamma^{1/2}$  such that for some constants $c>0$, with probability at least $1-n^{-c}$,
the estimator $\hat\sigma$ output by Algorithm \ref{alg:provable_algorithm} satisfies
\[ \mathrm{ov}(\sigma, \hat\sigma) \geq  \frac 1r + \frac1{8r\gamma} + O(n^{-c}), \]
where $\gamma:=\begin{cases}
    \gamma_2 & \text{if } \phi_1=\1\\
    \gamma_1 & \text{if } \phi_1\not=\1
\end{cases}$, and $\gamma_{1},\gamma_2$ are defined in \eqref{eq:def_gamma_i}.
\end{theorem}

\subsection{The choice for the  weights}
For different settings, the weights $w^{(k)}, k \in [K]$, can be chosen in various ways.  We illustrate this flexibility through the following examples. 
With appropriate choices of the weights, Algorithm~\ref{alg:provable_algorithm} attains two thresholds conjectured in \cite{chodrow2023nonbacktracking}.

\begin{example}[Unweighted case]\label{example:weight_1}
Setting $w^{(k)} = 1$ for all $k \in [K]$ yields the unweighted non-backtracking matrix $B$. 
In the assortative binary setting, the algorithmic threshold for spectral methods based on $B$ was conjectured in \cite[Conjecture 4.3]{chodrow2023nonbacktracking}. 
Combining Theorems~\ref{thm:main}, \ref{th:main_reduced}, and \ref{thm:weakreconstruction},  the unweighted non-backtracking spectral algorithm achieves weak recovery if and only if
\begin{align}
    \frac{\left(\sum_{k=1}^K (q^{(k)} - 1)\mu_2^{(k)}\right)^2}
    {\sum_{k=1}^K (q^{(k)} - 1)\,d^{(k)}} > 1,  \qquad \text{where} \qquad \mu_2^{(k)}:=\lambda_2(\mathbf{Q}^{(k)}). \notag
\end{align}
thereby confirming \cite[Conjecture 4.3]{chodrow2023nonbacktracking}. In particular, this spectral method works without the knowledge of model parameters.
\end{example}

\begin{example}[Optimal values of $w^{(k)}$ when $r=2$] \label{example:weight_2}
When $r=2$ and $\pi_1=\pi_2=1/2$, since all $\bQ^{(k)}$ share the same second eigenvector $(1,-1)^\top$, we have \[\tau_2^{-1}= \frac{\left( \sum_{k}w^{(k)}(q^{(k)}-1) \mu_2^{(k)}\right)^2}{\sum_{k} \left(w^{(k)}\right)^2 (q^{(k)}-1)d^{(k)}}.\] 
Then by the Cauchy-Schwarz inequality,
    \begin{align}
        \tau_2^{-1}\leq \sum_{k} \frac{(q^{(k)}-1)|\mu_2^{(k)}|^2}{d^{(k)}} :=  (\tau_2^\star)^{-1}, \notag
    \end{align}
    and the optimal choice of $w^{(k)}$   is given by ${w^{(k)}}^\star =\frac{\mu_2^{(k)}}{d^{(k)}}$. The condition \begin{align}
       (\tau_2^\star)^{-1}= \sum_{k=1}^K   \frac{(q^{(k)}-1)|\mu_2^{(k)}|^2}{d^{(k)}}>1 \notag
    \end{align} is  the detection threshold  in \cite[Conjecture 6.3]{chodrow2023nonbacktracking} for non-uniform HSBMs with 2 blocks. Our results show that with the optimal hyperedge reweighting ${w^{(k)}}^\star $, a spectral method based on the weighted non-backtracking operator is able to achieve this bound.
\end{example}

 To achieve a nontrivial estimate of the eigenspaces of $\bQ$ in general, the best weight vector $w$ is given by maximizing the  SNR associated with the nontrivial eigenspaces. Namely we choose
 \begin{align*}
    w^\star &= \arg \max_{w} \frac{\lambda_1(\bQ-\sum_{k=1}^K w^{(k)}(q^{(k)}-1)d^{(k)}\phi_1 (\Pi \phi_1)^\star ) }{\lambda_1(\bK)} \\
    &= \arg \max_{w} \frac{1}{\lambda_1(\bK)}\left(\max_{\|v\|_2=1, \langle v, \phi_1\rangle =0} \langle v, \bQ v\rangle_{\pi}\right)^2 .
 \end{align*}
When all $\bQ^{(k)}$ share the same eigenspaces, the expression of $w^\star $ is more explicit.
\begin{example}[All $\bQ^{(k)}$ have shared eigenvectors]
Assume $\bQ^{(k)}$ has the same eigenvectors $\phi_i, 1\leq i\leq r$ with associated eigenvalues $\mu_i^{(k)}$ for all $k\in [K]$. Here $\mu_i^{(k)}$ are not necessarily ordered identically for all $k$ except $\mu_1^{(k)}=d^{(k)}$ which is always the largest eigenvalue of $\bQ^{(k)}$.

Then all of the  eigenvalues of $\bQ$ are given by 
   $\sum_{k=1}^K w^{(k)}(q^{(k)}-1)\mu_i^{(k)}, 1\leq i\leq r$,
with one of the eigenvalues given by $\sum_{k=1}^K w^{(k)}(q^{(k)}-1)d^{(k)}$ associated with eigenvector $\mathbf{1}$. Then the optimal weight for weak recovery is given by
\begin{align*}
    w^\star = \arg \max_{w} \max_{i\not=1}\frac{\left( \sum_{k}w^{(k)}(q^{(k)}-1) \mu_i^{(k)}\right)^2}{\sum_{k} \left(w^{(k)}\right)^2 (q^{(k)}-1)d^{(k)}}.
 \end{align*}
\end{example}

\begin{example}[An example when $\bQ^{(k)}$ do not share the same eigenvectors]
\label{example:weight_4}
Consider the case \(K=2\), \(r=3\), \(q^{(1)}=q^{(2)}=2\), and
$
\pi=\left(\tfrac13,\tfrac13,\tfrac13\right)$.
Define

\[
\bQ^{(1)}=
\begin{pmatrix}
\frac76 & \frac16 & \frac23\\[1mm]
\frac16 & \frac76 & \frac23\\[1mm]
\frac23 & \frac23 & \frac23
\end{pmatrix},
\qquad
\bQ^{(2)}=
\begin{pmatrix}
\frac{16}{15} & \frac{7}{15} & \frac{7}{15}\\[1mm]
\frac{7}{15} & \frac{23}{30} & \frac{23}{30}\\[1mm]
\frac{7}{15} & \frac{23}{30} & \frac{23}{30}
\end{pmatrix}.
\]
We have 
 \(d^{(1)}=d^{(2)}=2\) and \(\bQ^{(1)}\) and \(\bQ^{(2)}\) are not simultaneously diagonalizable.
The weighted signal matrix and
variance parameters are
\[
\bQ(w)=w_1\bQ^{(1)}+w_2\bQ^{(2)},
\qquad
\vartheta(w)=2\bigl(w_1^2+w_2^2\bigr).
\]
The two nontrivial eigenvalues of \(\bQ(w)\) are
$
\lambda_\pm(w)
=
\frac{w_1+\frac35 w_2
\pm
\sqrt{\left(w_1-\frac35 w_2\right)^2+\frac95\,w_1w_2}}{2}$.
Hence, the signal-to-noise ratio is
$\mathrm{SNR}(w)
=
\frac{\lambda_+(w)^2}{2(w_1^2+w_2^2)}$.
An optimal weight can be solved numerically $w^*\approx (1,0.51)$.
\end{example}

\subsection{Discussion}\label{sec:discuss}

We present two additional remarks on our main results and their connections to related work.

\paragraph{Detection threshold in the binary non-uniform HSBM}
In the non-uniform HSBM setting, \cite{chodrow2023nonbacktracking} conjectured that, in the binary partition case (\(r=2\)), a spectral method can achieve a Kesten--Stigum-type detection threshold determined by the combined signal-to-noise ratios of all hypergraph layers. Our main results confirm \cite[Conjecture 6.3]{chodrow2023nonbacktracking} and further develop a more efficient spectral approach that applies in a broader setting. More precisely, their method is based on a reduced belief-propagation Jacobian matrix of size \((2Krn)\times(2Krn)\), whereas our approach relies on a reduced non-backtracking operator of size \((2Kn)\times(2Kn)\). Both their method and ours require knowledge of the model parameters.  

For $q$-uniform HSBMs, recent works \cite{gu2023weak,gu2024community} establish that the Kesten–Stigum threshold, conjectured in \cite{angelini2015spectral}, is tight in the two-block case when $q \le 6$ and the degree $d$ is sufficiently large, while it is not tight for certain parameter regimes when $q \ge 5$. We expect that a suitable generalization of their techniques can be used to characterize the tightness and non-tightness of the KS threshold in the non-uniform HSBM setting.

\paragraph{Bethe--Hessian method for non-uniform hypergraphs}
The new Ihara--Bass formula in Proposition~\ref{thm:IharaBass} yields the Bethe--Hessian matrix~\eqref{eq:def_BH}. A similar analysis to that in \cite{stephan2024community} can be applied to \eqref{eq:def_BH} to rigorously justify a spectral method based on \(H(\lambda)\) above the Kesten--Stigum threshold, both for estimating the number of communities and for achieving weak recovery when the signal-to-noise ratio is sufficiently large.

A recent work  \cite{li2026higher} proposed a different Bethe--Hessian matrix, corresponding to the special case \(w^{(k)}=1\) for all \(k \in [K]\) in \eqref{eq:def_BH}. Their empirical results suggest that, with weights \(w^{(k)}=1\), the Bethe--Hessian method is weaker than belief propagation. We conjecture that, with the optimal weights in Example~\ref{example:weight_2}, the Bethe--Hessian method can achieve the same detection threshold as belief propagation.

\paragraph*{Notation} 
We use $\mathrm{polylog}(n)$ to denote $\log^c(n)$, where $c > 0$ is an absolute constant, with the choice of $c$ possibly changing for different appearances of $\mathrm{polylog}(n)$. We say that $a \lesssim b$ if $a$ is bounded from above by $b$ times some absolute constant. We use $A\circ B$ to denote the Hadamard product of two matrices.

\section{Proof of Theorems~\ref{thm:main} and~\ref{th:main_reduced}} \label{sec:proof_main}

\subsection{Spectral structure of $B$}

The proof of Theorems~\ref{thm:main} and~\ref{th:main_reduced} depends on the construction of pseudo-eigenvectors of $B$, as well as a bulk control. Define
\begin{equation}\label{eq:def_R}
    R_{g} = (q_{\max} -1) \sum_{k=1}^K \|P^{(k)}\|_\infty, \quad R_{w} = \frac{\|w\|_\infty}{\sqrt{\vartheta}} \quad \text{and} \quad R = R_g^2 \cdot R_w,
\end{equation}
where $q_{\max} = \max_k q^{(k)}$, which controls both the growth rate of $G$ and its interaction with $w$, and is scale-invariant in $w$. We shall study the matrix $B^\ell$, where
\begin{equation}\label{eq:def_ell}
    \ell = \kappa \log_R(n) \quad \text{with} \quad \kappa < \frac1{12}.
\end{equation}
The lifted eigenvectors $\chi_i \in \R^{\vec{H}}$ are defined for $i \in [r]$ as
\begin{equation}\label{eq:def_chi}
    \chi_i = S^\star \tilde\phi_i, \quad \text{or equivalently} \quad \chi_i(x \to e) = \phi_i(\sigma(x)).
\end{equation}
For $i \in [r_0]$, we define the pseudo left and right eigenvectors  
\begin{equation}\label{eq:def_u_v}
    u_i = \frac{B^\ell J \chi_i}{\sqrt{n}\,\mu_i^{\ell}} \quad \text{and} \quad v_i = \frac{(B^\star)^\ell D_w \chi_i}{\sqrt{n}\,\mu_i^{\ell+1}}.
\end{equation}
They are collected in the matrices $U \in \R^{m \times r_0}$ and $V \in \R^{m \times r_0}$. 

We need a deterministic equivalent to the covariance matrices of $U$ and $V$: for $t \geq 0$, we define the covariance matrices 
\begin{align}
 \bC_{ij}^{(t)} &= \frac{1 - \tau_{ij}^{t+1}}{1 - \tau_{ij}} \delta_{ij} + \tau_{ij} \cdot \frac{1 - \tau_{ij}^{t}}{1 - \tau_{ij}} \phi_i^\star \Pi \bQ_{(q-2)w^{2}/\vartheta} \, \phi_j \label{eq:def_C_l} \\
 [\bC_u^{(t)}]_{ij} &= d\cdot \bC_{ij}^{(t)} +\phi_i^\star \Pi \bQ_{q-2} \phi_j \label{eq:def_C_u} \\
 [\bC_v^{(t)}]_{ij} &= d_{w^2/((q-1)\vartheta)} \cdot \bC_{ij}^{(t)} + \phi_i^\star \Pi \bQ_{w^2(q-2)/((q-1)\vartheta)} \phi_j  \label{eq:def_C_v}
\end{align}
where
\begin{equation}
    \tau_{ij}=\frac{\vartheta}{\mu_i\mu_j}, \quad \text{so that} \quad \tau_{ii} = \tau_i,
\end{equation} 
and the weighted matrices $\bQ_a$ and degrees $d_a$ were defined in eq.~\eqref{eq:def_Q_weighted}, \eqref{eq:def_d_weighted}, respectively. 

Finally, the pseudo-eigenvalues of $B^\ell$ are the $\mu_i^\ell$, that are grouped in the diagonal matrix
\begin{equation}
    \Sigma = \diag(\mu_1, \dots, \mu_{r_0}). \notag
\end{equation}
Let $P_{\img(U)^\bot}$and $P_{\img(V)^\bot}$ be the projection onto the orthogonal complement of the linear space spanned by the column vectors of $U,V$, respectively.
We are now ready to state our main intermediary result:
\begin{proposition} \label{prop_main}
There exists a  constant $c > 0$ such that, with probability at least $1 - O(n^{-c})$, the following inequalities hold:
 \begin{align}
     \norm{U^\star U - \bC_u^{(\ell)}} &\lesssim n^{-1/4}, \label{eq:norm_bound1}\\
     \norm{V^\star V - \bC_v^{(\ell)}} &\lesssim n^{-1/4}, \label{eq:norm_bound2}\\
     \norm{U^\star V - I_{r_0}} &\lesssim   n^{-1/4} ,\label{eq:norm_bound3}\\
     \norm{ V^\star  B^\ell U  - \Sigma^\ell} &\lesssim n^{-1/4} \mu_1^\ell, \label{eq:norm_bound4}\\
     \norm{B^\ell P_{\img(V)^\bot}} &\lesssim \log^{c}(n)\vartheta^{\ell/2},\label{eq:norm_bound5}\\
     \norm{P_{\img(U)^\bot}B^\ell} &\lesssim \log^{c}(n)\vartheta^{\ell/2}\label{eq:norm_bound6}\\ 
     \norm{B^\ell} &\lesssim \log^{c}(n)R^{\ell}\vartheta^{\ell/2} \label{eq:norm_bound7}.
 \end{align}
\end{proposition}

The proof of Proposition~\ref{prop_main} takes up most of the rest of the paper and is done fully in Section~\ref{sec:proof_prop_main}. 

We first begin by showing how this proposition implies our main theorems. 
Informally, if the left (resp. right) eigenvectors of $B$ were the $v_i$ (resp. $u_i$), with associated eigenvalue $\mu_i$, then $B$ would satisfy exactly \eqref{eq:norm_bound3}-\eqref{eq:norm_bound7}. To use a perturbative analysis, we define $B' = U \Sigma^\ell V^\star$ and show that, simultaneously:
\begin{enumerate}[(i)]
    \item $B'$ approximately has the $\mu_i^\ell$ as eigenvalues, with associated eigenvectors $u_i, v_i$;
    \item $\|B' - B^\ell\| \lesssim \log(n)^c \, \vartheta^{\ell/2}$.
\end{enumerate}
Since $\ell \propto \log(n)$ and $\mu_i > \sqrt{\vartheta}$, the ``signal'' coming from $B'$ is larger than the distance from $B'$ to $B^\ell$, allowing the use of perturbative arguments.

Formally, we use the following result from \cite{stephan2020non}:

\begin{lemma}[Theorem  9 from \cite{stephan2020non}]\label{th:bauer_powers}
Let \( \Sigma = \diag(\theta_1, \dots, \theta_r) \) with
\[ |\theta_1| \geq \cdots \geq |\theta_r|, \]
$A \in \dR^{m \times m}$ and \( U, U', V, V' \in \dR^{m \times r} \). We set
\[ {M} = U\Sigma^\ell V^\star  \quand {M'} = U'\Sigma^{\ell'}{(V')}^\star , \]
for two coprime integers \( \ell \leq \ell' \).
  Assume the following holds for both $(\ell, U, V)$ and $(\ell', U', V')$:
  \begin{enumerate}
    \item the matrices \( U, V \) are well-conditioned:\label{item:bauer_power_cond_ii}
          \begin{itemize}
            \item they all are of rank \( r \),
            \item for some \( \alpha \geq 1 \), for \( X \) in \( \{U, V \} \),
                  \[ \norm{X^\star X} \leq \alpha \quand \norm{{(X^\star X)}^{-1}} \leq \beta, \]
            \item for some small \( \delta < 1 \),
                  \[ \norm{U^\star V - I_r} \leq \delta\]
          \end{itemize}
    \item there exist \( \eps > 0 \) such that\label{item:bauer_power_cond_iii}
          \[ \norm{A^\ell - {M}} \leq \eps\]
    \item if we let\label{item:bauer_power_cond_iv}
          \begin{equation}\label{eq:def_sigma0}
               \sigma_0 := 84 r^3\alpha^{9/2}( \eps + 5r \alpha^3 \delta |\theta_1|^\ell),
          \end{equation}
          then
          \begin{equation}\label{eq:power_perturbation_bound}
            \sigma_0 < \ell|\theta_r|^\ell.
          \end{equation}
  \end{enumerate}
 Let
  \[ \sigma := \frac{\sigma_0}{\ell |\theta_r|^\ell}. \]
  Then, the \( r \) largest eigenvalues of \( A \) are close to the \( \theta_i \) in the following sense:  for \( i\in [r] \),
  \[ \left| \lambda_{i} - \theta_{i} \right| \leq 4r\sigma, \]
  and all other eigenvalues of \( A \) have modulus less that \( \sigma_0^{1/\ell} \).
  
  Additionally, for $i \in [r_0]$, let $\xi$ be a unit eigenvector associated to \( \lambda_{\pi(i)} \). Let $\mathrm{Vect}(\Set{u_j \given \theta_j = \theta_i})$ be the vector space spanned by the pseudo-eigenvectors $u_j$ such that $\theta_j=\theta_i$. Then there exists a unit vector  $\zeta \in \mathrm{Vect}(\Set{u_j \given \theta_j = \theta_i})$  such that
  \begin{equation}\label{eq:bauer_eigenvector_bound}
      \norm{\xi - \zeta} \leq \frac{3\sigma}{\delta_i - \sigma},
  \end{equation}
  where $\delta_i:=\min_{\theta_j\neq\theta_i}|\theta_j-\theta_i|$  is the smallest gap between distinct eigenvalues. 
\end{lemma}

\subsection{Proof of Theorem~\ref{thm:main}}
    We check the assumptions of Theorem~\ref{th:bauer_powers}. We define $\ell' = \ell + 1$, and $U', V'$ accordingly. Proposition~\ref{prop_main} also applies to $\ell'$, and hence there exists an event $\mathcal{E}$ with probability at least $1 - O(n^{-c})$ such that \eqref{eq:norm_bound1}-\eqref{eq:norm_bound7} hold for $\ell$ and $\ell'$.

    \textbf{Condition~\ref{item:bauer_power_cond_ii}} In view of \eqref{eq:norm_bound1}-\eqref{eq:norm_bound3}, it suffices to lower and upper bound the eigenvalues of $\bC_u^{(\ell)}, \bC_v^{(\ell)}$. For simplicity, we bound $\bC^{(\ell)}$, since the bounds on $\bC_u, \bC_v$ are similar. This is done through the following lemma, whose proof is deferred to Appendix~\ref{sec:proof_lem_cov_bound}:
    \begin{lemma}\label{lem:cov_bound}
        For any $t \geq 0$, the matrices $\bC^{(t)}, \bC_u^{(t)}, \bC^{(t)}_v$ satisfy
        \begin{align*}
            I_{r_0} \preceq \bC^{(t)} &\preceq \left(1 + \frac{\tau_{r_0}(q_{\max}-1)}{1 - \tau_{r_0}}\right) \cdot I_{r_0} \\
            d_{1/(q-1)} \cdot I_{r_0} &\preceq \bC^{(t)}_u \preceq \left[d \left(1 + \frac{\tau_{r_0}(q_{\max}-1)}{1 - \tau_{r_0}}\right) + d_{q-2} \right] I_{r_0} \\
\vartheta^{-1} d_{w^2/(q-1)^2} \cdot  I_{r_0}&\preceq \bC^{(t)}_v
\preceq \vartheta^{-1} \left[ d_{w^2/(q-1)}  \left(1 + \frac{\tau_{r_0}(q_{\max}-1)}{1 - \tau_{r_0}}\right) + d_{w^2(q-2)/(q-1)} \right] \cdot I_{r_0} 
        \end{align*} 
    \end{lemma}
    In particular, using the Cauchy-Schwarz inequality and $d \geq 1$, we have
    \[ \alpha \lesssim 1 \quand \delta \lesssim n^{-1/4}. \]

    \textbf{Condition~\ref{item:bauer_power_cond_iii}} We have
     \begin{align*}
         \| B^\ell - M \| &= \|(P_{\mathrm{im}(U)} + P_{\mathrm{im}(U)^\bot})B^\ell(P_{\mathrm{im}(V)} + P_{\mathrm{im}(V)^\bot}) - M\| \\
         &\leq \|P_{\mathrm{im}(U)}B^\ell P_{\mathrm{im}(V)} - M \| + \|P_{\mathrm{im}(U)^\bot}B^\ell\| + \|B^\ell P_{\mathrm{im}(V)^\bot}\| 
     \end{align*}
     Since $U, V$ have full rank $r_0$, the expression for $P_{\mathrm{im}(U)}$ and $P_{\mathrm{im}(V)}$ are
     \[ P_{\mathrm{im}(U)} = U(U^\star U)^{-1}U^\star \quad \text{and} \quad P_{\mathrm{im}(V)} = V(V^\star V)^{-1}V^\star. \]
     Hence, if we define
     \[ \tilde U = V(V^\star V)^{-1}  \quad \text{and} \quad \tilde V = U (U^\star U)^{-1} ,\]
     then
     \[ \|P_{\mathrm{im}(U)}B^\ell P_{\mathrm{im}(V)} - M \| \leq \|U\| \|\tilde V^\star B^\ell \tilde U - \Sigma^\ell\| \|V\|. \]
     On the other hand,
     \begin{align*}
         U &= P_{\mathrm{im}(V)} U + P_{\mathrm{im}(V)^\bot} U \\
         &= \tilde U + E_1 + P_{\mathrm{im}(V)^\bot} U
     \end{align*}
     where
     \[\|E_1\| = \| V (V^\star V)^{-1} (V^\star U - I_{r_0})\| \leq \sqrt{\alpha} \beta \delta. \]
     Similarly, we can write
     \[ V = \tilde V + E_2 + P_{\mathrm{im}(U)^\bot} V \quad \text{with} \quad \|E_2\| \leq \sqrt{\alpha}\beta\delta .\]
     Therefore,
     \begin{align*}
         \|\tilde V^\star B^\ell \tilde U - \Sigma^\ell\| &\leq \|V^\star B^\ell U - \Sigma^\ell \| + \| V^\star B^\ell U - \tilde V^\star B^\ell \tilde U \| \\
         &\leq \|V^\star B^\ell U - \Sigma^\ell \| + \|V\| \cdot \left( \|B^\ell \| \cdot \|E_1\| + \|B^\ell P_{\mathrm{im}(V)^\bot} \| \cdot \|U\| \right) \\
         &\phantom{\leq \|V^\star B^\ell U - \Sigma^\ell\|\ } + \|\tilde U \| \cdot \left( \|B^\ell \| \cdot \|E_2\| + \|B^\ell P_{\mathrm{im}(U)^\bot} \| \cdot \|V\| \right)  \\
     \end{align*}
     Combining the above inequalities with \eqref{eq:norm_bound1}-\eqref{eq:norm_bound7} yields
     \[ \|B^\ell - M\| \lesssim \log(n)^c \, \vartheta^{\ell/2} + n^{-1/4} \mu_1^\ell. \]

     \textbf{Condition~\ref{item:bauer_power_cond_iv}} The parameter $\sigma_0$ satisfies
     \[ \sigma_0 \lesssim \log(n)^c \vartheta^{\ell/2} + \mu_1^\ell n^{-1/4} \lesssim \log(n)^c \vartheta^{\ell/2}\]
     for our choice of $\ell$, since
     \[ \vartheta^{\ell/2} \geq \left(\frac{R_g \|w\|_\infty}{R}\right)^\ell \geq  |\mu_1|^\ell n^{-1/4}.\]
     As a result,
     \[ \sigma_0 \leq \ell \mu_{r_0}^\ell\] 
     is equivalent to $\tau_{r_0}^{\ell/2} \lesssim \log(n)^{-(c-1)}$. Since $\ell = \kappa\log_R(n) = \kappa' \log(n)$, we have $\tau_{r_0}^\ell \lesssim n^{-c'}$, and the condition is satisfied for large enough $n$.

     \textbf{Conclusion}

     Since the conditions of Lemma~\ref{th:bauer_powers} are satisfied, its conclusion holds: if we let
     \[ \sigma = \frac{\sigma_0}{\ell \mu_{r_0}^\ell} \lesssim \log(n)^c \tau_{r_0}^\ell\]
     \begin{itemize}
         \item for $i \in [r_0]$, we have
         \[ |\lambda_i(B) - \mu_i| \leq 4r\sigma \lesssim \log(n)^c n^{-c'};  \]
         \item for $i > r_0$,
         \[ |\lambda_i(B)| \leq \sigma_0^{1/\ell}\leq  \sqrt\vartheta+o(1).\]
     \end{itemize}
     This finishes the proof of Theorem~\ref{thm:main}.

\subsection{Proof of Theorem~\ref{th:main_reduced}}

Let $y_1, \dots, y_{r_0}$ be the vectors mentioned in Theorem~\ref{th:main_reduced}. From Proposition~\ref{prop:Bethe_hessian_relation}, we have $y_i = S D_w \xi_i$, where $\xi_i$ is the $i$-th eigenvector of $B$. The second consequence of Lemma~\ref{th:bauer_powers} then implies that there exists a vector $\tilde u_i \in \operatorname{span}(\{u_j : \mu_j = \mu_i\})$ such that 
\[ \left\|\xi_i - \frac{\tilde u_i}{\|\tilde u_i\|}\right\| \lesssim n^{-c}. \] 
For simplicity, we take $\tilde u_i = u_i$, and we let $\tilde y_i = S D_w u_i$. From Lemma~\ref{lem:growth-exploration}, we have $\|S\| \lesssim \log(n)$ with probability at least $1 - n^{-1}$, and hence
\begin{align}\label{eq:approx_y_SDu}
\| \tilde y_i - S D_w u_i \| \lesssim \|w\|_\infty n^{-c'},  \quad \text{where} \quad \tilde y_i = \|u_i\| y_i. 
\end{align}

Using the same methods as in the proof of Proposition~\ref{prop_main}, we show the following Lemma in Appendix~\ref{appendix:proof}:
\begin{lemma}\label{lem:reduced_eigen_computations}
  Let $u_i$ be the vector defined in \eqref{eq:def_u_v}.  For any $i, j \in [r_0]$, we have
    \begin{align}
        \langle SD_w u_i, \tilde \phi_j \rangle &= \sqrt{n} \mu_i \delta_{ij} + O(n^{-c}), \notag\\
        \| S D_w u_i \|^2  &= \mu_i^2 \gamma_i  + O(n^{-c}).\notag
    \end{align}
\end{lemma}
When $i \leq r_0$, we have $\mu_i \gtrsim \|w\|_\infty$, and using the classic inequality $\left\|\frac{x}{\|x\|} - \frac{y}{\|y\|}\right\| \leq 2 \frac{\|x-y\|}{\|y\|}$, when $n$ is large enough, 
\[ \left\| \frac{y_i}{\|y_i\|} -  \frac{SD_w u_i}{\|SD_wu_i\|}\right\| \lesssim n^{-c}. \]
On the other hand,
\[ \| \tilde \phi_i \|^2 = \sum_{x \in [n]} \phi_i(\sigma(x))^2 = \sum_{j \in [r]}  (n \pi_i) \phi_i(j)^2 = n, \]
and hence
\[ \frac{\langle y_i, \tilde \phi_i \rangle}{\|y_i\| \|\tilde \phi_i\|} = \frac{1}{\sqrt{\gamma_i}} + O(n^{-c}), \]
as required. 

On the other hand, if $\mu_i \neq \mu_j$, then from Lemma~\ref{lem:reduced_eigen_computations}, $\tilde \phi_i$ is asymptotically orthogonal to \[\operatorname{span}(\{S D_w u_\ell : \mu_\ell = \mu_j\}),\] and hence to $SD_w\tilde u_j$. This proves the second part of the theorem.

\section{Matrix decomposition and operator norm bounds}\label{sec:matrix_decomp}
We define the hypergraph non-backtracking walk, which can be seen as a vertex-hyperedge sequence starting from a vertex and ending at a hyperedge.
	\begin{definition}[Non-backtracking walk]\label{def:nb_path}
	A \textit{non-backtracking walk} of length $\ell$ in $G=(V,H)$ is a walk $\gamma=(x_0,e_0,x_1, e_1,\cdots, x_{\ell},e_{\ell})$ such that 
	\begin{enumerate}
	    \item $x_0\in e_0$, $x_j\in (e_{j}\setminus \{x_{j-1}\})\cap e_{j-1}$ for $1\leq j\leq \ell$,
	    \item  $e_j\not=e_{j+1}$ for $0\leq j\leq \ell-1$.
	\end{enumerate}
\end{definition}
	$\gamma$ can be seen as a walk $(\gamma_0, \gamma_1,\dots,\gamma_{\ell})$ of length $\ell$ on the space of $\vec{H}$ with $\gamma_j=(x_j,e_j)\in \vec{H}, 0\leq j\leq \ell$.

For a non-uniform hypergraph $G=(V,H)$, let $H(V)$ and $\vec{H}(V)$ be the set of all hyperedges and oriented hyperedges, respectively, in the complete hypergraph indexed by $V$. Define the \textit{incidence matrix} of $G$ in $\mathbb R^{q\times |H(V)|}$ as 
\begin{align*}
    \mathcal A_{x\to e}=\begin{cases}
     1  & \text{if }   e\in H,\\
     0 & \otherwise
    \end{cases}
\end{align*}
for any $x\in e, e\in H(V)$. Here $w_e $ is the weight assigned to $e$. In our HSBM  $G$,   $\mathcal A_{x\to e}$ are independent for different hyperedges $e$. 
We can also write $B$ as 
\begin{align}\label{eq:BtoA}
    B_{(x\to e), (y\to f)}
    &=w_f \mathcal  A_{x\to e}  \mathcal A_{y\to f} \mathbf{1}\{x\in e, f\not=e, y\not=x\},
\end{align}

Define $\Gamma_{(x\to e), (y\to f)}^k$ to be the set of non-backtracking walks of length $k$ denoted by $\gamma=(\gamma_0,\dots,\gamma_{k})$ with $\gamma_0=(x\to e)$ and $\gamma_{k}=(y\to f)$ in $\vec{H}(V)$. According to \eqref{eq:BtoA}, we have for $k\geq 1$,
\begin{align}
    B^k_{(x\to e), (y\to f)}=\sum_{\gamma \in \Gamma_{(x\to e), (y\to f)}^{k}} \mathcal A_{\gamma_{0}}\prod_{s=1}^{k}\mathcal A_{\gamma_s}w{_{\gamma_s}},\notag
\end{align}
where $w_{\gamma_s}=w_e$ if $\gamma_s=(x\to e)$.

\begin{definition}[Tangle-freeness]\label{def:tangle_free} A hypergraph spanned by $\gamma$ is given by all vertices and hyperedges from  $\gamma$. 
We say $\gamma$ is a \textit{tangle-free} path if the hypergraph spanned by $\gamma$ contains at most one cycle. Otherwise, we call $\gamma$ a tangled path. A hypergraph $G$   is called  $\ell$-tangle-free  if  for any $x\in [n]$, there is at most one cycle in $(G,x)_{\ell}$.
\end{definition}

Let $F_{(x\to e), (y\to f)}^k\subseteq \Gamma_{(x\to e), (y\to f)}^k$ be the subset of all tangle-free paths. 
If $G$ is $\ell$-tangle free, then for all $1\leq k\leq \ell$, we must have $B^k=B^{(k)}$, with  
\begin{align}
    B^{(k)}_{(x\to e), (y\to f)}=\sum_{\gamma \in F_{(x\to e), (y\to f)}^{k}}\mathcal A_{\gamma_{0}}\prod_{s=1}^{k}\mathcal A_{\gamma_s}w_{_{\gamma_s}}. \notag
\end{align}
For any oriented hyperedge $(x\to e)\in \vec{H}(V)$ , define the centered random variable
\[\underline{\mathcal A}_{x\to e}=\mathcal A_{x\to e}-\frac{ p_{\underline\sigma(e)}}{\binom{n}{q_e-1}}, \]
where $q_e$ is the size of $e$.
We then define $\Delta^{(k)}$, a centered version of $B^{(k)}$, as 
\begin{align}
  \Delta^{(k)}_{(x\to e), (y\to f)}=  \sum_{\gamma \in F_{(x\to e), (y\to f)}^{k}}  \underline {\mathcal A}_{\gamma_0}\prod_{s=1}^{k}\underline {\mathcal A}_{\gamma_s}w_{_{\gamma_s}}.\notag
\end{align}
From the telescoping sum formula
\begin{align}
    \prod_{s=0}^{\ell} a_s=\prod_{s=0}^{\ell} b_s+\sum_{t=0}^{\ell} \prod_{s=0}^{t-1}b_s(a_t-b_t)\prod_{s=t+1}^{\ell} a_s,\notag
\end{align}
we can decompose $B^{(\ell)}$ as 
\begin{align}\label{eq:Bexpansion}
    B^{(\ell)}_{(x\to e), (y\to f)}&=\Delta^{(\ell)}_{(x\to e), (y\to f)} + \sum_{\gamma \in F_{(x\to e), (y\to f)}^{\ell}} \frac{ p_{\underline\sigma(\gamma_0)}}{\binom{n}{q_{\gamma_0}-1}}\prod_{s=1}^{\ell} \mathcal A_{\gamma_s}w_{\gamma_s} \\
    &+ \sum_{\gamma \in F_{(x\to e), (y\to f)}^{\ell}}\sum_{t=1}^{\ell}  \underline{\mathcal A}_{\gamma_0}\prod_{s=1}^{t-1}\underline{\mathcal A}_{\gamma_s}w_{\gamma_s}\frac{w_{\gamma_t} p_{\underline\sigma(\gamma_t)}}{\binom{n}{q_{\gamma_t}-1}}\prod_{s=t+1}^{\ell} \mathcal A_{\gamma_s}w_{\gamma_s}.\notag
\end{align}
Define the set $F_{t,(x\to e), (y\to f)}^{\ell}\subseteq \Gamma_{(x\to e), (y\to f)}^{\ell}$ is the set of non-backtracking tangled paths $\gamma=(\gamma_0,\dots,\gamma_{\ell})$ such that:
\begin{itemize}
\item if $1 \leq t \leq \ell -1$, both $(\gamma_0,\dots,\gamma_{t-1})$ and $(\gamma_{t+1},\dots,\gamma_{\ell})$ are tangle-free,
\item if $t = 0$ (resp. $t = \ell$),  $(\gamma_1,\dots,\gamma_{\ell})$ (resp. $(\gamma_0, \dots, \gamma_{\ell-1})$) is tangle-free.
\end{itemize}
For $1\leq t\leq \ell$, define $R_t^{(\ell)}$ as
\begin{align}
    (R_t^{(\ell)})_{(x\to e), (y\to f)}&=\sum_{\gamma \in F_{t,(x\to e), (y\to f)}^{\ell}}\underline {\mathcal A}_{\gamma_0}\prod_{s=1}^{t-1}\underline {\mathcal A}_{\gamma_s} w_{\gamma_s}\frac{w_{\gamma_t} p_{\underline\sigma(\gamma_t)}}{\binom{n}{q_{\gamma_t}-1}}\prod_{s=t+1}^{\ell}  \mathcal A_{\gamma_s}w_{\gamma_s}, \notag\\
    (R_0^{(\ell)})_{(x\to e), (y\to f)}&=\sum_{\gamma \in F_{0,(x\to e), (y\to f)}^{\ell}}\frac{ p_{\underline\sigma(\gamma_t)}}{\binom{n}{q_{\gamma_t}-1}}\prod_{s=1}^{\ell}  \mathcal A_{\gamma_s}w_{\gamma_s}, \notag
\end{align}

 For $(x\to e), (y\to f) \in \vec{H}$, define
\begin{align}\label{eq:defK}
    \mathsf{K}_{(x\to e), (y\to f)}&= \frac{p_{\underline{\sigma}(e)}w_f}{\binom{n}{q_e-1}}~\mathbf{1}((x,e)\to (y,f)) ,\\
    \mathsf{K}^{(1)}_{(x\to e), (y\to f)}&= \frac{p_{\underline{\sigma}(f)}w_f}{\binom{n}{q_f-1}}~\mathbf{1}((x,e)\to (y,f)) \label{eq:defK1}
\end{align} 
where $(x,e)\to (y,f)$ represents the non-backtracking property $y\in e, f\not=e$.  We also define  $ \mathsf{K}^{(2)}$ as  
\begin{align}\label{eq:defK2}
\mathsf{K}^{(2)}_{(x\to e), (y\to f)} =\sum_{(x, e) \to (z, g) \to (y, f)} \frac{p_{\underline \sigma(g)} w_gw_f}{\binom{n}{q_g-1}} ,
\end{align}
where the sum is over all $(z,g)\in \vec{H}$ such that $(x,e)\to(z,g)\to(y,f)$ is a non-backtracking walk of length $2$. 
By adding and subtracting $  R_t^{(\ell)}$, to the $t$-th term of the summand in \eqref{eq:Bexpansion} for $0\leq t\leq \ell$, we then have  the following expansion for $B^{(\ell)}$. 
\begin{lemma} \label{lem:expansionBl}
For any $\ell\geq 1$, $B^{(\ell)}$ can be expanded as 
 \begin{align}
   & \Delta^{(\ell)}+\mathsf{K}B^{(\ell-1)}+\sum_{t=1}^{\ell-1} \Delta^{(t-1)}\mathsf{K}^{(2)} B^{(\ell-t-1)}+ \Delta^{(\ell-1)}\mathsf{K}^{(1)} - \sum_{t=0}^{\ell} R_t^{(\ell)}. \label{eq:expansionBl}
\end{align} 
\end{lemma}
\begin{proof}
    By adding $R_0^{\ell}$ and $R_{\ell}^{\ell}$ to the matrix version of \eqref{eq:Bexpansion}, when $t=0$ and $t=\ell$, we are summing over all non-backtracking walks with the last $\ell-1$ step tangle-free and the first $\ell-1$ step tangle-free, respectively.  This gives the $\mathsf{K}B^{(\ell-1)}$ and $\Delta^{(\ell-1)}\mathsf{K}^{(1)}$ terms in \eqref{eq:expansionBl}. The term $\sum_{t=1}^{\ell-1} \Delta^{(t-1)}\mathsf{K}^{(2)} B^{(\ell-t-1)}$ follows from the definition of $\mathsf{K}^{(2)}$ and the expansion of the matrix product. 
\end{proof}

Recall the definition of $\chi_i$ from \eqref{eq:defchi}.
Define 
\begin{equation}\label{eq:def:barD}
    \overline D := \frac{1}{n} \sum_{i=1}^r \mu_i (J\chi_i)\chi_i^\star D_w.
\end{equation}
Accordingly, we define
\begin{align}\label{eq:def:LL}
    L=\mathsf{K}^{(2)}-\overline {D}
\end{align}
and for $1\leq t\leq \ell-1$,
\begin{align}\label{eq:def_Sk_l}
    S_t^{(\ell)}=\Delta^{(t-1)}LB^{(\ell-t-1)}.
\end{align}
Note that 
\begin{align*}
    \overline{D}_{(x\to e),(y\to f)}&=\frac{1}{n} \sum_{i=1}^r \mu_i \sum_{z\in e, z\not=x} \chi_i(z\to e) \chi_i(y\to f)w_f\\
    &=\frac{1}{n} \sum_{i=1}^r \mu_i \phi_i(\sigma (z))\phi_i(\sigma(y)) w_f\\
    &=\frac{1}{n}w_f \sum_{z\in e,z\not=x} \bD_{\sigma(z)\sigma(y)}\\
    &=\frac{1}{n}\sum_{z\in e}\sum_{k=1}^K(q^{(k)}-1) \bD^{(k)}_{\sigma(z)\sigma(y)}w^{(k)}w_f.
\end{align*}
Therefore we expect $\overline{D}$ is a good approximation of $\mathsf{K}^{(2)}$ defined in \eqref{eq:defK2}.

\begin{lemma}\label{lem:Blwnorm}
 For any unit vector $w\in \mathbb C^{\vec H}$,
  \begin{align}
      \|B^{(\ell)} w\|\leq &\| \Delta^{(\ell)}\| +\|\mathsf{K}B^{(\ell-1)}\|+ \frac{1}{n}\sum_{j=1}^r  \sum_{t=1}^{\ell-1} \mu_j\| \Delta^{(t-1)}J \chi_j\|  |\langle   \chi_j, D_w B^{(\ell-t-1)}w\rangle | \notag \\
      &+\sum_{t=1}^{\ell-1} \| S_t^{(\ell)}\| +\|\Delta^{(\ell-1)}\mathsf{K}^{(1)}\| + \sum_{t=0}^{\ell} \|R_t^{(\ell)}\|. \notag 
  \end{align}
\end{lemma}

\begin{proof}
Fix $\ell\ge 1$ and let $w\in \mathbb{C}^{\vec H}$ with $\|w\|=1$.
By Lemma~\ref{lem:expansionBl},  
\begin{align*}
\|B^{(\ell)}w\|
\le\;&
\|\Delta^{(\ell)}\|
+\|\mathsf{K}B^{(\ell-1)}\|
+\sum_{t=1}^{\ell-1}\|\Delta^{(t-1)}\mathsf{K}^{(2)}B^{(\ell-t-1)}w\|
+\|\Delta^{(\ell-1)}\mathsf{K}^{(1)}\|
+\sum_{t=0}^{\ell}\|R_t^{(\ell)}\|.
\end{align*}
For each $t\in\{1,\dots,\ell-1\}$,
\[
\Delta^{(t-1)}\mathsf{K}^{(2)}B^{(\ell-t-1)}w
=
S_t^{(\ell)}w + \Delta^{(t-1)}\overline DB^{(\ell-t-1)}w,
\]
and  $\|S_t^{(\ell)}w\|\le \|S_t^{(\ell)}\|$. Note that for any vector $u$,
\[
\overline Du=\frac{1}{n}\sum_{j=1}^r \mu_j (J\chi_j)\langle \chi_j,D_w u\rangle.
\]
Applying this with $u=B^{(\ell-t-1)}w$ gives
\[
\|\Delta^{(t-1)}\overline DB^{(\ell-t-1)}w\|
\le
\frac{1}{n}\sum_{j=1}^r \mu_j\,\|\Delta^{(t-1)}J\chi_j\|\;
\big|\langle \chi_j, D_w B^{(\ell-t-1)}w\rangle\big|.
\]

Collecting all the above bounds yields the claimed inequality.
\end{proof}

 {It remains to bound all terms present in Lemma \ref{lem:Blwnorm}. For most of them, this is done in the following proposition:}
 \begin{proposition}\label{prop:trace}
 Let $\chi$ be any vector among $\chi_1,\dots, \chi_r\in \mathbb C^{\vec{H}(V)}$.
 For every constant $c_1 > 0$ there exists constant $c_2 > 0$ such that the following is true:
  For sufficiently large $n$, with  probability at least $1-n^{-c_1}$ the following norm bounds hold for all 
  $1\leq k\leq \ell$:  where $\ell=\kappa \log_{R}(n)$ with $0<\kappa \leq 1/6$:
  \begin{align}
      \| \Delta^{(k)}\| &\le \log^c(n) \vartheta^{k/2}, \label{eq:Delta}\\ 
      \| \Delta^{(k)}J \chi\| &\le \log^c(n)n^{1/2}\vartheta^{k/2},\label{eq:Delta_chi}  \\
      \|R_{k}^{(\ell)}\| &\le \log^c(n)n^{-1}R^{\ell}\vartheta^{\ell/2} , \label{eq:Rk}\\
      \| \mathsf{K}B^{(\ell-1)}\| &\le  \log^c(n)\left(\sum_{k = 1}^K \binom{n}{q^{(k)}-1}^{-1/2}\right)R^{\ell}\vartheta^{\ell/2} , \label{eq:KBl}\\
      \|B^\ell\| &\le \log^c(n)R^\ell\vartheta^{\ell/2} , \label{eq:Bl} \\
      \| \Delta^{(\ell-1)} \mathsf{K}^{(1)}\| &\le  \log^c(n)\left(\sum_{k = 1}^K \binom{n}{q^{(k)}-1}^{-1/2}\right) R^{\ell}\vartheta^{\ell/2}, \label{eq:DeltaK}\\
      \|S_{k}^{(\ell)}\|&\le  \log^c(n)n^{-1/2}R^{\ell}\vartheta^{\ell/2}.\label{eq:Sk}
  \end{align}
\end{proposition}

\subsection{Proof of \eqref{eq:Delta}}

Let $\vec{e}_1,\dots,\vec {e}_{2m}$ be oriented hyperedges. 
With the convention that $\vec e_{2m+1}=\vec e_{1}$, we have the following trace expansion bound:
\begin{align}
    \|\Delta^{(k)}\|^{2m}&\leq \tr \left(\Delta^{(k)} {\Delta^{(k)}}^\star \right)^m  \notag \\
    &=\sum_{\vec{e}_1,\dots,\vec{e}_{2m}}\prod_{i=1}^m \Delta^{(k)}_{\vec{e}_{2i-1},\vec{e}_{2i}} \Delta^{(k)}_{\vec{e}_{2i+1},\vec{e}_{2i}} =\sum_{\gamma\in W_{k,m}}\prod_{i=1}^{2m}\prod_{s=0}^k w_{\gamma_{i,s}}\underline{\mathcal A}_{\gamma_{i,s}}, \label{eq:DA}
\end{align}
where $W_{k,m}$ is the set of sequence of paths $(\gamma_1,\dots,\gamma_{2m})$ such that 
$ \gamma_i=(\gamma_{i,0},\dots,\gamma_{i,k})$ is a non-backtracking tangle-free walk of length $k$ for all $i=1,\dots,2m$, and for all $1\leq i\leq m$,  
\begin{align}\label{eq:requivalent}
    \gamma_{2i-1,k}=\gamma_{2i,k}, \quad \gamma_{2i,0}=\gamma_{2i+1,0},
\end{align}
with the convention that $\gamma_{2m+1}=\gamma_1$. Taking the expectation yields,
\begin{align}\label{eq:EDelta}
    \mathbb E\|\Delta^{(k)}\|^{2m} \leq  \sum_{\gamma\in W_{k,m}'}\mathbb E\prod_{i=1}^{2m}\prod_{s=0}^k w_{\gamma_{i,s}}\underline{\mathcal A}_{\gamma_{i,s}},
\end{align}
where $W_{k,m}'$ is the subset of $W_{k,m}$ such that each distinct hyperedge is visited at least twice, and the rest of the terms in \eqref{eq:DA} are zero after taking the expectation. At this point we will further assume that all $w_{\gamma_{i,s}}$ are non-negative. This is because the moments of the entries of $\underline{\cA}$ are non-negative and this implies that the sign of the summand appearing in the right hand side of \eqref{eq:EDelta} can be made positive by replacing all appearances of $w_{\gamma_{i,s}}$ with its absolute value.

In order to upper bound the right-hand side of \eqref{eq:EDelta} we will partition $W_{k,m}$ according to an equivalence relation and then compute the summation for each piece of the partition. To define the equivalence relation we first define the \textit{factor graph representation} of $\gamma$. Given a walk on the hypergraph, $\gamma = \vec e_1\vec e_2\cdots$, which has $h$ distinct oriented hyperedge and $v$ many distinct non-interior vertices, we define the simple graph $G(\gamma) = (V(\gamma),H(\gamma),E(\gamma))$ where $V(\gamma) = \{1,3,\ldots,2v-1\}, H(\gamma) = \{2,4,\ldots,2h\}$. Furthermore, letting $f$ be the unique bijection from the hyperedges and non-interior vertices visited by $\gamma$ to $V(\gamma) \cup H(\gamma)$ in the order in which the hyperedges and non-interior vertices are first visited (i.e. $f^{-1}(2k)$ is the $k$'th distinct hyperedge appearing in the walk and $f^{-1}(2k-1)$ is the $k$'th distinct non-interior vertex appearing in the walk) $2i-1$ and $2j$ are adjacent in $G(\gamma)$ if $x = f^{-1}(2i-1), e = f^{-1}(2j)$ and 
\begin{itemize}
\item $(x,e)$ is an oriented hyperedge 
\item there exists $y$ and $f$ such that $(y,e),(x,f)$ is a step in the walk.
\end{itemize}
We say that $[\gamma] := f(\gamma)$ is the canonical representation of $\gamma$.
We let $\mathcal{W}_{k,m}(v,e,h)$ denote the set of factor graph representations of walks from $W_{k,m}$ that have $v$ vertices, $h$ hyperedges, and $e$ edges. 

\begin{lemma}\label{Lem:encoding-1}

\begin{equation}\label{lem:encoding-1}
    |W_{k,m}(v,e,h)| \le (4k^3m^2)^{2m(e-(v+h)+3)}.
\end{equation}

\end{lemma}

\begin{proof}
Recall that $\gamma = (\gamma_1,\ldots,\gamma_{2m})$ where each $\gamma_i$ is a hypergraph non-backtracking walk of length $k$. Since $\gamma_i$ is a non-backtracking walk, the span of $f(\gamma_i)$ has at most 1 cycle. We begin by giving an initial encoding/decomposition of each $[\gamma_i]$. Note that the first part of $[\gamma_1]$ is the path $v_{1,1},v_{1,1}+1,\ldots,u_{1,1}$, where $v_{1,1} := 1$ and $u_{1,1} \ge 1$. We mark all the edges on this path as \textit{tree edges}. Then, because $[\gamma_1]$ is non-backtracking, the next step takes a non-tree edge from $u_{1,1}$ to a previously visited vertex $w_{1,1}$. This constitutes the first section of the walk. The second section of the walk starts at $w_{1,1}$ and follows the unique path from $w_{1,1}$ to $v_{1,2}$ consisting solely of previously marked tree edges. The walk then follows the unique increasing path $v_{1,2},u_{1,1}+1,\ldots,u_{1,2}$ and all edges traversed along the path are marked as tree edges. The second section concludes with a non-tree edge from $u_{1,2}$ to $w_{1,2}$. Note that in the second section, we can have $v_{1,2} = w_{1,1}$ (the section uses no tree edges) or $u_{1,2} = v_{1,2}$ (the section visits no new vertices). The $i$th section of the non-backtracking walk is defined analogously to that of the second section, with the vertices $v_{1,i},u_{1,i},w_{1,i}$ defined accordingly. This allows us to encode the walk $[\gamma_1]$ as 
\[
[\gamma_1] \cong (v_{1,1},u_{1,1},w_{1,1}),(v_{1,2},u_{1,2}-u_{1,1},w_{1,2}),\ldots, (v_{1,t_1},u_{1,t_1}-u_{1,t_1-1},w_{1,t_1}),
\]
for some $t_1 \ge 1$.
Indeed, first note that the $u's$ can be recovered from $u_{1,1}$ and the difference of $u$'s. Therefore the representation gives us each step from $u_{1,i}$ to $w_{1,i}$ for every $i$. In addition, because the paths from $v_{1,i}$ to $u_{1,i}$ are always of the form $v_{1,i},u_{1,i-1}+1,\ldots,u_{1,i}$ they can be recovered from the $u$'s and $v$'s. Furthermore tree edges are exactly the edges on the paths from $v_{1,i}$'s to $u_{1,i}$, for each $i$, meaning we can recover all of the tree edges and recover the tree paths from $w_{1,i}$ to $v_{1,i+1}$ for every $i$. The remaining $\gamma_i$ can be encoded exactly in the same way as $\gamma_1$, although we crucially retain the marked tree edges from the previous $\gamma_i$'s.  
\par
For a section of a walk, we note that there are at most $vhk$ choices. This is because there are at most $k$ choices for the $u$ and $vh$ choices for the $v$ and $w$ (every edge is incident to a vertex in $V(\gamma)$ and a vertex in $H(\gamma)$).  
\par 
We now bound the number of possible encodings for representing $[\gamma_i]$. Note that there are $v+h-1$ tree edges (since every vertex in $V(\gamma) \cup H(\gamma)$ is visited at least once) and therefore $e-(v+h)+1$ non-tree edges. If $\gamma_i$ has no cycles, then it traverses at most $e-(v+h)+1$ non-tree edges. Otherwise, it traverses some non-tree edge at least twice, implying $\gamma_i$ has a cycle. Therefore the encoding of $\gamma_i$ has at most $e-(v+h)+1$ sections. Therefore there are at most $(vhk)^{e-(v+h)+1}$ encodings for $[\gamma_i]$.
Suppose then that $\gamma_i$ has 1 cycle. Then the edge set of $[\gamma_i]$ is spanned by $P_1eP_2$, where $P_1$ is the first portion of the walk that is strictly a path, $e$ is the unique edge such that $P_1e$ has a cycle and $P_2$ is the last portion of the walk that is strictly a path. Since $P_1$ and $P_2$ are edge-disjoint, the encodings of $P_1$ and $P_2$ must together have at most $e-(v+h)+1$ sections. Now, given the encoding of $P_1$ and $P_2$, $[\gamma_i]$ can be recovered given $e$ and $\tau$, where $\tau+1$ is the number steps between the end of $P_1$ and the beginning of $P_2$ in the walk. Since there are $k$ ways to choose $\tau$ and $vh$ ways to specify the edge, the number of encodings of $\gamma_i$ is at most $(vhk)^{e-(v+h)+2}$.
Since $vhk \ge 2$ the combined number of encodings is at most $(vhk)^{e-(v+h)+3}$. Since $v,h \le 2km$ the conclusion follows. 
\end{proof}
We now bound \eqref{eq:EDelta}.
To do this, we will need to record some technical facts. 
\begin{enumerate}
\item 
Recall that $n_i$ is the number of vertices in community $i$ with $\pi_i = n_i/n$. Because all the model parameters are at most $n^{o(1)}$ we have that

\[
\binom{n}{q^{(j)} - 1} = \frac{n^{q^{(j)}-1}}{(q^{(j)} - 1)!}\left(1 + o\left(n^{-1+o(1)}\right)\right),~~\binom{n_i}{q^{(j)} - 1} = \pi_i^{q^{(j)}-1}\binom{n}{q_i}\left(1 + o\left(n^{-1+o(1)}\right)\right),
\]
for all $i \in [r], j \in [K]$. 
\item 
A Bernoulli random variable $X$ with mean $p$ satisfies 
\[
\dE(X-p)^k \le p(1-p) \le p, 
\]
for all $k \ge 2$. 
\end{enumerate}
\par
We note that explicitly giving a procedure to specify $\gamma$ 
for any particular $\gamma$ in such a way so as to make bounding the right hand side of \eqref{eq:DA} reasonable is rather tedious. Here we will carefully write down the order of choices by which we specify $\underline{\mathcal{A}}_ {\gamma_{i,s}}$.
\begin{enumerate}
    \item Pick the number of vertices $v$ in your factor graph representation. This number satisfies $1 \le v \le 1 + 2mk \le \mathrm{polylog}(n)$.
    \item Pick the number of hyperedges $h$ in your factor graph representation. This number satisfies $1 \le h \le 1 + 2mk \le \mathrm{polylog}(n)$.
    \item Pick the number of edges $e$ in your factor graph representation. This number satisfies $e \ge v+h$ (the walk starts and ends at the same hyperedge so the factor graph representation must be connected and contain a cycle).
    \item Pick a canonical representative $[\gamma] \in \mathcal{W}_{k,m}(v,e,h)$. Implicitly this determines a set of non-negative integers $s_1,\ldots s_h$ where $s_i$ is the number of times $\vec{e_i}$ is visited in the walk. The numbers satisfy $s_1 + \cdots + s_h = 2m(k+1) \le \mathrm{polylog}(n)$. 
    \item Pick non-negative integers $t_1,\ldots t_h$, where $t_i$ counts the number of interior vertices in $\vec{e_i}$ that will also be in some $\vec{e_j}$ with $j < i$. Clearly $t_1 = 0$. The choices also satisfy $t_1 + \cdots + t_h \ge e - h - (v-1)$. This is because each hyperedge has one non-interior vertex and each vertex in the factor graph is first seen as an interior vertex in some hyperedge, save the first vertex (it first appears as a non-interior vertex). Discarding the corresponding edges from the factor graph, the remaining edges count the number of times a vertex appears as a non-interior vertex for a hyperedge but not for the first time.
    \item Pick the uniformity of $\vec{e_1},\ldots, \vec{e_h}$, which we denote as $q_1,\ldots , q_h.$ This should not be confused with $q^{(1)},\cdots,q^{(K)}$, which are the uniformities of the HSBMs in our model. The choice of $q_i$ automatically determines a choice of weights $w_1,\cdots,w_h$ (which should not be confused with $w^{(1)},\cdots,w^{(K)})$ and average degrees $d_1,\cdots,d_k$ (which should not be confused $d^{(1)},\cdots,d^{(k)}$ nor with the definition in \eqref{eq:def_d_weighted}).
    \item For each $\vec{e_i}$ assign the number of vertices of each community type that will appear in it. The assignments must also be consist with the choice of $[\gamma]$ and $t_i$'s. We will denote the assignment for $\vec{e_i}$ as $\sigma_i$.
    \item For each $\vec{e_i}$ pick the vertices that define it, making sure that the choice of vertices agrees with the community assignments specified by $\sigma_i$.
\end{enumerate}
This yields the following bound on \eqref{eq:DA}

\begin{equation}\label{eq:Delta-bound-1}
\sum_{\gamma \in W'_{k,m}} \dE \prod_{i = 1}^{2m}\prod_{s=0}^k w_{\gamma_{i,s}}\underline{\mathcal{A}}_{\gamma_{i,s}} \le \sum_v\sum_h\sum_e\sum_{[\gamma]}\sum_t\sum_{t_i}\sum_{q_i}\underbrace{\sum_{\sigma}\prod_{i = 1}^h\left(\frac{p_{\sigma_i}}{\binom{n}{q_i-1}}w_i^{s_i}\right)}_{=: y},
\end{equation}
where explicit ranges on the indexing variables have been omitted and the range of such a variable possibly depending on the value of variables introduced earlier in the iterated summation. 
We now analyze $y$ (with all indices appearing prior to it fixed). By picking the vertices appearing in each $\vec{e_i}$ in order we can write 

\begin{equation}\label{eq:Delta-bound-2}
y = \underbrace{\sum_{\sigma_1} \frac{p_{\sigma_1}}{\binom{n}{q_1-1}}}_{=: y_1}w_{1}^{s_1}\prod_{i = 2}^h \underbrace{\sum_{\sigma_i} \frac{p_{\sigma_i}}{\binom{n}{q_i-1}}w_{i}^{s_i}}_{=: y_i}
\end{equation}

Expanding $y_i$ we have 

\begin{equation}\label{eq:Delta-bound-3}
y_i \lesssim (q_{\max}h)^{t_i}\sum_{\ell_1+\cdots+\ell_r = q_{i-1-t_i}} \left(\prod_{j = 1}^r \binom{n_j}{\ell_j}\right)\frac{p_{\sigma_i}}{\binom{n}{q_i-1}}w_i^{s_i},
\end{equation}

where we've used the fact that there at most $q_{\max}h$ chosen vertices from previous $\vec{e_j}$ and that, after choosing said vertices, the proportion of the remaining vertices in the $i$th community pool is $\pi_i\left(1+o\left(n^{-1+o(1)}\right)\right)$ for every $i$. We do some further estimating:
\begin{equation}\label{eq:Delta-bound-4}
\begin{split}
    \prod_{j = 1}^r \binom{n_j}{\ell_j}\frac{p_{\sigma_i}}{\binom{n}{q_{i}-1}}w_{i}^{s_i}x 
    &\lesssim \frac{n^{q_{i}-1-t_i}}{\ell_1!\cdots\ell_r!} \cdot \frac{(q_{i}-1)!}{n^{q_{i}-1}} \cdot w_{i}^{s_i} \cdot p_{\sigma_i}\cdot \prod_{j=1}^r\pi_{j}^{\ell_j} \\
    &= \frac{(q_{i}-1-t_i)!}{\ell_1!\cdots\ell_r!} \cdot \frac{(q_{i}-1)!}{(q_{i}-1-t_i)!} \cdot \frac{n^{q_{i}-1-t_i}}{n^{q_{i}-1}} \cdot w_{i}^{s_i} \cdot p_{\sigma_i}\cdot \prod_{j=1}^r\pi_{j}^{\ell_j} \\
    &\lesssim \binom{q_{i}-1-t_i}{\ell_1,\cdots,\ell_r} \cdot \left(\frac{q_{\max}}{n}\right)^{t_i} \cdot w_{i}^{s_i} \cdot p_{\sigma_i}\cdot \prod_{j=1}^r\pi_{j}^{\ell_j}
\end{split}
\end{equation}
We now case on $t_i$. If $t_i = 0$ then we can apply the average degree assumption to our estimate to get

\begin{equation}\label{eq:Delta-bound-5}
y_i \lesssim w_{i}^{s_i}\sum_{\underline{j} \in[r]^{q_{i}-1}} p_{{i},j} \prod_{j \in 
\underline{j}}\pi_j = (q_i-1)d_iw_{i}^{s_i},
\end{equation}

where in the $i = 1$ case we get an additional factor of $n$.
If $t_i \ge 1$ we can bound $p_{\sigma_i}$ by $p_{\max} := \max_k \|P^{(k)}\|_\infty$ to get

\begin{equation}\label{eq:Delta-bound-6}
y_i \lesssim w_{i}^{s_i}p_{\max}\left(\frac{q_{\max}^2h} {n}\right)^{t_i} \lesssim (q_i-1)d_iw_{i}^{s_i} \left(\frac{q^2_{\max}h p_{\max}}{nd_i(q_i-1)}\right)^{t_i},
\end{equation}
where from these estimates we conclude that 

\begin{equation}\label{eq:Delta-bound-7}
y \lesssim n\cdot \left(\frac{q_{\max}^2h p_{\max}}{n d_{\max}}\right)^t\prod_{i=1}^h (q_{i}-1)d_{i}
w_{i}^{s_i},
\end{equation}

where $d_{\max} := \max_k d^{(k)}$.
If we now sum over all choices of uniformities $q_{i}$ we get 

\begin{equation}\label{eq:Delta-bound-8}
\sum_{q_{i}}y \lesssim n \cdot \left(\frac{q_{\max}^2h p_{\max}}{n d_{\max}(q_{\max}-1)}\right)^t \cdot \prod_{i = 1}^h\left(\sum_{k=1}^K(q^{(j)} - 1)d^{(j)}(w^{(j)})^{s_i}\right).
\end{equation}

Next we note that, for any valid choice of $s_i$, we have

\begin{align}\label{eq:Delta-bound-9}
\prod_{i=1}^h\left(\sum_{j = 1}^K (q^{(j)}-1)d^{(j)}(w^{(j)})^{s_i}\right) &\le \left(\sum_{j=1}^K(q^{(j)}-1)d^{(j)}\left(w^{(j)}\right)^2\right)^{\frac{1}{2}(s_1+\cdots+s_h)}\\
&\le \left(\sum_{j=1}^K(q^{(j)}-1)d^{(j)}\left(w^{(j)}\right)^2\right)^{km}.\notag
\end{align}

The first inequality follows from the assumptions that $\sum_j (q^{(j)}-1)d^{(j)} > 1$, that $s_i \ge 2$ for all $i$ and the fact that the weighted $\ell_p$ norm, with weights summing to at least 1, is monotone decreasing in $p$ for $p \ge 1$. Here the weights are the $(q^{(k)}-1)d^{(k)}$ factors and the vector is $\left(w^{(1)},\cdots,w^{(K)}\right)$. We note that the right hand side does not depend on the choice of $[\gamma],t$ or the $t_i$'s. Also note that 

\begin{equation}\label{eq:Delta-bound-10}
\sum_t \sum_{t_i} \left(\frac{q_{\max}^2h p_{\max}}{n d_{\max}(q_{\max}-1)}\right)^t \lesssim \left(\frac{q_{\max}^3h^2 p_{\max}}{n d_{\max}(q_{\max}-1)}\right)^{e-(v+h)},
\end{equation}
where we use the fact that there are at most $\binom{t+h-1}{t-1} \le (q_{\max} h)^t$ ways to pick the $t_i$. From these observations we conclude the following intermediate estimate of \eqref{eq:DA}

\begin{align}\label{eq:Delta-bound-11}
&\sum_{\gamma \in W'_{k,m}}\dE\prod_{i=1}^{2m}\prod_{s = 0}^k w_{\gamma_{i,s}}\underline{\mathcal{A}}_{\gamma_{i,s}} \\
&\lesssim n \cdot \left(\sum_{j=1}^K(q^{(j)}-1)d^{(j)}\left(w^{(j)}\right)^2\right)^{km} \cdot \sum_v\sum_h\sum_e \underbrace{|\mathcal{W}_{k,m}(v,h,e)|\left(\frac{q_{\max}^3h^2 p_{\max}}{n d_{\max}(q_{\max}-1)}\right)^{e-(v+h)}}_{=:z}. \notag
\end{align}

This leaves us with bounding the remaining summation. First note that, from Lemma \ref{Lem:encoding-1} we have 

\begin{align}\label{eq:Delta-bound-12}
z &\le (4k^3m^2)^{2m(e-(v+h)) + 6m}\left(\frac{q_{\max}^2h p_{\max}}{n d_{\max}(q_{\max}-1)}\right)^{e-(v+h)} \\
&= \left(\frac{(4k^3m^2)^{2m}q_{\max}^2h p_{\max}}{n d_{\max}(q_{\max}-1)}\right)^{e-(v+h)}(4k^3m^2)^{6m}. \notag
\end{align}

Since $h,k = {\rm polylog}(n),q_{\max} = n^{o(1)}$ we may pick \[m = c\log(n)/\log \log(n),\] with $c$ a sufficiently small absolute positive constant, such that $\left(\frac{(4k^3m^2)^{2m}q_{\max}^2h p_{\max}}{n d_{\max}(q_{\max}-1)}\right) \le 1/2$. Then 

\begin{equation}\label{eq:Delta-bound-13}
\sum_e z \le 2(4k^3m^2)^{6m}.
\end{equation}

Since this estimate is independent of $v$ and $h$ we conclude the following estimate on \eqref{eq:DA}:

\begin{equation}\label{eq:Delta-bound-14}
\begin{split}
\sum_{\gamma \in W'_{k,m}} \dE \prod_{i = 1}^{2m}\prod_{s = 0}^k w_{\gamma_{i,s}}\underline{\mathcal{A}}_{\gamma_{i,s}} 
&\le  n \cdot 4m^2k^2 \cdot 2(4k^3m^2)^{6m} \cdot \left(\sum_{j=1}^K(q^{(j)}-1)d^{(j)}\left(w^{(j)}\right)^2\right)^{km} \\
&\le n(\mathrm{polylog}(n))^m \left(\sum_{j=1}^K(q^{(j)}-1)d^{(j)}\left(w^{(j)}\right)^2\right)^{km},
\end{split}
\end{equation}

where we used the fact that $k,m = \mathrm{polylog}(n)$. 
Since $\sum_{\gamma \in W'_{k,m}} \dE \prod_{i = 1}^{2m}\prod_{s = 0}^k \underline{\mathcal{A}}_{\gamma_{i,s}}$ is an upper bound on $\dE \|\Delta^{(k)}\|^{2m}$ 
we conclude that 

\begin{equation}\label{eq:Delta-bound-15}
\dE \|\Delta^{(k)}\|^{2m} 
\le n ({\rm polylog}(n))^m \cdot \left(\sum_{j=1}^K(q^{(j)}-1)d^{(j)}\left(w^{(j)}\right)^2\right)^{km}.
\end{equation}
The resulting upper tail estimate on $\|\Delta^{(k)}\|$ follows from Markov's inequality and the estimate for $\dE\|\Delta^{(k)}\|^{2m}$.
\subsection{Proof of \eqref{eq:Delta_chi}}
From the definition of $\chi$ in \eqref{eq:defchi} and 
the fact that $\phi_i$ is a normalized eigenvector, we have $\|J_{\chi_i}\|_\infty \le (q_{\max} - 1)$ for any $i \in [r]$. We write 

\begin{equation}\label{eq:Deltachi-equivalent}
\|\Delta^{(k)}J_{\chi}\|^2 = \sum_{\vec{e},\vec{f},\vec{g}}\Delta^{(k)}_{\vec{e},\vec{f}}\Delta^{(k)}_{\vec{e},\vec{g}}(J_\chi)(\vec{f})(J_\chi)(\vec{g}). 
\end{equation}

Therefore 

\begin{equation}\label{eq:Deltachi-bound-1}
\dE\|\Delta^{(k)}J_{\chi}\|^{2m} \le (q_{\max}-1)^{2m}\sum_{\gamma \in W''_{k,m}} \dE \prod_{i = 1}^{2m}\prod_{s = 0}^k w_{\gamma_{i,s}}\underline{\mathcal{A}}_{\gamma_i,s},
\end{equation}

where $W''_{k,m}$ is the set of paths $\gamma = (\gamma_1,\ldots,\gamma_{2m})$ 
such that $\gamma_i = (\gamma_{i,0},\ldots,\gamma_{i,k})$ is a non-backtracking tangle-free walk and $\gamma_{2i-1,0} = \gamma_{2i,0}$ and every distinct hyperedge is visited at least twice. 
Note that the bound that we have on the number of canonical representatives of walks from $W'_{k,m}$ also applies to $W''_{k,m}$. Therefore we can run the same argument used for bounding $\sum_{\gamma \in W'_{k,m}} \prod_{i = 1}^{2m}\prod_{i =0}^k \underline{\mathcal{A}}_{\gamma_i,s}$ to bound $\sum_{\gamma \in W''_{k,m}} \prod_{i = 1}^{2m}\prod_{i = 0}^k \underline{\mathcal{A}}_{\gamma_i,s}$. 
We now record the two differences between $W''_{k,m}$ and $W_{k,m}$ that appear in the analysis and how they are accounted for.

\begin{itemize}
    \item The graph spanned by $\gamma \in W''_{k,m}$ potentially has up to $m$ distinct components and no component needs to have a cycle. In particular we can only bound $e$ from below by $e \ge (v+h)-\ell$, where $\ell$ is the number of components in the factor graph, and this is best possible (in contrast to representatives from $\mathcal{W}_{k,m}$ where $e \ge v+h$). On the other hand, if the factor graph has $\ell$ components with $v_i,e_i,h_i$ being the number of vertices, edges, and hyper edges in the $i$th component, then $t \ge \sum_i e_i - (v_i-1) - h_i = e - (v+h) + \ell$. Therefore the exponent of $e-(v+h)$ becomes $e-(v+h)+\ell$ in \eqref{eq:Delta-bound-10}, \eqref{eq:Delta-bound-11} and in \eqref{eq:Delta-bound-12}. Therefore, even with the lower bound of $e \ge (v+h)-\ell$, \eqref{eq:Delta-bound-13} still holds.
    \item When enumerating all walks in $W''_{k,m}$ via the factor factor graph representative partitioning and picking the vertices of each hyperedge, there are potentially $m$ hyperedges where all of its vertices need to be picked. This is in contrast to walks in $W'_{k,m}$ where there is exactly one such vertex. Consequently there are $m$ instances between \eqref{eq:Delta-bound-5} and \eqref{eq:Delta-bound-6} where an additional factor of $n$ should appear. The result is that the last inequalitty in \eqref{eq:Delta-bound-15} should have an additional factor of $n^{m-1}$.  
\end{itemize} 

Altogether this implies that 

\begin{equation}\label{eq:Deltachi-bound-2}
\dE\|\Delta^{(k)}J_{\chi}\|^{2m} \le n^{m}\cdot {\rm polylog}(n)\cdot \left(\sum_{j = 1}^K \left(w^{(j)}\right)^2(q^{(j)}-1)d^{(j)}\right)^{km}.
\end{equation}
We conclude the desired tail bound on $\|\Delta^{(k)}J_{\chi}\|$ from Markov's inequality and the bound on $\dE\|\Delta^{(k)}J_{\chi}\|^{2m}$.

\subsection{Proof of \eqref{eq:Rk}}
We write $R_{k}^{(\ell)} = \sum_{j = 1}^K R_{k,j}^{(\ell)}$, where $(R_{k,j}^{(\ell)})_{(x \to e),(y \to f)}$ is the sum of the summands appearing in $(R_{k}^{(\ell)})_{(x \to e),(y \to f)}$ for which the corresponding path $\gamma \in F^{\ell}_{k,(x \to e),(y \to f)}$ has $\gamma_{k}$ as a hyperedge of size $q^{(j)}$. Note that this determines the weight and (unnormalized) edge probability of $\gamma_k$.
For $0\leq k \leq \ell$ we have
\begin{equation}\label{eq:R_k-equivalent}
\begin{split}
\|R_{k,j}^{(\ell)}\|^{2m} &\le \tr\left(R_{k,j}^{(\ell)} R_{k,j}^{(\ell)^\star }\right)^m \\&= \sum_{\gamma \in T'_{\ell,m,k}}\left(\prod_{i=1}^{2m}\prod_{s = 0}^{k-1}w_{\gamma_{i,s}}\underline{\mathcal{A}}_{\gamma_{i,s}}\right)\frac{w_{\gamma_{i,k}}p_{\gamma_{i,k}}}{\binom{n}{q^{(j)}-1}}\left(\prod_{s = k+1}^{\ell} w_{\gamma_{i,s}}\mathcal{A}_{\gamma_{i,s}}\right) \\
&\le \left(\frac{p_{\max}}{\binom{n}{q^{(j)}-1}}\right)^{2m} \sum_{\gamma \in T'_{\ell,m,k}}\prod_{i=1}^{2m}\prod_{s = 0}^{k-1}\left(w_{\gamma_{i,s}}\underline{\mathcal{A}}_{\gamma_{i,s}}\right)w_{\gamma_{i,k}}\left(\prod_{s = k+1}^{\ell} w_{\gamma_{i,s}}\mathcal{A}_{\gamma_{i,s}}\right)
\end{split}
\end{equation}
where $T'_{\ell,k,m}$ is the set of sequences such that $\gamma^1_i = (\gamma_{i,1},\ldots,\gamma_{i,k-1}),\gamma^1_i = (\gamma_{i,k+1},\ldots,\gamma_{i,\ell})$ are non-backtracking tangle-free walks and $\gamma_i = (\gamma_i^1,\gamma_{i,k},\gamma_i^2)$ is non-backtracking tangled and every distinct hyperedge visited in $\gamma$ appears at least twice if it appears at least once in some $\gamma_i^1$. 
In addition they satisfy the following boundary condition: For $1 \leq i \leq m$ we have $\gamma_{2i,0} = \gamma_{2i+1,0},\gamma_{2i,\ell} = \gamma_{2i-1,\ell}$, with the convention that $\gamma_{2m+1} = \gamma_1$. 
We let $\mathcal{T}_{\ell,k,m}(v,h,e)$ denote the number of factor graph representations for walks in $\mathcal{T}_{\ell,k,m}$ with $v$ vertices, $h$ hyperedges, and $e$ vertices.

\begin{lemma}\label{lem:encoding-2}
\begin{equation}\label{eq:encoding-2}
    |\mathcal{T}_{\ell,k,m}(v,e,h)| \le (8\ell^3m^2)^{2m(e-(v+h))+4}.
\end{equation}
\end{lemma}
\begin{proof}
From our proof of Lemma \ref{lem:encoding-1} we can encode, in sequential order, the representatives for a collection of non-backtracking walks that are tangle-free. In particular the number of ways to encode any one of the walks in the collection is at most $(4k^3m^2)^{(e-(v+h)+3)}$. Letting $\rho([\gamma])$ denote the encoding of the representative of $\gamma$ in the factor graph we have
\begin{equation}\label{eq:encoding-rep-2}
\rho([\gamma]) \cong (\rho([\gamma_1^1]),\rho([\gamma_{1,k}]),\rho([\gamma_1^2])),(\rho([\gamma_2^1]),\rho([\gamma_{2,k}]),\rho([\gamma_2^2])),\cdots,(\rho([\gamma_{2m}^1]),\rho([\gamma_{2m,k}]),\rho([\gamma_{2m}^2])),
\end{equation}
where $\rho([\gamma_{i,k}])$, is the encoding of $\gamma_{i,k}$ in the factor graph representation. Indeed this allows us to recover $[\gamma]$ by decoding each $(\rho[\gamma_i^1],\rho([\gamma_{i,k}]),\rho[\gamma_i^2])$ in succession. Thus it suffices to bound the number of ways write the right hand side of \eqref{eq:encoding-rep-2}. 
To that end note that the representative of $\gamma_{i,k}$ in the factor graph is merely an edge (the vertex in $V(\gamma)$ specifies the non-interior vertex and the vertex in $H(\gamma)$ specifies the hyperedge). Therefore, since each $\gamma_i^1$ and $\gamma_i^2$ has length at most $\ell$, the number of choices for $(\rho([\gamma_i^1]),\rho([\gamma_{i,k}]),\rho([\gamma_i^2]))$ is at most 
$(4\ell^3(2m)^2)^{2(e-(v+h)+3)}\cdot(vh) \le (16\ell^3m^2)^{2(e-(v+h)+4)}$. Therefore the total number of ways to choose $\rho([\gamma])$ is at most $(16\ell^3m^2)^{4m(e-(v+h)+4)} \le (3\ell m)^{12(e-(v+h)+4)}$. The bound on $|\mathcal{T}_{k,\ell,m}|$ is immediate.
\end{proof}

We say that a hyperedge in the factor graph is a ``bridge" if the corresponding hyperedge in $\gamma$ is $\gamma_{i,k}$ for some $i$. We can again run the argument used for bounding $\sum_{\gamma \in W'_{k,m}} \prod_{i = 1}^{2m}\prod_{i =0}^k w_{\gamma_{i,s}}\underline{\mathcal{A}}_{\gamma_i,s}$ to bound $\sum_{\gamma \in T'_{\ell,m,k}}\prod_{i=1}^{2m}\left(\prod_{s = 0}^{k-1}w_{\gamma_{i,s}}\underline{\mathcal{A}}_{\gamma_{i,s}}\right)w_{\gamma_{i,k}}\left(\prod_{s = k+1}^{\ell} w_{\gamma_{i,s}}\mathcal{A}_{\gamma_{i,s}}\right)$. We record the differences between $T'_{\ell,m,k}$ and $W'_{k,m}$ appearing in the analysis and how they are accounted for. 
\begin{itemize}
    \item For a given $[\gamma]$ let $x$ be the number of hyperedges in $H(\gamma)$ that strictly appear in the walk as a bridge (i.e. the hyperedge's pre-image in $\gamma$ is a subset of the $\gamma_{i,k}$). For such hyperedges the corresponding $y_i$ from \eqref{eq:Delta-bound-1} does not have a factor of $p_{\sigma_i} \cdot \binom{n}{q_{i}-1}^{-1}$.
    \item In estimating the $y_i$ in \eqref{eq:Delta-bound-5} and \eqref{eq:Delta-bound-6} we can assume worst case estimates of $\max_j p_{i,j}(q_{i}-1)$ and $\max_j p_{i,j}(q_{i}-1)(q^2_{\max}h/n)^{t_i}$ respectively. Because of this and the previous point we can replace the estimate of \eqref{eq:Delta-bound-8} with
    \begin{equation}\label{eq:Deltachi-bound-3}
        \sum_{q_{i}}y \lesssim n \cdot \binom{n}{q_i - 1}^x \cdot \left(\frac{q^2_{\max}h}{n}\right)^t \cdot \prod_{i = 1}^h\left(\sum_{j=1}^K (q^{(j)}-1)(w^{(j)})^{s_i}\|\mathbf P^{(j)}\|_\infty\right).
    \end{equation}
    \item Via an identical argument to \eqref{eq:Delta-bound-9} we have
    \begin{equation}\label{eq:Deltachi-bound-4}
        \prod_{i = 1}^h\left(\sum_{j=1}^K (q^{(j)}-1)(w^{(j)}))^{s_i}\|\mathbf P^{(j)}\|_\infty\right) \le \left(\sum_{j = 1}^K (q^{(j)} - 1)|w^{(j)}|\|\mathbf P^{(j)}\|_\infty\right)^{2km} \le (R_g \|w\|_\infty)^{2km}.
    \end{equation}
    \item One can lower bound $e$ by $e \ge (v+h) + x$. To see this note that after deleting all hyperedges from the factor graph that strictly appear as bridges, the connected components in the remaining factor graph each have a cycle. This is immediate from the assumption on the $\gamma_i$ ($\gamma_i^1$ and $\gamma_i^2$ are tangle-free but $\gamma_i$ is tangled). Since adding back the hyper edges to the factor graph creates $x$ more faces the bound on $e$ follows. As a corollary $t \ge x$.
    \item Combining \eqref{eq:Deltachi-bound-3},\eqref{eq:Deltachi-bound-4} and the new lower bound on $e$ we can follow the rest of the argument for bounding $\|\Delta^{(k)}\|^{2m}$ and deduce the following analogue of \eqref{eq:Delta-bound-15}:
    \begin{equation}\label{eq:Deltachi-bound-5}
    \begin{split}
       & \sum_{\gamma \in T'_{\ell,k,m}} \prod_{i=1}^{2m}\left(\prod_{s=0}^{k-1}w_{\gamma_{i,s}}\underline{\mathcal{A}}_{\gamma_{i,s}}\right)w_{\gamma_{i,k}}\left(\prod_{s = k+1}^{\ell}w_{\gamma_{i,s}}\mathcal{A}_{\gamma_{i,s}}\right)\\
        &\le \binom{n}{q^{(j)} - 1}^x\cdot \left(\frac{{\rm poly} \log n}{n}\right)^x \cdot ({\rm poly}\log(n))^m \cdot (R_g \|w\|_\infty)^{2km} \\
        &\le \binom{n}{q^{(j)} - 1}^{2m}\cdot \left(\frac{{\rm poly} \log n}{n}\right)^{2m} \cdot ({\rm poly}\log(n))^m \cdot (R_g \|w\|_\infty)^{2km}.
        \end{split}
    \end{equation}
\end{itemize}
Therefore $\dE \|R_{k,j}^{\ell}\|^{2m} \le \left(\frac{{\rm poly}\log(n)}{n}\right)^{2m} (R_g \|w\|_\infty)^{2km}$. Since $\dE \|R_k^{\ell}\|^{2m} \le K ^{2m-1}\sum_j 
\dE \|R_{k,j}^{(\ell)}\|^{2m}$ and $R_g \|w\|_\infty \le R \vartheta^{1/2}$, the desired concentration for $\|R_k^{\ell}\|$ then follows from Markov's inequality and the bounds on $\dE\|R_{k,i}\|^{2m}$.

\subsection{Proof of \eqref{eq:KBl}, \eqref{eq:Bl}, and \eqref{eq:DeltaK}}
We first provide the proof of \eqref{eq:KBl}. Note that we can write $\mathsf{K} = \sum_{j = 1}^K \mathsf{K}_j$ where $(\mathsf{K}_j)_{(x \to e, y \to f)} = \mathsf{K}_{(x \to e, y \to f)}$ if $e$ is a hyperedge of size $q^{(j)}$ and is 0 otherwise. In particular we have $\mathsf{K}B^{(k-1)} = \sum_{j = 1}^K \mathsf{K}_jB^{(k-1)}$. Now we analyze the operator norm of an individual $\mathsf{K}_jB^{(k-1)}$. 
We have 
\begin{equation}\label{eq:KBl-bound}
\begin{split}
   & \dE\|\mathsf{K}_jB^{(k-1)}\|^{2m} \\
    &\le \dE\tr(\mathsf{K}_jB^{k-1}(\mathsf{K}_jB^{k-1})^\star )^m \\
    &\le \dE\sum_{\gamma \in W_{k,m,j}}\prod_{i = 1}^m \left((K_j)_{(\gamma_{2i-1,0},\gamma_{2i-1,1})} \left(\prod_{s = 1}^k w_{\gamma_{2i-1,s}}\mathcal{A}_{\gamma_{2i-1,s}}\right) \left(\prod_{s=0}^{k-1} w_{\gamma_{2i,s}}\mathcal{A}_{\gamma_{2i,s}}\right)
    (K_j^\star )_{(\gamma_{2i,k-1},\gamma_{2i,k})}\right) \\
    &\le \frac{(w^{(k)})^{2m}p^{2m}_{\max}}{\binom{n}{q_j - 1}^{2m}} \sum_{\gamma \in W_{k,m,j}}\dE\prod_{i = 1}^m\left(\prod_{s = 1}^kw_{\gamma_{2i-1,s}}\mathcal{A}_{\gamma_{2i-1,s}}\right)\left(\prod_{s = 0}^{k-1}w_{\gamma_{2i,s}}\mathcal{A}_{\gamma_{2i,s}}\right),
\end{split}
\end{equation}
where $W_{k,m,j} \subset W_{k,m}$ consists of the subset of paths for which $\gamma_{2i-1,0}$ and $\gamma_{2i,k+1}$ are hyperedges of size $q^{(j)}$ for all $1 \leq i \leq m$. We can then run the argument used for bounding the corresponding sum in the proof of \eqref{eq:Rk}. The primary difference is that the ``bridge edges" are now the hyperedges at $\gamma_{2i-1,0}$ and $\gamma_{2i,k}$, $x \le m$, and, since we only know that $\gamma$ is a closed walk, $e \ge (v+h)$. In particular we get that 
\begin{equation}
\begin{split}
    &\sum_{\gamma \in W_{k,m,j}} \dE\prod_{i = 1}^m\left(\prod_{s=1}^kw_{\gamma_{2i-1,s}}\mathcal{A}_{\gamma_{2i-1,s}}\right)\left(\prod_{s=1}^{k-1}w_{\gamma_{2i,s}}\mathcal{A}_{\gamma_{2i,s}}\right) \\
    &\le \binom{n}{q^{(k)}-1}^x \cdot \left({\rm poly} \log(n)\right)^m \cdot (R_g\|w\|_\infty)^{2km} \notag \\
    &\le \binom{n}{q^{(k)}-1}^{m} \cdot \left({\rm poly} \log(n)\right)^m \cdot (R_g\|w\|_\infty)^{2km}.
\end{split}
\end{equation}
Therefore $\dE\|\mathsf{K}_jB^{(k-1)}\|^{2m} \le  \binom{n}{q^{(k)}-1}^{-m} \left(\frac{{\rm poly} \log (n)}{n}\right)^{2m}(R_g\|w\|_\infty)^{2km}$. Since \[\dE\|\mathsf{K}B^{(k-1)}\|^{2m} \le K^{2m-1}\sum_j \dE\|\mathsf{K}_jB^{(k-1)}\|^{2m},\] the desired bound for $\|\mathsf{K}B^{k-1}\|$ follows from Markov's inequality and the bound on $\dE\|\mathsf{K}_jB^{(k-1)}\|$. 

$\|\Delta^{(k-1)}\mathsf{K}^{(1)}\|$ can be bounded in exactly the same way as $\|\mathsf{K}B^{(k-1)}\|$, albeit with a slightly different trace expansion that admits the same analysis, with the same concentration. Lastly $\|B^k\|$ can be bounded in same way as $\|\mathsf{K}B^{k-1}\|$ except the factor of $\binom{n}{q^{(k)}-1}^{2m}$ appearing in \eqref{eq:KBl} is no longer present. The resulting bound on $\|B^k\|$ is then \eqref{eq:Bl}.
\subsection{Proof of \eqref{eq:Sk}}
In this subsection we will use $p_{\sigma(g)},w_g,q_g$ to denote the (unnormalized) edge probability, weight, and uniformity of a hyperedge $g$.
Recall 
\begin{align}
    \mathsf{K}^{(2)}_{(x\to e), (y\to f)} &=\sum_{(x, e) \to (z, g) \to (y, f)} \frac{p_{\underline \sigma(g)} w_gw_f}{\binom{n}{q_g-1}}=\sum_{(x, e) \to (z, g) \to (y, f)} \frac{q_g-1}{n-q_g+2}\frac{p_{\underline \sigma(g)} w_gw_f}{\binom{n}{q_g-2}} \notag\\
 \overline{D}_{(x\to e),(y\to f)}
    &=\sum_{z\in e}\sum_{k=1}^K\frac{q^{(k)}-1}{n} \bD^{(k)}_{\sigma(z)\sigma(y)}w^{(k)}w_f.    \notag
\end{align}
We can decompose
\begin{align}
    L= \mathsf{K}^{(2)}-\overline{D}=L^{(1)}+L^{(2)}, \notag
\end{align}
where 
\begin{align}
    L^{(1)}_{(x\to e), (y\to f)}=\sum_{(x, e) \to (z, g) \to (y, f)} \left(\frac{q_g-1}{n-q_g+2}-\frac{q_g-1}{n}\right)\frac{p_{\underline \sigma(g)} w_gw_f}{\binom{n}{q_g-2}} \notag
\end{align}
is the approximation error for the $\frac{q_g-1}{n-q_g+2}$ factor in $\mathsf{K}^{(2)}$,
and 
\begin{align}
    L^{(2)}_{(x\to e), (y\to f)}=\sum_{(x, e) \to (z, g) \to (y, f)} \frac{q_g-1}{n}\frac{p_{\underline \sigma(g)} w_gw_f}{\binom{n}{q_g-2}}-\sum_{z\in e}\sum_{k=1}^K\frac{q^{(k)}-1}{n} \bD^{(k)}_{\sigma(z)\sigma(y)}w^{(k)}w_f, \notag
\end{align}
which the  nonzero term is over  $(z,g)$ such that (i) $z=x$ or $z=y$  (ii) $g=e$ or $g=f$.
Each entry in $L^{(1)}$ and $L^{(2)}$ is of order $O\left( \frac{R}{n}\right)$. Apply the same analysis of \eqref{eq:KBl} and \eqref{eq:DeltaK} to $ S_t^{(\ell)}=\Delta^{(t-1)}LB^{(\ell-t-1)}$ gives the desired bound.

\section{Local non-uniform hypertrees}\label{sec:local_hypertree}

\subsection{Local analysis}\label{sec:local}

The goal of this section is to compare the local neighbourhoods of $G$ to those of a Galton-Watson hypertree defined below.

\begin{definition}[non-uniform Galton-Watson hypertree] 
We define the non-uniform Galton-Watson hypetree as follows:

\begin{itemize}
    \item Start from a root $\rho$ with a given spin $\sigma(\rho)$;
    \item For each mode $k\in [K]$, generate $m_k \sim {\rm Poi}(d^{(k)})$ $q^{(k)}$-uniform hyperedges (with $K$ distinct colors) intersecting only at $\rho$. 
    \item For a mode-$k$ hyperedge we assume a fixed ordering of the $q^{(k)}-1$ associated children $v_1,\ldots, v_{q^{(k)} -1}$, we denote the ordered children by $v = (v_1,\ldots,v_{q_{i-1}})$. Then, we assign a type to each $v_i$ randomly such that 
    \[
    \dP(\underline{\sigma}(v) = \underline{j}) = \frac{1}{d^{(k)}} \cdot p^{(k)}_{\sigma(\rho),\underline{j}} \cdot \prod_{\ell \in \underline{j}} \pi_\ell.
    \]
    \item Repeat the process for each child of $\rho$, treating it as the root of an i.i.d. non-uniform Galton-Watson tree. 
\end{itemize}
We denote the corresponding hypertree as $(T, \rho)$.
\end{definition}

\begin{lemma}\label{lem:growth-poisson-tree}
Let $S_t = |(T, \rho)_t|$. There exist constants $c_1,c_2 > 0$ such that for all $s \ge 0$,

\[
\dP(\exists t \ge 1, S_t > sR_g^t) < c_2e^{-c_1s}.
\]
\end{lemma}
Moreover, there exists a universal constant $c_3 > 0$ such that for every $p \ge 1$,

\[
\dE\left[\max_{t \ge 0} \left(\frac{S_t}{R_g^t}\right)^p\right] \le (c_3p)^{p}.
\]

We also have 
\[
(\dE|(T,\rho)_t|^p)^{1/p} \le 2c_3pR_g^t.
\]
\begin{proof}
We first note that a sum of independent Poisson random variables $X_1,\ldots X_K$ with means $\lambda_1,\ldots,\lambda_K$ is distributed as a Poisson random variable $Y$ with mean $\lambda_1 + \cdots + \lambda_K$. In particular one can construct a coupling of $(X_1,\cdots,X_K,Y)$ where $X_1 + \cdots X_K = Y$ by first sampling $Y$ according to $\Poi(\lambda_1+\cdots+\lambda_K)$ and then sampling $(X_1,\cdots,X_K)$ according to the multinomial distribution with $n = Y$ and $p_i = \lambda_i/(\lambda_1 + \cdots + \lambda_K)$. Furthermore, one can couple a non-uniform Galton-Watson tree $T$ with degrees $(d^{(1)},\ldots,d^{(k)})$ and uniformities $(q^{(1)},\ldots,q^{(k)})$ to a uniform Galton-Watson $\tilde{T}$ tree with degree $(d^{(1)}+\ldots+d^{(k)})$ and uniformity $q_{\max}$, both with common root $\rho$, such that $(T,\rho)_\ell \subset (\tilde{T},\rho)_\ell$ for all $\ell \ge 1$ as follows.
\begin{enumerate}
    \item Sample $\tilde{T}$ with $\rho$ as its root.
    \item Starting from $\rho$ construct $T$ as follows: 
        \begin{enumerate}
            \item When expanding $T$ from the vertex $v$ let $Y_v$ be the number of hyperedges in $\tilde{T}$ rooted at $v$. 
            \item 
            Sample $(X_{v,1} ,\cdots, X_{v,K})$ according to the multinomial distribution with $n = Y_v$ and $p_k = d^{(k)}/(d^{(1)}+\cdots +d^{(k)}).$ Let $\cH_v$ denote the set of hyperedges in $\tilde{T}$ with $v$ as their root. Partition $\cH_v = \cH_{v,1} \cup \cdots \cup_{\cH_{v,K}}$ where $\cH_{v,k}$ has $X_k$ many hyperedges. 
            \item 
            For each hyperedge $\tilde{f}$ in $\cH_{v,k}$ take $f$ to be a $q^{(k)}$-uniform hyperedge containing $v$ that is contained in the vertex set of $\tilde{f}$ and add $f$ to $T$. The underlying type assigned to each new vertex $v$ in $f$ with respect to $T$ is done according to the definition of the non-uniform Galton-Watson tree.  
        \end{enumerate}     
\end{enumerate}
Because the edges in $\cH_{v,k}$ are $q_{\max}$-uniform, (c) can always be done and by the previously mentioned coupling the number of mode-$k$ edges in $T_1$ with some common root $v$ is distributed as $\Poi(d^{(k)})$. Therefore $T$ is distributed as a non-uniform Galton-Watson tree with degrees $(d^{(1)},\ldots,d^{(K)})$ and uniformities $(q^{(1)},\ldots,q^{(K)})$ and $T \subset \tilde{T}$. Letting $\tilde{S}_t$ denote the number of vertices in the $t$-th layer of $\tilde{T}$ we have that $\tilde{S}_t \ge S_t$ for all $t \ge 1$. Lemma 1 then immediately follows from Lemma 9 of \cite{stephan2024sparse} applied to a uniform Galton-Watson tree with uniformity $q_{\max}$ and degree $(d^{(1)} + \cdots + d^{(k)})$.
\end{proof}

\subsection{Growth property for random hypergraphs}
\begin{definition}[Exploration process] Denote $\cS_t(v)$ the set of vertices at distance $t$ from $v$. We consider the exploration process of the neighborhood of $v$ which starts with $\cA_0 = \{v\}$ and at stage $t \ge 0$, if $\cA_t$ is not empty, take a vertex in $\cA_t$ at minimal distance from $v$, denoted by $v_t$, reveal its neighborhood $\cN_{t+1}$ in $[n]\setminus \cA_t$, and update $\cA_{t+1} = (\cA_t \cup \cN_{t+1})\setminus \{v_t\}$. Denote $\cF_t$ the filtration generated by $(\cA_0,\ldots,\cA_t).$ The set of discovered vertices at time $t$ is denoted by $\mathcal{D}_t = \cup_{0 \le s \le t} \cA_s$.

\end{definition}

\begin{lemma}\label{lem:growth-exploration}
    Let $S_t(v) = |\cS_t(v)|$. There exists $c_1,c_2 > 0$ such that for all $s \ge 0$ and for any $v \in [n]$,
    \[
    \dP(\exists t \ge 0, S_t(v) > sR_g^t) \le c_2\exp(-c_1s).
    \]
Consequently, for any $p \ge 1$, there exists $c_3 > 0$ such that
\begin{align*}
    \dE \max_{t \ge 0}\left(\frac{S_t(v)}{R_g^t}\right)^p &\le (c_3p)^p, \\
    \dE \max_{v \in [n], t \ge 0} \left(\frac{S_t(v)}{R_g^t}\right)^p &\le (c_3\log n)^p + (c_3p)^p, \\
    \dE[|(G,x)_t|^p]^{1/p} &\le 2c_3pR_g^t, \\
    \dE \left[\max_{x \in [n]}|(G,x)_t|^p\right]^{1/p} &\le 2c_3(\log n + p)R_g^t.
\end{align*}
\end{lemma}

\begin{proof}
    For the first statement, consider the exploration process. Given $\cF_t$, the number of neighbors of $v_t$ in $[n]\setminus \mathcal{D}_t$ is stochastically dominated by $(q_{\max}-1)\sum_{k = 1}^KV_k$, where 
    \[
    V_k \sim \Bin\left(\binom{n}{q^{(k)} - 1}, \frac{\|\bQ^{(k)}\|_\infty}{\binom{n}{q^{(k)} - 1}}\right).
    \]
    Since a random variable with distribution $\Bin(n,p)$ has a moment generating function that satisfies $M(\lambda) = ((1-p) + pe^{\lambda})^n \le \exp(np(e^\lambda -1))$ we have that
    \[
    \dE\left[\exp\left(\lambda \sum_{i = 1}^n(q_{\max}-1)V_k\right)\right] \le \exp\left((e^{(q_{\max}-1)\lambda} - 1)\sum_{i = 1}^n\|\bQ^{(k)}\|_\infty\right) = M((q_{\max}-1)\lambda),
    \]
    where $M$ is the moment generating function of a Poisson random variable with mean equal to $\left(\sum_{i = 1}^n \|Q^{(i)}\|\right)^{-1}$. Therefore the number of vertices in $\cS_t(v)\setminus \cS_{t-1}(v)$ is stochastically dominated by a sum of $S_{t-1}(v)$ independent non-negative random variables, each with moment generating function bounded by the moment generating function of a Poisson random variable with mean $\left(\sum_{i = 1}^n \|Q^{(i)}\|\right)^{-1}$. Therefore we may apply the analogous strategy from Lemma of \cite{stephan2024sparse} for bounding $S_t(v)$ in a $q_{\max}$-uniform Galton-Watson tree with degree $(q_{\max}-1)\sum_{i = 1}^K \|Q^{(i)}\|_\infty$ to conclude the first statement. The remaining statements follow from the tail estimate on $S_t(v)$, just as in Lemma 10 of \cite{stephan2024sparse}. 
\end{proof}
In the remaining part of this section, we take $\ell = \kappa \log_R n$, for some constant $\kappa$. From here on we will use $G$ to denote the non-uniform random hypergraph with signal matrices $\bQ^{(1)},\ldots,\bQ^{(K)}$.

\begin{lemma}[tangle-freeness]\label{lem:tangle-free} The following is true:
\begin{enumerate}[(i)]
    \item $G$ is $\ell$-tangle-free with probability at least $1 - \tilde O(n^{4\kappa-1})$.
    \item The probability that a given vertex has a cycle in its $\ell$-neighborhood is   $O(\log^2(n) R_g^{2\ell} /n) $.
\end{enumerate}

\begin{proof}
    We use the same argument as in the proof of Lemma 11 from \cite{stephan2024sparse}. Let $\tau$ be the time where all vertices in $(G,x)_\ell$  have been revealed. In the exploration process we only discover $q^{(k)}$-type hyperedges that add $q^{(k)}-1$ vertices at each step. Given $\cF_\ell$, the number of undiscovered $q^{(k)}$-type hyperedges that contain at least 2 vertices in $(G,x)_\ell$ is stochastically dominated by $\Bin\left(m, \frac{\|\bQ^{(k)}\|_\infty}{\binom{n}{q^{(k)} - 1}}\right)$. with $m = |(G,v)_\ell|^2 \cdot \binom{n}{q^{(k)} - 2}$. Since one of these edges must be contained in $G$ in order for $G(v,\ell)$ to not be a tree, we have by Markov's inequality
    
    \begin{align*}
    \dP [(G,x)_\ell \text{ is a not a hypertree}] 
        &\le \dE[\text{bad edges in } (G,x)_\ell] \\
        &\le \dE|(G,v)_\ell|^2\sum_{k = 1}^K \frac{\|\bQ^{(k)}\|_\infty}{\binom{n}{q^{(k)} - 1}}\binom{n}{q^{(k)} - 2} \\
        &\lesssim R_g^{2\ell}\log^2(n)\sum_{k = 1}^K \frac{\|\bQ^{(k)}\|_\infty(q^{(k)}-1)}{n-q^{(k)}} \\
        &\le O(\log^2(n) R_g^{2\ell} /n) .
    \end{align*}
    For the third inequality we used the estimate on $\dE \max_{x \in [n]}|(G,x)_\ell|^2$ from Lemma~\ref{lem:growth-exploration}.
    This proves the second claim. For the first claim we note that there are two ways for at least 2 cycles to be in the $\ell$ neighborhood: (1) There are at least two undiscovered hyperedges, (2) There is a hyperedge containing at least 3 vertices in $(G,v)_{\ell}$. Since the probability that a $\Bin(n,p)$ binomial random variable is at least 2 is at most $n^2p^2$ we have by Markov's inequality that (1) occurs with probability at most 
    
    \[
    \dE|(G,x)_\ell|^4\sum_{k = 1}^K\binom{n}{q^{(k)}-2}^2 \frac{\|\bQ^{(k)}\|_\infty^2}{\binom{n}{q^{(k)} - 1}^2} \le \frac{CR^{4\ell}\log^4(n)q^2\sum_{k = 1}^K\|\bQ^{(k)}\|_\infty^2}{n^2} \le \tilde{O}(n^{4\kappa - 2}).
    \]

    Similarly, (2) occurs with probability at most 
    \[
    \dE|(G,x)_\ell|^3\sum_{k = 1}^K \binom{n}{q^{(k)} - 3}\frac{\|\bQ^{(k)}\|_\infty}{\binom{n}{q^{(k)} - 1}}  \le \frac{CR^{3\ell}\log^{3\ell}(n)q^2\sum_{i = 1}^K\|\bQ^{(k)}\|_\infty^3}{n^2} \le \tilde{O}(n^{3\kappa - 2}).
    \]
    Taking a union bound over all $x \in [n]$ ends the proof. 
\end{proof}

\end{lemma}
\subsection{Coupling between random hypertrees and hypergraphs}
\begin{proposition}\label{prop:coupling}
For every $x \in V$ and $\ell > 0$, we have
\[
d_{\Var}(\cL(G,x)_\ell),\cL(T,x)_\ell)) \lesssim \frac{\log(n) R_g^{2\ell}}{n}.
\]
\begin{proof}
    Let $E_\ell$ be the event under which $|(G,x)_\ell| \le  c\log(n)R^\ell$. By lemma \ref{lem:growth-exploration} and taking $c$ large we have that 
    
    \[
    \dP[E_\ell] \ge 1 - \frac{R_g^\ell}{n}.
    \]

As done in \cite{stephan2024sparse} we note that the coupling between $(G,x)_\ell$ and $(T,x)_\ell$ fails at time $t$ if either 
\begin{enumerate}
\item The revealed hyperedges of $G$ up to time $t+1$ do not form a hypertree.
\item For some $\underline{j} \in [r]^{q-1}$ and some $1 \leq k \leq K$, the coupling between the mode-$k$ edges of type $(u,\underline{j})$ adjacent to $v_t$ fails.
\end{enumerate}
The first case occurs with probability at most $\frac{\log(n)^2 R^{2\ell}}{n}$ by the second conclusion of Lemma \ref{lem:tangle-free}. For the second case we analyze a similar coupling to the one in \cite{stephan2024sparse} with an additional parameter governing the mode of the edge. Specifically we write 
\[
\underline{\tau}^{(k)} = (\tau^{(k)}_1,\cdots,\tau^{(k)}_r) \text{ with } \tau^{(k)}_1 + \cdots + \tau^{(k)}_r = q^{(k)} - 1
\]
and define $\dP_{\underline{\tau}^{(k)}},\dQ_{\underline{\tau}^{(k)}}$ to be the distributions for the number of edges of uniformity $q^{(k)}$ with $\tau^{(k)}_i$ vertices of type $i$ for all $i$, adjacent to $v_t$ in $G (\text{resp. }T)$. From Poisson thinning and the definition of $G$ and $T$, we have that 

\[
    \dP_{\underline{\tau}^{(k)}} = \Bin\left(\prod_{i = 1}^r \binom{n_i(t)}{\tau_i^{(k)}} \cdot \frac{p_{u,\underline{j}}^{(k)}}{\binom{n}{q^{(k)}-1}}\right)
\]
    
\[
    \dQ_{\underline{\tau}^{(k)}} = \Poi\left(p_{u,\underline{j}} \cdot \binom{q^{(k)} - 1}{\tau_1^{(k)},\ldots,\tau_r^{(k)}}\prod_{i = 1}^r \pi_i^{\tau_i^{(k)}}\right)
\]
where $n_k(t)$ denotes the number of vertices of type $k$ in $[n]\setminus D_t$ and $\underline{j}$ is any $q^{(k)}-1$-tuple of [r] for which $i$ appears $\tau_i^{(k)}$ times. By lemma 12 of \cite{stephan2024sparse} we have the bound 
\[
d_{\Var}(\dP_{\underline{\tau}^{(k)}},\dQ_{\underline{\tau}^{(k)}}) \le \frac{c(q^{(k)}-1)! r \|\bQ^{(k)}\|_\infty R_g^\ell}{n} \lesssim \frac{R_g^\ell}n.
\]

Since there at most $\binom{r+q^{(k)}-2}{q^{(k)}-1} \le \frac{(r+q^{(k)}-2)^{q^{(k)}-1}}{(q^{(k)}-1)!}$ possible choices for $\underline{\tau}^{(k)}$ and the coupling must hold for all $1 \leq k \leq K$ and $t \le c_2\log(n)R_g^{\ell}$ we conclude that 
\begin{align*}
    d_{\Var}(\cL((G,x)_\ell),\cL((T,x)_\ell)) &\lesssim \sum_{k = 1}^K\frac{(r+q^{(k)}-2)^{q^{(k)}-1}}{(q^{(k)}-1)!} \cdot \frac{R_g^\ell}n \cdot \log(n) R_g^\ell \\
    &\lesssim \frac{\log(n)R_g^{2\ell}}{n},
\end{align*}
which ends the proof.
\end{proof}
\end{proposition}
\section{Functionals on non-uniform Galton-Watson hypertrees}\label{sec:functional_hypertree}
The goal of this section is to study specific functionals on random hypertrees. Given a tree $T$ with root $\rho$ and a vector $\eta \in \R^r$ and $t \ge 0$ we define 
\begin{equation}\label{eq:def_f_tree}
    f_{\xi,t}(T,\rho) = \sum_{x_t \in \partial(T,\rho)_t} w(\rho \to x_t)\xi(\sigma(x_t)),
\end{equation}

where $\xi(i)$ denotes the $i$th entry of $\xi$ and $w(\rho \to x_t) := \prod_{j = 1}^t w(e_j)$, where $e_1e_2\cdots e_t$ is the unique path in $T$ connecting $\rho$ to $x_t$. Note that $w(e) = w_i$ if $e$ is a $q_i$ uniform edge. 

For convenience we define the 3-dimensional tensor $\bQ^{(3 k)}$ according to
\[
\bQ^{(3, k)}_{ij\ell} = \pi_j\pi_\ell\sum_{\underline{m} \in [r]^{q^{(k)}-3}}p_{ij\ell\underline{m}}^{(k)}\left(\prod_{m \in \underline{m}}\pi_{m}\right)
\]
We will also need weighted versions of the $\bQ^{(3, k)}$; for a vector $a \in \R^k$ we let
\begin{equation}
    \bQ^{3}_a = \sum_{k=1}^K a^{(k)} (q^{(k)} - 1)(q^{(k)}-2) \bQ^{(3, k)} \quad \text{and} \quad \bQ^{(3)} = \bQ_{w\circ w}^{(3)}.\notag
\end{equation}

\subsection{Martingale approach}
Let $(T,\rho)$ be a non-uniform Galtson-Watson hypertree and let $\mathcal{F} = (\mathcal{F}_t)_{t \ge 0}$ be the filtration adapted to the random variables $(T,\rho)_t$. We define $\dE_t$ as the expectation conditioned on the sub-algebra $\cF_t$.
\begin{proposition}\label{prop:martingale}
    Let $(\mu,\phi)$ be an eigenpair of $\bQ$. Define the random process 
    \[
    Z_t := \mu^{-t}f_{\phi,t}(T,\rho).
    \]
The process $(Z_t)_{t \ge 0}$ is an $\mathcal{F}$-martingale, with common expectation $\phi_{\sigma(\rho)}$.
\end{proposition}
The proof will require the following lemma: 
\begin{lemma}\label{lem:functional-idenities}
    Let $x_t \in \partial(T,\rho)_t$ and $\xi \in \R^n$. Then 
    \[
    \dE_{t}\left[\sum_{x_{t+1}:x_t \to x_{t+1}} w(\rho \to x_{t+1}) \xi(\sigma(x_{t+1})) \right] = w(\rho \to x_t) [\bQ \xi](\sigma(x_t)).
    \]
    \[
    \dE_{t}\left[\sum_{x_{t+1}:x_t \to x_{t+1}} w(\rho \to x_{t+1})^2 \xi(\sigma(x_{t+1})) \right] = w(\rho \to x_t)^2 [\bK \xi](\sigma(x_t)).
    \]
\end{lemma}
\begin{proof}
The proof of each identity is analogous to the proof of Lemma 14 from \cite{stephan2024sparse}. We prove the first one. If $e$ is an edge in $T$ with $x_t$ as its parent we define 
\[
X_e = \sum_{y \in e, y \not \in x_t}w(x_t \to x_{t+1})\xi(\sigma(x_{t+1})).
\]
We take $E_k$ to be an enumeration of the edges of mode $k$ in $T$ with $x_t$ as a parent, with $1 \leq k \leq K$. Then we have 
\begin{align*}
    \dE_{t}\left[\sum_{x_{t+1}:x_t \to x_{t+1}} w(\rho \to x_{t+1}) \xi(\sigma(x_{t+1})) \right]
    &= w(\rho \to x_t)\dE_t\left[\sum_{k=1}^K \sum_{e \in E_k }X_e\right] \\
    &= w(\rho \to x_t)\sum_{k=1}^K w^{(k)} \dE_t\left[\sum_{e \in E_k }X_e\right] \\
    &= w(\rho \to x_t)\sum_{k=1}^K w^{(k)} \dE_t\left[\sum_{e \in E_k }X_e\right]. 
\end{align*}
Since $\sum_{e \in E_k} X_e$ is a sum of i.i.d random variables, with the number of such random variables being distributed as a Poisson random variable with mean $d^{(k)}$, we have by Lemma 13 of \cite{stephan2024sparse} that \[\dE_t\left[\sum_{e \in E_k} X_e\right] = d^{(k)} \dE_t X_e.\] 

For the remaining expectation we have 
\begin{align*}
    \dE_t X_e 
    &= \dE_t\left[\sum_{y \in e, y \not \in x_t} \xi(\sigma(x_{t+1}))\right] \\
    &= \sum_{\underline{j} \in [r]^{q^{(k)} - 1}}\frac{1}{d^{(k)}}p_{i\underline{j}}^{(k)}\left(\prod_{\ell \in \underline{j}}\pi_\ell\right) \sum_{\ell \in \underline{j}} \xi_\ell \\
    &= \sum_{k = 1}^{q^{(k)}-1}\sum_{\underline{j} \in [r]^{q^{(k)} - 1}}\frac{1}{d^{(k)}}p_{i\underline{j}}^{(k)}\left(\prod_{\ell \in \underline{j}}\pi_\ell\right)  \xi_{j_k} \\
    &= \frac{(q^{(k)}-1)}{d^{(k)}}\sum_{\underline{j} \in [r]^{q^{(k)} - 1}}p_{i\underline{j}}^{(k)}\left(\prod_{\ell \in \underline{j}}\pi_\ell\right)  \xi_{j_1} \\
    &= \frac{(q^{(k)}-1)}{d^{(k)}}\left(\sum_{j \in [r]} \sum_{\underline{k} \in [r]^{q^{(k)} - 2}}p_{ij\underline{k}}^{(k)}\left(\prod_{\ell \in \underline{k}}\pi_\ell\right)  \pi_j\xi_{j}\right) \\
    &= \frac{(q^{(k)}-1)}{d^{(k)}}\sum_{j \in [r]}\bQ_{ij}^{(k)}\xi_j = \frac{(q^{(k)}-1)}{d^{(k)}}[\bQ^{(k)}\xi](i).
\end{align*}
For the fourth inequality we used the fact that $p_{i\underline{j}}^{(k)}$ and $\prod_{\ell \in \underline{j}} \pi_\ell$ are invariant under permutations of $\underline{j}$ and the fact that the transposition swapping $j_k$ with $j_1$ is a bijection from $[r]^{q^{(k)}-1}$ to $[r]^{q^{(k)}-1}$. The last inequality follows from the fact that $\bQ_{ij}^{(k)} = \bD_{ij}^{(k)}\pi_j$. Therefore 
\[
w(\rho \to x_t)\sum_{k = 1}^K w^{(k)} \dE_t \left[\sum_{e \in E_u} X_e\right] = w(\rho \to x_t)\sum_{k = 1}^K w^{(k)}(q^{(k)}-1)[\bQ^{(k)}\xi](i) = w(\rho \to x_t)[\bQ \xi](i),
\]
proving the first identity. The proof of the second identity is analogous.
\end{proof}
We now prove Proposition~\ref{prop:martingale}:  
\begin{proof}[Proof of Proposition~\ref{prop:martingale}]
    Clearly $\dE_0[Z_0] = \phi(\sigma(\rho))$ so it remains to show that $\dE_t[Z_{t+1}] = Z_t$. Using the first identity of Lemma \ref{lem:functional-idenities} we have that 
    \begin{align*}
    \dE_t[Z_{t+1}] &= \mu^{-(t+1)}\sum_{x_t \in \partial(T,\rho)_t }w(\rho \to x_t)\dE_t \left[\sum_{x_{t+1}:x_t \to x_{t+1}}w(x_t \to x_{t+1})\phi(\sigma(x_{t+1}))\right] \\
    &= \mu^{-(t+1)}\sum_{x \in \partial(T,\rho)_t} w(\rho \to x_t)[\bQ\phi](\sigma(x_t)) \\
     &= \mu^{-t}\sum_{x \in \partial(T,\rho)_t} w(\rho \to x_t)\phi(\sigma(x_t))= Z_t,
    \end{align*}
where for the second equality we used Lemma 14 and for  the fourth equality we used the fact that $(\phi,\mu)$ are an eigen pair for $\bQ$.
\end{proof}
\begin{proposition}\label{prop:eigenpair_identities}
Let $(\mu,\phi)$ and $(\mu',\phi')$ be two eigen pairs of $\bQ$, and let $(Z_t)_{t \ge 0},(Z'_t)_{t \ge 0}$ be associated martingales from Proposition \ref{prop:martingale}. We define the vector 
\[
y^{(\phi,\phi')} = \bK(\phi \circ \phi') + \bQ^{(3)} \times_1 \phi \times_2 \phi',
\]
where $\phi\circ \phi'$ is the Hadamard (entrywise) product of two vectors.

Then
\[
\dE[(Z_{t+1} -Z_t)(Z'_{t+1} - Z'_t)] = (\mu\mu')^{-(t+1)} [\bK^t y^{(\phi,\phi')}](\sigma(\rho)).
\]
As a result, the martingale $(Z_t)_{t \ge 0}$ converges in $\cL^2$. whenever $\mu^2 > \vartheta$. Furthermore 
\[
\dE[Z_tZ_t'] = (\phi \circ \phi')(\sigma(\rho)) + \sum_{s = 0}^{t-1}(\mu\mu')^{-(s+1)}[\bK^sy^{(\phi,\phi')}](\sigma(\rho)).
\]
\end{proposition}
\begin{proof}
We write $\Delta_t := Z_{t+1}-Z_t, \Delta'_t  =(Z'_{t+1}-Z'_t)$. Given $x_t \in \partial(T,\rho)_t$ we write 
\[
A_{x_t} = \sum_{x_{t+1}:x_t \to x_{t+1}} w(\rho \to x_{t+1})\phi(\sigma(x_{t+1})), 
A'_{x_t} = \sum_{x_{t+1}:x_t \to x_{t+1}} w(\rho \to x_{t+1})\phi'(\sigma(x_{t+1})).
\]
We first compute $\dE_t[\Delta'_t\Delta_t]$. We have that 
\[
\dE_t[\Delta'_t\Delta_t] = \sum_{x_t,x'_t \in \partial(T,\rho)_t} E(x_t,x_t'),
\]
where 
$(\mu\mu')^{t+1}E(x_t,x_t')$ is equal to
\[
\dE_t\left[(A_{x_t} - \mu w(\rho \to x_t)\phi(\sigma(x_t))(A'_{x'_t} - \mu'w(\rho \to x_t')\phi(\sigma'(x'_t))\right],
\]
When $x_t \neq x_t'$ by independence of the non-uniform GW tree  the conditional expectation of products becomes a product of conditional expectations. By the first identity of Lemma 14 we have $\dE_t(A_{x_t}) = \mu w(\phi \to x_t)\phi(\sigma(x_t))$. These two facts imply that $E(x,x_t') = {\rm Cov}_t(A_{x_t},A_{x'_t}) = 0$ unless $x_t = x_t'$. Suppose then that $x_t = x'_t$. Note that we may write $A_{x_t} = \sum_{k = 1}^K\sum_{e \in E_k}B_{x_t,e}$, where $B_{x_t,e}$ is the restriction of the sum $A_{x_t}$ to those terms corresponding to vertices adjacent to $x_t$ via $e$, with $E_k$ being an enumeration of all edges in $T$ of mode $k$ with $x_t$ as a parent. Furthermore note that $B_{x_t,e}$ and $B'_{x_t,e'}$ are independent whenever $e \neq e'$, hence ${\rm Cov}_t(B_{x_t,e},B'_{x_t,e'}) = 0$. Therefore 
\begin{align*}
(\mu\mu')^{t+1}E(x_t,x_t) &= {\rm \Cov}_t(A_{x_t},A'_{x_t}) \\
&= \sum_{u,u' = 1}^U\sum_{e \in E_u,e' \in E'_u} {\rm Cov }_t(B_{x_t,e},B'_{x_t,e'}) \\
&= \sum_{k = 1}^K\sum_{e \in E_u}{\rm Cov}_t(B_{x_t,e},B'_{x_t,e}) \\
&= \sum_{k = 1}^K d^{(k)} \dE_t(X_{e_k}X'_{e_k}),
\end{align*}
where $e_k$ is a random $q^{(k)}$ uniform hyper edge containing $x_t$ and  \[X_{e} = \sum_{y \in e, y \neq x_t} w(\rho \to x_t)w^{(k)}\phi(\sigma(y)).\]
Note that the last equality follows from \cite[Lemma 13]{stephan2024sparse}. Next note that 
\begin{align*}
X_{e_k}X'_{e_k} = &\sum_{e_k \ni y \neq x_t} w(\rho \to x_t)^2\left(w^{(k)}\right)^2\phi(\sigma(y))\phi'(\sigma(y)) \\
&+ \sum_{x_t \neq y \in e_u \ni y' \neq x_t} w(\rho \to x_t)^2\left(w^{(k)}\right)^2\phi(\sigma(y))\phi'(\sigma(y')).
\end{align*}
By Lemma 13 of \cite{stephan2024sparse} and the second identity of Lemma \ref{lem:functional-idenities} we have 
\begin{align*}
&\dE_t \left[\sum_{k = 1}^K d^{(k)}\sum_{e_u \ni y \neq x_t} w(\rho \to x_t)^2\left(w^{(k)}\right)^2\phi(\sigma(y))\phi'(\sigma(y))\right] \\
&= \dE_t\left[\sum_{x_{t+1}:x_t \to x_{t+1}}w(\rho \to x_{t+1})^2(\phi \circ \phi')(\sigma(x_{t+1})))\right] \\
&= w(\rho \to x_t)^2[\bK(\phi \circ \phi')](\sigma(x_t)).
\end{align*}
Next by repeating our analysis of $X_e$ from Lemma \ref{lem:functional-idenities} to $\dE_t\left[\sum_{x_t \neq y \in e_u \ni y' \neq x_t}\phi(\sigma(y))\phi'(\sigma(y'))\right]$ we get
\begin{align*}
    &d^{(k)} \dE_t\left[w(\rho \to x_t)^2\left(w^{(k)}\right)^2\sum_{x_t \neq y \in e_u \ni y' \neq x_t}\phi(\sigma(y))\phi'(\sigma(y'))\right] \\
    &= w(\rho \to x_t)^2\left(w^{(k)}\right)^2(q^{(k)}-1)(q^{(k)} - 2)\sum_{j,k \in [r]}\sum_{\bar{\ell} \in [r]^{q^{(k)} - 3}}p_{ijk,\underline{l}}\left(\prod_{m \in \underline{\ell}} \pi_m\right) \pi_j\pi_k\phi_j\phi'_k \\
    &= w(\rho \to x_t)^2\left(w^{(k)}\right)^2
    (q^{(k)} - 1)(q^{(k)} - 2)[Q^{(3, k)} \times_1 \phi \times_2 \phi'](i)
\end{align*}
All together this gives 
\begin{align*}
  \frac{(\mu\mu')^{t+1}E(x_t,x_t)}{(w_\rho \to x_t)^2} &= [\bK(\phi \circ \phi')](\sigma(x_t)) + \sum_{k = 1}^K \left(w^{(k)}\right)^2(q^{(k)}-1)(q^{(k)}-2)[Q^{(3, k)} \times_1 \phi \times_2 \phi'](\sigma(x_t)) \\
  &= y^{(\phi,\phi')}(\sigma(x_t)).  
\end{align*}

Applying the second identity of Lemma \ref{lem:functional-idenities} with $\xi =  y^{(\phi,\phi')}$ and summing over all vertices in $\partial(T,\rho)_t$ gives
\begin{align*}
(\mu \mu')^{t+1}\dE_t[\Delta_t\Delta'_t] &= \dE_t\left[\sum_{x_{t+1} \in \partial(T,\rho)_{t+1}} w(\rho \to x_{t+1})^2y^{(\phi,\phi')}\sigma(x_{t+1})\right]\\
&= \sum_{x_t \in \partial(T,\rho)_t} w(\rho \to x_t)^2[\bK y^{(\phi,\phi')}](\sigma(x_{t}))
\end{align*}
Since $\dE[\Delta_t \Delta_t'] = \dE_0[\dE_1[\cdots \dE_t[\Delta_t\Delta'_t] \cdots]]$ we may iterate the above procedure and conclude that 
$\dE[\Delta_t \Delta_t'] = (\mu\mu')^{-(t+1)}\bK^t y^{(\phi,\phi')}(\sigma(x_t))$.
\par
For the convergence part, since both $Z_t$ and $Z_t'$ are $\mathcal{F}$-martingales, if $t \le t'$, 
\[
\dE[Z_tZ_t'] = Z_0Z_0' + \sum_{s = 0}^{t-1} \dE[\Delta_t\Delta_t'] = (\phi \circ \phi')(\sigma(\rho)) + \sum_{s = 0}^{t-1}(\mu\mu')^{-s}\bK^s y^{(\phi,\phi')}.
\]
By Doob's second martingale convergence theorem, $(Z_t)_{t \ge 0}$ converges in $\cL^2$ whenever $\dE[Z_t^2]$ is uniformly bounded. From the definition of $\dE[Z_t^2]$ it follows that this happens whenever $\mu^2 > {\rm rad}(\bK) = \vartheta$.
\end{proof}
\subsection{A top-down approach: Galton-Watson transforms}
The definition of a non-uniform Galton-Watson hypertree is such that the law of a Galton-Watson tree depends only on $\sigma(\rho)$. Therefore for a functional $f : \mathcal{G_*} \to \R$ we can define its Galton-Watson transofrm $\overline{f} : [r] \to \R$ via 
\begin{equation}\label{eq:def_gw_transform}
    \overline{f}(k) = \dE_{\sigma(\rho) = k}[f(T,\rho)].
\end{equation}

\begin{proposition}\label{prop:gw_transform:computations}
Let $(\mu,\phi),(\mu',\phi')$ be eigenpairs of $\bQ$ and let $f_{\xi,t}$ be the functional defined in \eqref{eq:def_f_tree}. Define $F_{\phi,t} = (f_{\phi,{t+1}}-\mu f_{\phi,t})^2$ Then 
\begin{align*}
    \overline{f_{\phi,t}} &= \mu^t \phi, \\
    \overline{f_{\phi,t}f_{\phi',t}} &= (\mu\mu')^t\left(\phi \circ \phi' + \sum_{s = 0}^{t-1} \frac{\mathbf{K}^s y^{(\phi,\phi')}}{(\mu\mu')^{s+1}}\right), \\
    \overline{F_{\phi,t}} &= \mathbf{K}^t y^{(\phi,\phi)}.
\end{align*}
\end{proposition}
\begin{proof}
    From Proposition \ref{prop:martingale} we have that $Z_t = \mu^{-t}f_{\phi,t}$ is martingale. Therefore 
    \[
    \overline{f_{\phi,t}} = \dE_{\sigma(\rho) = k}[f_{\phi,t}] = \dE_{\sigma(\rho) = k}[\mu^t Z_t] = \mu^t \dE_{\sigma(\rho) = k} [Z_0] = \mu^t\phi(k). 
    \]
    The identity for $\overline{f_{\phi,t}f_{\phi',t}}$ is the last conclusion of Proposition \ref{prop:eigenpair_identities}. The identity for $F_{\phi,t}$ follows from the identity of covariances of martingale differences from Proposition \ref{prop:eigenpair_identities}, since here $F_{\phi,t} = (f_{\phi,t+1} - \mu f_{\phi,t})^2 = (\mu^{t+1})^2(Z_{t+1}-Z_t)^2$.
\end{proof}

The results of Proposition~\ref{prop:gw_transform:computations} can also be recovered using a ``top-down'' approach, outlined below. For a vector $a \in \R^k$, we define the following operations on functionals:
\begin{align}
    \partial_a f(T, \rho) &= \sum_{e \ni \rho} a(e)\sum_{\substack{\rho_1 \in e \setminus \{\rho\}  }} f(T_1, \rho_1) \label{eq:gw_functionals:def_one_step} \\
    [f_1, f_2]_a(T, \rho) &= \sum_{e \ni \rho} a(e) \sum_{\substack{\rho_1, \rho_2 \in e \setminus \{\rho\}}} f_1(T_1, \rho_1) f_2(T_2, \rho_2). \label{eq:gw_functionals:def_one_step_correlation}
\end{align}
where $a(e) = a^{(k)}$ whenever $e \in E_k$ and $T_i$ is the subtree rooted at $\rho_i$. If no weight $a$ is given, we take $a = \mathbf{1}$.

The proof of Lemma~\ref{lem:functional-idenities} and Proposition~\ref{prop:eigenpair_identities} for $t = 0$ directly imply the following result:
\begin{lemma}\label{lem:gw_functionals:one_step_result}
    Let $f_1, f_2$ be functionals on hypertrees. For any weight vectors $a_1, a_2$, the following identities hold:
    \begin{align*}
        \overline{\partial_{a_1} f_1} &= \bQ_{a_1} \overline{f_1} \\
        \overline{[f_1, f_2]_{a_1}} &= \bQ_{a_1} (\overline{f_1 \cdot f_2}) + \bQ^{(3)}_{a_1} \times_1 \overline{f_1} \times_2 \overline{f_2} \\
        \overline{\partial_{a_1} f_1 \cdot \partial_{a_2}f_2} &= \bQ_{a_1} \overline{f_1} \cdot \bQ_{a_2} \overline{f_2} + \overline{[f_1, f_2]_{a_1 \circ a_2}}.
    \end{align*}
\end{lemma}
We leave it to the reader to check that these relations, along with the recurrence $f_{\phi, t+1} = \partial_{w} f_{\phi, t}$, imply Proposition~\ref{prop:gw_transform:computations}. However, the ``bottom-up'' approach is necessary to ensure that $(Z_t)_{t \geq 0}$ is a convergent martingale w.r.t $\cF$.

\subsection{Spatial averaging of hypergraph functions}\label{sec:hypergraph_functional}
We say a function $f$ from $\mathcal{G}_*$  to $\R$ is $t$-local if $f(G,o)$ is only a function of $(G,o)_t$. We first provide a moment inequality for hypergraph functions the Efron-Stein inequality. 
\begin{lemma}\label{lem:spatial-average-1}
    Let $f,\psi : \mathcal{G}_* \to \R$ be two $t$-local functions such that $|f(g,o)| \le \psi(g,o)$ for all $(g,o) \in \mathcal{G}_*$ and $\psi$ is a non-decreasing function with respect to hyperedge inclusion. Then there exists a universal constant $c > 0$ such that for all $p \ge 2$,
    \[
    \left(\dE\left[\sum_{o \in [n]} f(G,o) - \dE(f(G,o))\right]^p\right)^{1/p} \le c p^{3/2} \sqrt{n}\, R_g^t\left(\dE\left[\max_{o \in [n]} \psi(G,o)^{2p}\right]\right)^{\frac1{2p}}
    \]
\end{lemma}
\begin{proof}
    For every $x \in [n]$ we define $H_x$ to be the collection of undirected hyper edges of $H$ such that the largest vertex contained in every $e$ in $H_x$ is $x$. Clearly $H = \cup_{x \in [n]} H_x$, the vector $(H_1,\cdots, H_n)$ has independent entries and there is a measurable function $F$ such that 
    \[
    Y := \sum_{o \in [n]} f(G,o) = F(H_1,\ldots,H_n).
    \]
    Next we define $G_x$ to be the subgraph of $G$ whose hyperedge set is $H \setminus H_x$ and define 
    \[
    Y_x = \sum_{o \in [n]} f(G_x,o) = F(H_1,\ldots, H_{x-1},\emptyset, H_{x+1},\ldots,H_n).
    \]
    Clearly $Y_x$ is $H \setminus H_x$ measurable. By the Efron-Stein inequality (see Theorem 15.6 of \cite{boucheron2013concentration}), for all $p \ge 2$,
    \[
    (\dE|Y - \dE[Y]|^p) \le (c\sqrt{p})^p\dE\left[\left(\sum_{x \in [n]} (Y-Y_x)^2\right)^{p/2}\right].
    \]
    Since $f$ is $t$-local $|f(G,o) - f(G_x,o)|$ is only non-zero when $o$ is at most distance $t$ from $x$. Therefore 
    \[
    |Y-Y_x| \le \sum_{o \in (G,x)_t} |f(G,o) - f(G_x,o)|\le \sum_{o \in (G,x)_t} (\psi(G,o) + \psi(G_x,o)) \le 2|(G,x)_t|\max_{o \in [n]}\psi(G,o).
    \]
    where the last inequality uses monotonicity of $\psi$ with respect to hyperedge inclusion. Since $x \mapsto x^{p/2}$ is convex in $x$, since $p/2 \ge 1$, we have that $\left(\sum_i |x_i|^2\right)^{p/2} \le n^{p/2 - 1}\sum_{i} |x_i|^{p}$. Therefore 
    \begin{align*}
    \dE\left[\left(\sum_{x \in [n]} (Y-Y_x)^2\right)^{p/2}\right] &\le n^{p/2 - 1}2^p \dE\left[\sum_{x \in [n]} |(G,x)_t|^p \cdot \max_{o \in [n]} |\psi(G,o)^p|\right] \\
    &\le n^{p/2} 2^p\left(\max_{x \in [n]}\dE\left[|(G,x)_t|^{2p}\right]\right)^{1/2}\left(\dE\left[\max_{o \in [n]} |\psi(G,o)|^{2p}\right]\right)^{1/2}
    \end{align*}
    Applying the bounds from Lemma \ref{lem:growth-exploration} for $\dE|(G,x)_t|^{2p}$ and $\dE \max_{o \in [n]}|(G,o)_t|^{2p}$, gives the desired bound. 
\end{proof}
\begin{lemma}\label{lem:spatial-average-2}
    Let $t \geq 0$ and $f : \mathcal{G}_* \to \R$ be a $t$-local function. 
    Then for all $x \in [n]$
    \[
    |\dE f(G,x) - \dE f(T,x)| \lesssim \frac{R_g^{t}\log(n)}{\sqrt{n}}\left(\sqrt{\dE|f(G,x)|^2} \lor \sqrt{\dE|f(T,x)|^2}\right)
    \]
\end{lemma}
\begin{proof}
    Recall, from Proposition \ref{prop:coupling}, that there exists a coupling between $G$ and $T$ for which 
    \[
    d_{\Var}(\cL(G,x)_t, \cL(T,x)_t) \lesssim \frac{\log(n)^2 R_g^{2t}}{n}.
    \]
    Let $\mathcal{E}_t$ denote the event that, with respect to this coupling, that $(G,x)_t = (T,x)_t$. Since $f$ is $t$-local the difference is 0 when $\mathcal{E}_t$ occurs. Therefore 
    \begin{align*}
        |\dE f(G,x) - \dE f(T,x)| 
        &\le \dE |f(G,x) - f(T,x)| \\
        &\le \dE |(f(G,x) - f(T,x))\textbf{1}_{\mathcal{E}_t^c}| \\
        &\le \sqrt{  \dP(\mathcal{E}_t^c)\cdot \dE |(f(G,x) - f(T,x))|^2} \\
        &\le \sqrt{  \dP(\mathcal{E}_t^c)\cdot (\dE |f(G,x)|^2 + |f(T,x))|^2)} \\
        &\lesssim \sqrt{\frac{\log(n)^2 R_g^{2t}}{n}}\left(\sqrt{\dE|f(G,x)|^2} \lor \sqrt{\dE|f(T,x)|^2}\right) 
    \end{align*}
\end{proof}
\begin{lemma}\label{lem:spatial-average-3}
Let $t > 0$ and $f : \mathcal{G}_* \to \R$ be a $t$-local function such that $f(g,o) \le \alpha|(g,o)_\ell|^\beta$ for some $\alpha,\beta$ not depending on $g$.  Then, for all $s \ge 2$, with probability $1 - n^{-s}$,
\[
\left|\sum_{x \in [n]} f(G,x) - \dE \left(\sum_{x \in [n]} f(T,x)\right)\right| \le 
(sc)^{\beta + 3/2} \alpha  (\log n)^{3/2+\beta}  R_g^{t(1+\beta)} \sqrt{n},
\]     
where $c > 0$ is a universal constant.
\end{lemma}
\begin{proof}
    To begin we note that 
    \begin{align*}
    &\left|\sum_{x \in [n]} f(G,x) - \dE \left(\sum_{x \in [n]} f(T,x)\right)\right| \\
    &\le
    \left|\sum_{x \in [n]} f(G,x) - \dE \left(\sum_{x \in [n]} f(G,x)\right)\right| + 
    \sum_{x \in [n]} \left|\dE f(G,x) - \dE f(T,x)\right|.
    \end{align*}
    For the second summation we may apply Lemma \ref{lem:spatial-average-2} and the assumption $f(g,o) \le \alpha|(g,o)|_t^\beta$ to get
    \begin{align*}
        \sum_{x \in [n]} \dE|f(G,x) - f(T,x)|
        &\lesssim R_g^{t}\log(n) \sqrt{n}\cdot \alpha \cdot \left(\sqrt{\dE|(G,x)_t|^{2\beta}} \lor \sqrt{\dE|(T,x)_t|^{2\beta}}\right) \\
        &= R_g^{t}\log(n) \sqrt{n}\cdot \alpha  ((2cp)R_g^{t})^\beta \\
        &\lesssim \alpha (cp)^\beta\log(n)R_g^{(1 + \beta)t} \sqrt{n}.
    \end{align*}
    For the first summation we may apply Lemma \ref{lem:spatial-average-1}, with $\phi(g,o) := \alpha|(g,o)_t|^\beta$ and the upper bound on $\dE \max_{x \in [n]} |(G,x)_t|^p$ from Lemma \ref{lem:growth-exploration} to get
    \begin{align*}
        \left(\dE\left|\sum_{x \in [n]} f(G,x) - \dE \left(\sum_{x \in [n]} f(G,x)\right)\right|^p\right)^{1/p} &\lesssim cp^{3/2} \sqrt{n} R_g^t \cdot \left(\dE \left[ \max_{x \in [n]} \left(\alpha|(G,x)_t|^\beta\right)^{2p} \right]\right)^{1/(2p)}, \\
        & \le \alpha c p^{3/2} \left(\dE\left[\max_{x \in [n]} |(G,x)_t|^{2p\beta}\right]\right)^{1/(2p)} \\
        & \le \alpha cp^{3/2} \sqrt{n} R_g^t (2c)^{\beta}R^{2p\beta t/(2p)}(\log n + p)^\beta
    \end{align*}

    For any random variable $X$ and any $a > 0$, we have by Markov's inequality
    \[ \dP(X > a \dE[X^p]^{1/p}) < a^{-p}. \]
    
Taking $a = e$ and $p = s\log n$ we have, with probability at least $1 - n^{-s}$, that 
\[
\left|\sum_{x \in [n]} f(G,x) - \dE \left(\sum_{x \in [n]} f(G,x)\right)\right| \le 
\alpha c^\beta e (s \log n)^{3/2 + \beta} \sqrt{n} R_g^{t(1 + \beta)}  \lesssim \alpha \log(n)^{3/2 + \beta} \sqrt{n} R_g^{t(1 + \beta)}.
\]
Combining the two bounds via triangle inequality we may conclude the desired upper bound. 
\end{proof}

We can also rephrase the above lemma in terms of the Galton-Watson transform defined in \eqref{eq:def_gw_transform}:

\begin{proposition}\label{prop:concentration:gw_transform}
    Let $f$ be a  $t$-local function $f(g,o) \le \alpha|(g,o)_\ell|^\beta$ for some $\alpha,\beta$ not depending on $g$. Then with probability at least $1 - n^{-2}$,
    \[
    \left|\sum_{x \in [n]} f(G,x) - n\sum_{i \in [r]} \pi_i \bar{f}(i)\right| \lesssim 
\alpha  (\log n)^{3/2+\beta}  R_g^{t(1+\beta)} \sqrt{n}.
    \]
\end{proposition}
\begin{proof}
    From the definition of $\bar{f}$ and $\pi$ we have
    \[
    \frac{1}{n}\sum_{x \in [n]} \dE[f(T,x)] = \sum_{i \in [r]} \pi_i \bar{f}(i).
    \]
The proposition is then an immediate consequence of Lemma \ref{lem:spatial-average-3} with $s=2$.
\end{proof}

\subsection{Translating hypertree processes to $B$}\label{sec:hypertree_B}

We now leverage the bounds from the previous section to obtain structural results on $B$. We denote by $\Gamma^t_{(x \to e)}$ (resp. $\Gamma^t_x$) the set of all non-backtracking paths (as in Def.~\ref{def:nb_path}) in $G$ starting at $(x \to e)$ (resp. $x$). For $\gamma = (x_0, e_0, \dots, e_t, x_t, e_{t})$, we define its weight $w(\gamma)$ as
\[ w(\gamma) = \prod_{i=0}^t w(e_i). \]

\begin{lemma}\label{lem:hypertree_process_1}
    Let $\ell \leq \kappa\log_R(n)$. For any $i, j \in [r]$ and $t \leq 3\ell$, with probability $1 - O(n^{12\kappa-1})$,
    \[ \left|\langle \chi_j, D_w B^t J \chi_i \rangle - n \mu_j^{t+1} \delta_{ij}\right| \lesssim \log(n)^{5/2} \|w\|^{t+1}_\infty R_g^{2t+2} \sqrt{n}. \]
\end{lemma}

\begin{proof}
    Define the functional
    \begin{equation}
        f(g, o) = \mathbf{1}_{(g, o)_t \text{ is tangle-free}}\, \phi_j(\sigma(o)) \sum_{\gamma \in \Gamma^t_o} w(\gamma) \sum_{x_{t+1} \in e_t \setminus \{x_t\}} \phi_i(\sigma(x_{t+1})).\notag
    \end{equation}
    The functional $f$ is $(t+1)$-local, and whenever $(g, o)_{t+1}$ is tangle-free, there are at most two non-backtracking paths of length $t+1$ between $o$ and any vertex in $(g,o)_{t+1}$. As a result, we have
    \[ |f(g, o)| \leq 2\|w\|^{t+1}_\infty |(g, o)|_{t+1}. \]
    By definition of $B$, whenever $G$ is $3\ell$-tangle-free, which happens with probability at least $1 - O(n^{12\kappa-1})$, we have
    \[ \langle \chi_j, D_w B^t J \chi_i \rangle = \sum_{x \in [n]} f(G, x). \]
    On the other hand, when $(T, \rho)$ is a hypertree, we have
    \[ f(T, \rho) = \phi_j(\sigma(\rho)) f_{\phi_i, t}(T, \rho), \]
    where $f_{\xi, t}$ was defined in \eqref{eq:def_f_tree}. From Proposition~\ref{prop:concentration:gw_transform}, with $\beta = 1$ and $\alpha = \|w\|_\infty^{t+1}$ we have with probability at least $1 - n^{-2}$,
    \[ \left| \sum_{x\in[n]} f(G, x) - n\sum_{i \in [r]} \pi_i \bar f(i) \right| \lesssim  \log(n)^{5/2} \|w\|^{t+1} R_g^{2t+2}. \]
    It remains to compute the quantity $ \sum_{i \in [r]} \pi_i \bar f(i) = \langle \pi, \bar f \rangle$. From Proposition~\ref{prop:gw_transform:computations} we have
    \[ \bar f_{a} = \phi_j(a) \mu_i^{t+1} \phi_i(a),  \]
    hence
    \begin{align*} 
        \sum_{a \in [r]} \pi_a \bar f(a)& = \mu_i^{t+1} \sum_{a\in [r]} \pi_a \phi_i(a) \phi_j(a) \\
        &= \mu_i^{t+1} \langle \phi_i, \phi_j \rangle_\pi \\
        &= \mu_i^{t+1} \delta_{ij},
    \end{align*}
    having used \eqref{eq:pi_orthogonal} in the last step. This finishes the proof.
\end{proof}

\begin{lemma}\label{lem:hypertree_process_2}
    Let $\ell \leq \kappa \log_R(n)$. For any $i, j \in [r]$, with probability at least $1 - \tilde O(n^{4\kappa-1})$,
    \begin{align}
       & \left|\langle B^t J \chi_i, B^t J \chi_j \rangle - n (\mu_i \mu_j)^t\, \left[\bC_{u}^{(t)}\right]_{ij} \right| \lesssim \log(n)^{7/2} \|w\|_\infty^{2t} R_g^{3t} \sqrt{n} \label{eq:graph_correlation_u} \\
       & \left|\langle (B^\star)^t D_w \chi_i, (B^\star)^t D_w \chi_j \rangle -  n(\mu_i \mu_j)^{t+1} \left[\bC_{v}^{(t)}\right]_{ij} \right| \lesssim \log(n)^{7/2} \|w\|_\infty^{2(t+1)} R_g^{3t+3} \sqrt{n} \label{eq:graph_correlation_v},
    \end{align}
    where $\bC_{u}^{(t)}, \bC_{v}^{(t)}$ are defined in \eqref{eq:def_C_u}, \eqref{eq:def_C_v}, respectively.

    In particular, if $\kappa < 1/4$, we have
    \begin{align*}
        \| B^t J \chi_i \| &\lesssim  \sqrt{n} \max\left(|\mu_i|, \sqrt{\vartheta} \right)^{t} \\
        \| (B^\star)^t D_w \chi_i \| &\lesssim \sqrt{n} \max\left(|\mu_i|, \sqrt{\vartheta} \right)^{t+1}
    \end{align*}
\end{lemma}

\begin{proof}
    We begin with the first inequality. Define the functional
    \begin{align*}
        f(g, o) = &\mathbf{1}_{(g, o)_{t+1} \text{ is tangle-free}}\\
        &\cdot\sum_{e: o\in e} \left(\sum_{\gamma \in \Gamma^t_{o\to e}} w(\gamma) \sum_{x_{t+1} \in e_t \setminus \{x_{t}\}} \phi_i(\sigma(x_{t+1})) \right)\left(\sum_{\gamma \in \Gamma^t_{o\to e}} w(\gamma) \sum_{x_{t+1} \in e_t \setminus \{x_{t}\}} \phi_j(\sigma(x_{t+1})) \right).
    \end{align*}
    By the same argument as in the above lemma, $f$ is $(t+1)$-local and
    \[ f(g, o) \leq 4 \|w\|_{\infty}^{2t} |(g, o)|_{t+1}^2 \]
    Further, by definition,
    \[ \sum_{x \in [n]} f(G, x) = \langle B^t J \chi_i, B^t J \chi_j \rangle. \]
    From Proposition~\ref{prop:concentration:gw_transform}, with $\beta = 2$ and $\alpha = \|w\|_\infty^{2t}$ we have with probability at least $1 - n^{-1}$,
    \[ \left| \sum_{x\in[n]} f(G, x) - n\sum_{i \in [r]} \pi_i \bar f(i) \right| \lesssim  \log(n)^{7/2} \|w\|_\infty^{2t} R_g^{3t} \sqrt{n}. \]
    We now compute $\bar f(i)$. When $(T, \rho)$ is a tree, we have
    \[ f = [f_{\phi_i, t}, f_{\phi_{j}, t}], \]
    where the operator $[\cdot, \cdot]$ was defined in \eqref{eq:gw_functionals:def_one_step_correlation}. By Lemma~\ref{lem:gw_functionals:one_step_result} and Proposition~\ref{prop:gw_transform:computations},
    \begin{align*} 
        \bar f &= \bQ_{\mathbf{1}}\overline{f_{\phi_i, t} f_{\phi_j, t}} + \bQ_{\mathbf{1}}^{(3)} \times_1 \overline{f_{\phi_i, t}} \times_2 \overline{f_{\phi_j, t}} \\
        &= (\mu_i \mu_j)^{t} \bQ_{\mathbf{1}}\left(\phi_i \circ \phi_j + \sum_{s = 0}^{t-1} \frac{\mathbf{K}^s y^{(\phi_i,\phi_j)}}{(\mu_i\mu_j)^{s+1}}\right) + (\mu_i \mu_j)^t \bQ_{\mathbf{1}}^{(3)} \times_1 \phi_i \times_2 \phi_j.
    \end{align*}
    Now, by Assumption~\ref{assumption:degree}, $\pi$ is a left eigenvector of $\bQ^{(k)}$ with eigenvalue $d^{(k)}$, and for any $k \in [K]$,
    \[ \bQ^{(3, k)} \times_3 \pi =  \sum_{m} \pi_m \mathbf{Q}^{(3, k)}_{ijm} = [\Pi \mathbf{Q}^{(k)}]_{ij},  \]
    hence the following holds:
    \begin{align*}
        \pi^\star \bQ_{\mathbf{1}}\phi_i \circ \phi_j &= d \langle \phi_i, \phi_j \rangle_\pi = d \delta_{ij} \\
        \pi^\star \bQ_{\mathbf{1}}^{(3)} \times_1 \phi_i \times_2 \phi_j &= \phi_i^\star \Pi \bQ_{q-2} \phi_j \\
        \pi^\star \bK^s y^{(\phi_i, \phi_j)} &= \vartheta^{s+1} \delta_{ij} + \vartheta^s \phi^\star_i \Pi\bQ_{(q-2)w^2} \phi_j .
    \end{align*}
    Plugging into the above expression and summing the geometric series, we find
    \begin{align*} \langle \pi, \bar f \rangle &= (\mu_i \mu_j)^{t} \left( d\left[\frac{1 - \tau_{ij}^{t+1}}{1 - \tau_{ij}} \delta_{ij} + \tau_{ij} \cdot \frac{1 - \tau_{ij}^{t}}{1 - \tau_{ij}} \phi_i^\star \Pi\bQ_{(q-2)w^2/\vartheta}\, \phi_j \right] + \phi_i^\star\Pi  \bQ_{q-2} \phi_j \right)\\&= (\mu_i \mu_j)^t\, (d C_{ij}^{(t)} +\phi_i^\star \Pi \bQ_{q-2} \phi_j), \end{align*}
    which finishes the first part.

    For the second inequality, we define
    \begin{equation*}
        f(g, o) = \mathbf{1}_{(g, o)_{t+1} \text{ is tangle-free}}\,\sum_{e \ni o} \frac{w(e)^2}{q(e) - 1} \left(\sum_{\gamma \in \Gamma_{o \to e}^t} w(\gamma) \phi_i(\sigma(x_t)) \right) \left(\sum_{\gamma \in \Gamma_{o \to e}^t} w(\gamma) \phi_j(\sigma(x_t)) \right).
    \end{equation*}
    As before, $f$ is $(t+1)$-local and satisfies
    \[ |f(g, o)| \leq 4\|w\|_\infty^{2(t+1)} \, |(g, o)_{t+1}|^2\]
    On the other hand, we have
    \[ \langle (B^\star)^t D_w \chi_i, (B^\star)^t D_w \chi_j \rangle = \sum_{x \to e} \left( \sum_{\gamma: y \to (x \to e)} w(\gamma) \phi_i(\sigma(y)) \right) \left( \sum_{\gamma: y \to (x \to e)} w(\gamma) \phi_j(\sigma(y)) \right), \]
    and each non-backtracking path ending at $x \to e$ can be reversed into a non-backtracking path starting at $x' \to e$ for $x' \neq x$. As a result, we have
    \[\sum_{x \in [n]} f(G, x) = \langle (B^\star)^t D_w \chi_i, (B^\star)^t D_w \chi_j \rangle,\]
    and hence by Proposition~\ref{prop:concentration:gw_transform}
    \[ \left| \langle (B^\star)^t D_w \chi_i, (B^\star)^\ell D_w \chi_j \rangle - n \langle \pi, \bar f \rangle \right| \lesssim \log(n)^{7/2} \|w\|_\infty^{2(t+1)} R_g^{3t+3}. \]
    On the other hand, if we view $f$ as a tree functional we have
    \[ f = [f_{\phi_i, t}, f_{\phi_j, t}]_{w^2/(q-1)}, \]
    and hence from Lemma~\ref{lem:gw_functionals:one_step_result} and Proposition~\ref{prop:gw_transform:computations},
    \begin{align*}
        \bar f &= \bQ_{w^2/(q-1)}\overline{f_{\phi_i, t} f_{\phi_j, t}} + \bQ_{w^2/(q-1)}^{(3)} \times_1 \overline{f_{\phi_i, t}} \times_2 \overline{f_{\phi_j, t}} \\
        &= (\mu_i \mu_j)^{t} \bQ_{w^2/(q-1)}\left(\phi_i \circ \phi_j + \sum_{s = 0}^{t-1} \frac{\mathbf{K}^s y^{(\phi_i,\phi_j)}}{(\mu_i\mu_j)^{s+1}}\right) + (\mu_i \mu_j)^t \bQ_{w^2/(q-1)}^{(3)} \times_1 \phi_i \times_2 \phi_j.
    \end{align*}
    Similarly to the previous computation, this implies
    \[ \langle \pi, \bar f \rangle = (\mu_i \mu_j)^t \left(\left[ \sum_k \left(w^{(k)}\right)^2 d^{(k)} \right] \cdot \Gamma_{ij}^{(t)} + \phi_i^\star \bQ_{w^2(q-2)/(q-1)} \phi_j \right). \]

    Finally, when $t \leq \ell \leq \kappa \log_R(n)$, since $R_g \geq 1$ we have
    \[ \|w\|_\infty^{2t} R_g^{3t} \leq R^{2t} \vartheta^t \leq n^{4 \kappa} \sqrt{n}. \]
    As a result, the RHS of \eqref{eq:graph_correlation_u} is bounded by $\log(n)^{7/2} n^{4\kappa+1/2} \vartheta^t \lesssim n \vartheta^t$ as long as $\kappa < 1/8$. On the other hand, for $i \geq r_0$
    \[ n \mu_i^{2t} \bC_u(i, i) \lesssim n \max(1, \tau_i^{t}) \mu_i^{2t} \lesssim n \max(\mu_i^{2t}, \vartheta^t), \]
    hence the second part of the lemma follows from the triangle inequality.
\end{proof}

\begin{lemma}\label{lem:hypertree_process_3}
Let $\ell \le \kappa \log_R(n)$. With probability at least $1 - \tilde O(n^{4\kappa - 1})$, the following inequality holds for any $i \in [r]$ and $0 \le t \le \ell$:
    \begin{equation}
    \|B^{t+1} J_{\chi_i} - \mu_iB^tJ_{\chi_i}\|^2 \lesssim n \vartheta^{t+1} + \|w\|_\infty^{2t}\log^3(n)n^{6\kappa + 1/2}. \notag
    \end{equation}
    In particular for $\kappa \le 1/8$:
    \begin{equation}\label{eq:hypertree_process_3}
    \|B^{t+1} J_{\chi_i} - \mu_iB^tJ_{\chi_i}\|^2 \lesssim n\vartheta^{t+1}.
    \end{equation}
\end{lemma}

\begin{proof}
    Define the functional
    \begin{align*} f(g, o) = &\mathbf{1}_{(g, o)_{t+1} \text{ is tangle-free}}\, \sum_{e: o\in e} \\
    &\cdot\left( \sum_{\gamma \in \Gamma^{t+1}_{o\to e}} w(\gamma) \sum_{x_{t+1} \in e_t \setminus \{x_{t}\}} \phi_i(\sigma(x_t)) - \mu_i\sum_{\gamma \in \Gamma^t_{o\to e}} w(\gamma) \sum_{x_{t+1} \in e_t \setminus \{x_{t}\}} \phi_i(\sigma(x_t)) \right)^2. 
    \end{align*}
    $f$ is $(t+1)$-local, and by the same argument as before we have
    \[ |f(g, o)| \leq 2 |(g, o)_{t+1}|^2 \|w\|^{2t}_\infty \]
    It is also easy to check that
    \[ \sum_{x \in [n]} f(G, x) = \left\|B^{t+1}J {\chi_i} - \mu_iB^tJ {\chi_i}\right\|^2 \]
    As a result, from Proposition~\ref{prop:concentration:gw_transform} and the triangle inequality, 
    \[ \left\|B^{t+1}J {\chi_i} - \mu_iB^tJ {\chi_i}\right\|^2 \leq \left|n \langle \pi, \bar f \rangle \right| + c_1 \log^{7/2}(n) \|w\|_\infty^{2t}R_g^{3t} \sqrt{n}.\]
        On the other hand, we have $f = \partial F_{\phi_i, t}$, where $F$ was defined in Proposition~\ref{prop:gw_transform:computations}. As a result, using Lemma~\ref{lem:gw_functionals:one_step_result} and the computations from the previous lemma,
    \[ \langle \pi, \bar f \rangle = d \langle \pi, \mathbf K^t y^{(\phi_i, \phi_i)} \rangle = d \vartheta^{t+1} \bC^{(0)}_{ii} \lesssim \vartheta^{t+1}, \]
    having used Lemma~\ref{lem:cov_bound} for the last inequality. On the other hand, by the same reasoning as above, whenever $\kappa < 1/8$,
    \[ \log^{7/2}(n) \|w\|_\infty^{2t+2}R_g^{3t+3} \sqrt{n} \lesssim \log^{7/2}(n) n^{4\kappa +1/2} \vartheta^{t+1} \lesssim n \vartheta^{t+1}, \]
    and the result follows.
\end{proof}

\begin{lemma}\label{lem:hypertree_process_4} Let $\ell \le \kappa \log_R(n)$ and $\kappa \le 1/12$. With probability at least $1 - \tilde O(n^{4\kappa - 1})$, the following inequality holds for any unit vector $w$ that is orthogonal to all $(B^\star )^\ell D_w \chi_i, i \in [r_0]$:
    \begin{equation}\label{eq:hypertree_process_4}
         \left|\left\langle \chi_i, D_w B^t w\right\rangle\right| \lesssim \sqrt{n}\vartheta^{t/2}.
    \end{equation}
\end{lemma}
\begin{proof}
    To begin, we have 
    \begin{equation*}
    \begin{split}
        \left|\left \langle \chi_i, D_w B^tw\right \rangle \right| 
        &= \left|\left \langle (B^\star )^tD_w \chi_i, w\right \rangle\right| \\
        &= \left|\left\langle \mu_i^{t-\ell}(B^\star )^{\ell}D_w\chi_i -(B^\star )^t D_w\chi_i, w\right \rangle\right| \\
        &= \left|\left\langle \sum_{s = t}^{\ell-1} \mu_i^{t-s-1}(B^\star )^{s+1}D_w\chi_i - \mu_i^{t-s}(B^\star )^{s}D_w\chi_i,w \right\rangle\right| \\
        &\le \sum_{s = t}^{\ell-1}\left|\mu_i^{t-s-1}\right|\left|\left\langle (B^\star )^{s+1}D_w\chi_i - \mu_i(B^\star )^{s}D_w\chi_i,w \right\rangle\right|
    \end{split}
    \end{equation*}
    with the second and third equalities following from the orthogonality assumption.
    Next we consider an individual summand. Writing $I = JJ^{-1}$ we have 
    \begin{equation}
    \begin{split}
    \left|\left\langle (B^\star )^{s+1}D_w\chi_i - \mu_i(B^\star )^{s}D_w\chi_i,w \right\rangle\right| 
    &= 
    \left|\left\langle (B^\star )^{s+1}D_w\chi_i - \mu_i(B^\star )^{s}D_w\chi_i, JJ^{-1}w \right\rangle\right| \\
    &= \left|\left\langle J(B^\star )^{s+1}D_w\chi_i - J\mu_i(B^\star )^{s}D_w\chi_i, J^{-1}w \right\rangle\right| \\
    &= \left|\left\langle D_w(B^{s+1}J\chi_i) - D_w(\mu_iB^{s}J\chi_i), J^{-1}w \right\rangle\right| \\
    &\le \|D_w\| \|J^{-1}\| \|B^{s+1}J\chi_i - \mu_iB^{s}J \chi_i\|_2 \\
    &\lesssim \sqrt{n}\vartheta^{(s+1)/2}. \notag
    \end{split}
    \end{equation}
    Here the second equality follows from $J$ being symmetric and the third equality follows from Lemma \ref{lem:BPPB}. For the last inequality we used Lemma \ref{lem:J_eigenvalue}, which in particular implies that the $\|J^{-1}\| = 1$, and Lemma \ref{lem:hypertree_process_3}. 
    From this we get that 
    \begin{align*}
        \sum_{s = t}^{\ell-1}\left|\mu_i^{t-s-1}\right|\left|\left\langle (B^\star )^{s+1}D_w\chi_i - \mu_i(B^\star )^{s}D_w\chi_i,w \right\rangle\right| &\le 
        \sqrt{n}|\mu_i|^t \sum_{s = t}^{\ell-1}\left( \frac{\vartheta}{\mu_i^2} \right)^{(s+1)/2} \\
        &= \sqrt{n}\vartheta^{t}\sum_{s = 1}^{\ell-t}\tau_i^{(s+1)/2}.
    \end{align*}
    Lastly, since $i \in [r_0]$ we have that $\tau_i < 1$. Therefore the summation is at most $\sqrt{\tau_i}/(1-\sqrt{\tau_i})$, which implies the lemma.
\end{proof}
\section{Proof of Proposition~\ref{prop_main}}\label{sec:proof_prop_main}
\subsection{Proof of \eqref{eq:norm_bound1}, \eqref{eq:norm_bound2}, \eqref{eq:norm_bound3} and \eqref{eq:norm_bound4}}
To begin we note that any $r_0 \times r_0$ matrix $M$ satisfies
\begin{equation*}
    \|M\| \le \|M\|_{\mathrm{HS}} \le r_0 \max_{i,j} |M_{ij}|.
\end{equation*}
For \eqref{eq:norm_bound1} we have, for any $i,j \in [r_0]$ that 
\begin{equation}
    \begin{split}
    \left|[U^\star U - \bC_u^{(\ell)}]_{ij}\right| &= \left| \langle u_i,u_j  \rangle - [\bC_u^{(\ell)}]_{ij}\right| \\
    &= (n(\mu_i\mu_j)^{\ell})^{-1}\left|\left\langle B^\ell J{\chi_i},B^\ell J{\chi_j} \right \rangle - n(\mu_i\mu_j)^\ell(\bC_{u}^{(\ell)})_{ij}\right| \\
    & \lesssim \frac{\log(n)^{7/2} \|w\|_\infty^{2\ell} R_g^{3\ell}}{\sqrt{n} (\mu_i\mu_j)^\ell} \notag
\end{split}
\end{equation}
In the second line we used the definition of $U$, in the third line we used the first consequence of Lemma \ref{lem:hypertree_process_2} (with $t = \ell$). Since $i, j \in [r_0]$, we have $\mu_i\mu_j > \vartheta$, and hence
\[ \frac{\log(n)^{7/2} \|w\|_\infty^{2\ell} R_g^{3\ell}}{\sqrt{n} (\mu_i\mu_j)^\ell} \leq \frac{\log(n)^{7/2}R^{2\ell}}{\sqrt{n}} \leq \log(n)^{7/2} n^{2\kappa - 1/2}.\]
When $\kappa < 1/8$, the bound above is $O(n^{-1/4})$, as requested.
In a completely analogous argument for \eqref{eq:norm_bound2} we have, 
\begin{equation}
    \max_{i,j \in [r_0]} \left|[V^{\star} V-\bC_v^{(\ell)}]_{ij}\right| \lesssim \frac{\log(n)^{7/2}R^{2\ell+2}}{\sqrt{n}},\notag
\end{equation}
where the only difference is in the application of the second consequence of Lemma \ref{lem:hypertree_process_2} instead of the first. Since \(R^\ell=n^\kappa\), the preceding bound is \(O(n^{-1/4})\)
for \(\kappa<1/8\). Next, we consider \eqref{eq:norm_bound3}. We have, for any $i,j \in [r_0]$ that 
\begin{equation}
    \begin{split}
    \left|[U^\star V - I_{r_0}]_{ij}\right| &= |\left\langle u_i,v_j\right\rangle - \delta_{ij}| \\
    &= (n\mu_i^{\ell}\mu_j^{\ell+1})^{-1}\left|\langle B^{\ell}J\chi_i,(B^\star)^\ell D_w \chi_j \rangle - n\mu_i^{\ell+1}\mu_j^{\ell}\delta_{ij}\right| \\
    &= (n\mu_i^{\ell}\mu_j^{\ell+1})^{-1}\left|\langle D_w B^{2\ell}J\chi_i, \chi_j \rangle - n\mu_i^{2\ell+1}\delta_{ij}\right| \\
    &\lesssim 
        \frac{\log(n)^{5/2} \|w\|^{2\ell+1}_\infty R_g^{4\ell+2}}{\sqrt{n} \mu_i^\ell \mu_j^{\ell+1}}.
    \end{split}\notag
\end{equation}
To obtain the inequality on the last line we used Lemma \ref{lem:hypertree_process_1} with $t = 2\ell$. As before, the latter is bounded from above by $\log(n)^{5/2} R^{2\ell+1} n^{-1/2} \lesssim n^{-1/4}$ as long as $\kappa < 1/8$.

For \eqref{eq:norm_bound4} we note that $(\Sigma^\ell)_{ij}\mu_i^{\ell}\mu_j^{\ell+1}\delta_{ij} = \mu_i^{3\ell+1}\delta_{ij}$. Then we have 
\begin{equation}
    \begin{split}
        |(V^\star B^\ell U - \Sigma^\ell)_{ij}| &= \left|\left \langle v_i,B^\ell u_j \right \rangle - \Sigma_{ij}^\ell\delta_{ij}\right| \\
        &= (n\mu_i^{\ell+1}\mu_j^{\ell})^{-1}\left|\langle B^{2\ell}J\chi_j,(B^\star)^\ell D_w \chi_i \rangle - n\mu_i^{3\ell+1}\delta_{ij}\right| \\
    &= (n\mu_i^{\ell+1}\mu_j^{\ell})^{-1}\left|\langle D_w B^{3\ell}J\chi_j, \chi_i \rangle - n\mu_i^{3\ell+1}\delta_{ij}\right| \\
    &\lesssim 
        \frac{\log(n)^{5/2} \|w\|^{3\ell+1}_\infty R_g^{6\ell+3}}{\sqrt{n} \mu_i^\ell \mu_j^{\ell+1}}.\notag
    \end{split}
\end{equation}
Multiplying and dividing by $\mu_1^\ell$, we note again that 
\[ \frac{\|w\|^{3\ell+1}_\infty R_g^{6\ell+3}}{\mu_i^\ell \mu_j^{\ell+1}\mu_1^\ell} \leq R^{3\ell+1},\]
and the proof proceeds as above.
\subsection{Proof of \eqref{eq:norm_bound5}}
\eqref{eq:norm_bound5} follows immediately from the following lemma: 
\begin{lemma}
There exists constants $c_1,c_2 > 0$ such that the following holds w.p. at least $1-n^{-c_1}$: Let $\ell \le \kappa \log_{R}(n)$ with $\kappa < \frac{1}{12}$. Let $w$ be any unit vector that is orthogonal to $(B^\star )^\ell D_w \chi_j$ for every $j \in [r]$. Then for all $n$ sufficiently large we have 
\begin{equation}\label{eq:eq_bound5_vec2}
    \|B^\ell w\| \lesssim \log^{c_2}\vartheta^{\ell/2}.
\end{equation}
\end{lemma}
\begin{proof}
    With probability $1-n^{-c_1}$ we have, by Lemma \ref{lem:tangle-free} and Proposition \ref{prop:trace} that $H$ is $\ell$-tangle free and all the norm bounds in Proposition \ref{prop:trace} hold. Therefore $B^\ell w = B^{(\ell)}w$ and we may expand $\|B^{(\ell)}w\|$ according to \ref{lem:Blwnorm}, apply the norm bounds from Proposition \ref{prop:trace} on the terms in the expansion, and simplify the expression (collecting lower order terms and such) to get that
    \begin{equation}\label{eq:bound5_ineq_1}
            \|B^\ell w\| \le \log^c(n)\vartheta^{\ell/2} + \log^c(n)n^{\kappa-1/2} + \frac{\log^c(n)}{n}\sum_{j = 1}^r\sum_{t=1}^{\ell-1}\mu_j\vartheta^{(t-1)/2}\left|\left\langle\chi_j,D_w B^{(\ell-t-1)}w \right\rangle\right|. 
    \end{equation}
We now bound the double summation in \eqref{eq:bound5_ineq_1} by casing on $j$. When $j \in [r_0]$ we apply the estimate \eqref{eq:hypertree_process_4} from Lemma \ref{lem:hypertree_process_4} to get 
\begin{align*}
\frac{\log^c(n)}{n}\sum_{j = 1}^{r_0}\sum_{t=1}^{\ell-1}\mu_j\vartheta^{(t-1)/2}\left|\left\langle\chi_j,D_w B^{(\ell-t-1)}w \right\rangle\right| &\lesssim n^{-1/2}\log^c(n) \vartheta^{(t+1)/2}\vartheta^{(\ell-t-1)/2} \\
&= n^{-1/2}\log^c(n)\vartheta^{\ell/2}. 
\end{align*}
When $j > r_0$ we have $\mu_j^2 \le \vartheta$. Since $\kappa < 1/12$, Lemma \ref{lem:hypertree_process_2} implies that 
\[
    \left| \left\langle \chi_j,D_w B^{(\ell-t-1)}w \right\rangle \right| \le \|(B^\star )^{\ell-t-1}D_w\chi_j\|_2 \lesssim  \sqrt{n}\sqrt{\vartheta}^{(\ell-t-1)},
\]
from which we get 
\[
\frac{\log^c(n)}{n}\sum_{j = r_0+1}^{r}\sum_{t=1}^{\ell-1}\mu_j\vartheta^{(t-1)/2}\left|\left\langle\chi_j,D_w B^{(\ell-t-1)}w \right\rangle\right| \lesssim n^{-1/2}\log^c(n)\vartheta^{\ell/2}
\]
We conclude that the right hand side of \eqref{eq:norm_bound5} is dominated by $\log^c(n)\vartheta^{\ell/2}$. The result follows.
\end{proof}
\subsection{Proof of \eqref{eq:norm_bound6}}
\eqref{eq:norm_bound6} follows immediately from the following lemma: 
\begin{lemma}\label{lem:norm_bound_5}
There exists constants $c_1,c_2 > 0$ such that the following holds w.p. at least $1-n^{-c_1}$: Let $\ell \le \kappa \log_{R}(n)$ with $\kappa < 1/12$. Let $w$ be any unit vector that is orthogonal to $B^\ell J \chi_j$ for every $j \in [r]$. Then for all $n$ sufficiently large we have 
\begin{equation}\label{eq:eq_bound5_vec}
    \|(B^\ell)^\star  w\| \lesssim \log^{c_2}\vartheta^{\ell/2}.
\end{equation}
Proving this requires a different expansion of $B^{(\ell)}w$ to that of \eqref{eq:expansionBl}, albeit completely analogous, to then prove operator norm bounds on the terms appearing in the expansion exactly as was done for \eqref{eq:expansionBl} in Proposition \ref{prop:trace}, and then finally repeat the proof of \eqref{eq:norm_bound5}. As this type of symmetric argument was carried out in \cite{stephan2024sparse} to prove a bound on $\|(B^\star )^\ell w\|$ in their setting, we will omit the proof and refer the interested reader to Section B.8 in \cite{stephan2024sparse} for more details.   
\end{lemma}

\subsection{Proof of \eqref{eq:norm_bound7}}
\begin{proof}
\eqref{eq:norm_bound7} is an immediate consequence of \eqref{eq:Bl} from Proposition \ref{prop:trace} since, from the definition of $R$ and $\ell$, $R^\ell \lesssim n^{1/4}$.
\end{proof}

\paragraph{Acknowledgements}
YZ was partially
supported by the Simons Grant MPS-TSM-00013944.
\bibliographystyle{alpha}
\bibliography{ref.bib}

\appendix

\section{Additional proofs}\label{appendix:proof}
\subsection{Proof of Lemma~\ref{lem:deterministic_Q}}
\begin{proof} In this proof, for simplicity, we drop the dependence on $k$ in $d^{(k)}, q^{(k)}, \mathbf P^{(k)}, \bQ^{(k)}$, and $\bD^{(k)}$.
For any vectors $u,v\in \R^r$, define
\[
\langle u,v\rangle_\pi := \sum_{a=1}^r \pi_a u_a v_a,
\qquad
\|v\|_\pi^2 := \langle v,v\rangle_\pi.
\]
Since $\bD$ is symmetric and $\bQ=\bD\Pi$, we have
$
\langle u,\bQ v\rangle_\pi
= u^\star \Pi D\Pi v
= v^\star \Pi D\Pi u
= \langle Qu,v\rangle_\pi$.

So $Q$ is self-adjoint with respect to $\langle\cdot,\cdot\rangle_\pi$. Therefore, it is enough to prove that for every $v\in \mathbb{R}^r$,
\begin{equation}
\langle v,Qv\rangle_\pi \ge -\frac{d}{q-1}\|v\|_\pi^2.
\label{eq:goal-qf}
\end{equation}
Indeed, from \eqref{eq:goal-qf}, taking $v$ to be an eigenvector of $Q$ immediately gives the eigenvalue bound.

Using $\bQ_{ab}=\bD_{ab}\pi_b$, we get
$
\langle v,\bQ v\rangle_\pi
= \sum_{a,b=1}^r \pi_a v_a \bQ_{ab} v_b
= \sum_{a,b=1}^r \pi_a\pi_b \bD_{ab} v_a v_b$.
Now substitute the definition of $\bD_{ab}$:
\[
\langle v,\bQ v\rangle_\pi
= \sum_{a,b=1}^r \sum_{\underline k\in [r]^{q-2}}
 p_{a,b,\underline k}\,\pi_a\pi_b\Bigl(\prod_{\ell\in \underline k}\pi_\ell\Bigr) v_a v_b.
\]
If we relabel $(a,b,\underline k)$ as a full $q$-tuple $(i_1,\dots,i_q)$, then
\begin{equation}
\langle v,\bQ v\rangle_\pi
= \sum_{i_1,\dots,i_q\in [r]}
 p_{i_1,\dots,i_q}\,\pi_{i_1}\cdots \pi_{i_q}\, v_{i_1}v_{i_2}.
\label{eq:qf-expanded}
\end{equation}
For convenience, write
$
c_{i_1,\dots,i_q}:=p_{i_1,\dots,i_q}\,\pi_{i_1}\cdots \pi_{i_q}$.
Because $\mathbf P$ is symmetric, the coefficients $c_{i_1,\dots,i_q}$ are symmetric in the $q$ indices. For each ordered pair $(s,t)$ with $s\neq t$, define
\[
S_{st}:=\sum_{i_1,\dots,i_q\in [r]} c_{i_1,\dots,i_q}\, v_{i_s}v_{i_t}.
\]
By symmetry,  $S_{st}=S_{12}$ for every $s\neq t$. 
Since there are exactly $q(q-1)$ ordered pairs $(s,t)\in [q]^2$ with $s\neq t$, averaging gives
\[
S_{12}=\frac{1}{q(q-1)}\sum_{s\neq t} S_{st}.
\]
Using \eqref{eq:qf-expanded}, this becomes
\begin{equation}
q(q-1)\langle v,Qv\rangle_\pi
= \sum_{i_1,\dots,i_q\in [r]} c_{i_1,\dots,i_q}\sum_{s\neq t} v_{i_s}v_{i_t}.
\label{eq:symmetrized}
\end{equation}
For any real numbers $z_1,\dots,z_q$, $\sum_{s\neq t} z_s z_t
= \Bigl(\sum_{s=1}^q z_s\Bigr)^2 - \sum_{s=1}^q z_s^2
\ge -\sum_{s=1}^q z_s^2$.
Apply this with $z_s=v_{i_s}$. Since every coefficient $c_{i_1,\dots,i_q}$ is nonnegative, \eqref{eq:symmetrized} implies
\begin{equation}
q(q-1)\langle v,Qv\rangle_\pi
\ge -\sum_{i_1,\dots,i_q\in [r]} c_{i_1,\dots,i_q}\sum_{s=1}^q v_{i_s}^2.
\label{eq:after-ineq}
\end{equation}
Because the coefficients $c_{i_1,\dots,i_q}$ are symmetric, 
\[
\sum_{i_1,\dots,i_q} c_{i_1,\dots,i_q}\sum_{s=1}^q v_{i_s}^2
= q\sum_{i_1,\dots,i_q} c_{i_1,\dots,i_q}v_{i_1}^2.
\]
Now separate the first index:
\[
q\sum_{i_1,\dots,i_q} c_{i_1,\dots,i_q}v_{i_1}^2
= q\sum_{i_1=1}^r \pi_{i_1}v_{i_1}^2
\sum_{i_2,\dots,i_q\in [r]} p_{i_1,\dots,i_q}\prod_{u=2}^q \pi_{i_u}.
\]
By Assumption~\eqref{assumption:degree}, the inner sum equals $d$ for every $i_1$. Therefore,
\begin{equation}
\sum_{i_1,\dots,i_q} c_{i_1,\dots,i_q}\sum_{s=1}^q v_{i_s}^2
= qd\sum_{i_1=1}^r \pi_{i_1}v_{i_1}^2
= qd\,\|v\|_\pi^2.
\label{eq:rightside}
\end{equation}
Combining \eqref{eq:after-ineq} with \eqref{eq:rightside}, we obtain
$
q(q-1)\langle v,Qv\rangle_\pi \ge -qd\,\|v\|_\pi^2.
$
This proves \eqref{eq:goal-qf}. Hence, the conclusion holds.
\end{proof}
\subsection{Proof of Lemma~\ref{lem:J_eigenvalue} }
    \begin{proof}
    Note that $J$ acts independently on each hyperedge $e\in H$. 
Restricting $J$ to the oriented hyperedges $\{x\to e : x\in e\}$ of a fixed hyperedge $e$ of size $q_e\ge 2$, its matrix representation is
\[
J_e = \mathbf 1\mathbf 1^\top - I_{q_e},
\]
where $\mathbf 1\in\mathbb R^{q_e}$ is the all-ones vector.
The eigenvalues of $J_e$ are $q_e-1$ with multiplicity $1$ (corresponding to $\mathbf 1$) and $-1$ with multiplicity $q_e-1$.
Since $q_e\ge 2$, all eigenvalues are nonzero, and hence $J_e$ is invertible.
Because $J$ is block-diagonal with blocks $\{J_e\}_{e\in H}$, it follows that $J$ is invertible and its eigenvalues are given by $-1$ and $q^{(k)}-1, 1\leq k\leq K$.
\end{proof}

\subsection{Proof of Lemma~\ref{lem:BPPB}}
\begin{proof}
When $k=0$, we can check by definition that 
\begin{align}\label{eq:JDw}
JD_w =D_w J.
\end{align} When $k=1$, we can check that 
  \begin{align}
      (D_w BJ)_{(x\to e),(y\to f)}=w_ew_f |e\cap f\setminus \{x,y\}|\1 \{e\not=f\}= (JB^\star D_w)_{(x\to e),(y\to f)}, \notag
  \end{align}
  hence $D_w BJ=JB^\star D_w$. Since $J$ is invertible, from \eqref{eq:JDw},
  \begin{align}
      (J^{-1}D_w) B=B^\star (J^{-1}D_w).\notag
  \end{align}
  Then for any $k\geq 2$,
  \begin{align}
      J^{-1} D_w B^k=J^{-1}D_w B B^{k-1}=B^\star J^{-1}D_w B^{k-1}=\cdots=(B^\star )^k J^{-1} D_w=(B^\star )^k D_w J^{-1},\notag
  \end{align}
  which implies $D_w B^kJ=J(B^\star )^k D_w$.
\end{proof}

\subsection{Proof of Lemma~\ref{lem:cov_bound}}\label{sec:proof_lem_cov_bound}

Throughout this section, we shall use the following consequence of Lemma~\ref{lem:deterministic_Q} and the Perron-Frobenius theorem: if $a$ is a positive vector, the matrix $\bQ_a$ satisfies
\begin{equation}\label{eq:app:bounds_Q_weighted}
    - d_{a/(q-1)} \cdot I_{r} \preceq  \bQ_a \preceq d_a I_r.
\end{equation}

We recall the definition of $\bC^{(\ell)}$:
\[ \bC^{(\ell)}_{ij} = \frac{1 - \tau_{ij}^{\ell+1}}{1 - \tau_{ij}} \delta_{ij} + \tau_{ij} \frac{1 - \tau_{ij}^{\ell+1}}{1 - \tau_{ij}} \phi_i^\star \Pi \bQ_{(q-2)w^2/\vartheta} \phi_j = \delta_{ij} + \sum_{t = 1}^\ell \underbrace{\tau_{ij}^t \left( \delta_{ij} + \phi_i^\star \Pi \bQ_{(q-2)w^2/\vartheta} \phi_j \right)}_{:= \tilde\bC^{(t)}_{ij}} \]
It thus suffices to show that $\tilde\bC^{(t)} \succeq 0$ to imply that $\bC^{(\ell)} \succeq I_{r_0}$. For $i \in [r_0]$, we define
\[ \eta_i = \frac{\sqrt{\vartheta}}{\mu_i}, \]
so that $\tau_{ij} = \eta_i\eta_j$. Given $\alpha \in \R^{r_0}$, recalling that $\phi_i^\star \Pi \phi_j = \delta_{ij}$,
\begin{align*}
    \alpha^\star \tilde\bC^{(t)} \alpha &= \vartheta^{-1} \left( \sum_i \alpha_i \eta_i^t \phi_i \Pi^{1/2} \right) \left( \vartheta I_{r_0} + \Pi^{1/2}\bQ_{(q-2)w^2}\Pi^{-1/2} \right) \left( \sum_i \alpha_i \eta_i^t \phi_i \Pi^{1/2} \right) \\
    &= \vartheta^{-1}\beta_t^\star \left( \sum_{k=1}^K \left(w^{(k)}\right)^2 (q^{(k)} - 1) \left[ d^{(k)} I_{r_0} + (q^{(k)}-2)\Pi^{1/2} D^{(k)} \Pi^{1/2} \right]  \right)\beta_t,
\end{align*}
having defined $\beta_t = \sum_i \alpha_i \eta_i^t \phi_i \Pi^{1/2}$. The matrix $\Pi^{1/2} \bD^{(k)} \Pi^{1/2}$ is a symmetric matrix with the same eigenvalues as $\bQ^{(k)}$, and from \eqref{eq:app:bounds_Q_weighted} we have
\[ \vartheta I_{r_0} + \bQ_{w^2(q-2)} \succeq (\vartheta - d_{w^2(q-2)/(d-1)}) \cdot I_{r_0} = d_{w^2/(q-1)} \cdot I_{r_0} \]
This implies that $\alpha^\star \tilde\bC^{(t)} \alpha \geq 0$ and the lower bound $\bC^{(\ell)} \succeq I_{r_0}$ holds.
For the upper bound, we have $\bQ_{w^2(q-2)} \preceq d_{w^2(q-2)} \cdot I_{r_0}$, hence
\begin{align*}
    \alpha^\star \tilde\bC^{(t)} \alpha &\leq \vartheta^{-1} \beta_t^\star \left(  \sum_{k=1}^K \left(w^{(k)}\right)^2 (q^{(k)} - 1)^2 d^{(k)} \right)\beta_t \\
    &\leq  (q_{\max}-1) \|\beta_t\|^2 \\
    &\leq (q_{\max}-1) \sum_{i=1}^{r_0} \tau_i^{t} \alpha_i^2 \\
    &\leq (q_{\max} -1) \tau_{{r_0}}^t \|\alpha\|^2.
\end{align*}
Summing those bounds yields
\[ \bC^{(\ell)} \preceq 1 + \frac{\tau_{r_0}(q_{\max}-1)}{1 - \tau_{r_0}}, \]
as requested.

We now move to $\bC_u^{(t)}$ and $\bC_v^{(t)}$. From eq. \eqref{eq:app:bounds_Q_weighted}, we have
\begin{align*} 
-d_{(q-2)/(q-1)} \cdot I_{r_0} \preceq \bQ_{(q-2)} &\preceq d_{(q-2)} \cdot I_{r_0}  \\
-d_{w^2(q-2)/(q-1)^2} \cdot I_{r_0}\preceq \bQ_{w^2(q-2)/(q-1)} &\preceq d_{w^2(q-2)/(q-1)} \cdot I_{r_0}
\end{align*}
Summing with the previous results on $\bC^{(t)}$, we have
\begin{align*} 
d_{1/(q-1)} \cdot I_{r_0} &\preceq \bC^{(t)}_u \preceq \left[d \left(1 + \frac{\tau_{r_0}(q_{\max}-1)}{1 - \tau_{r_0}}\right) + d_{q-2} \right] I_{r_0} \\
\vartheta^{-1} d_{w^2/(q-1)^2} \cdot  I_{r_0}&\preceq \bC^{(t)}_v
\preceq \vartheta^{-1} \left[ d_{w^2/(q-1)}  \left(1 + \frac{\tau_{r_0}(q_{\max}-1)}{1 - \tau_{r_0}}\right) + d_{w^2(q-2)/(q-1)} \right] \cdot I_{r_0} 
\end{align*}

\subsection{Proof of Lemma~\ref{lem:reduced_eigen_computations}}

The first equality comes from the identity
\[ \langle SD_w B^\ell J\chi_i, \tilde \phi_j \rangle = \langle D_w B^\ell J\chi_i, S^\top \tilde\phi_j \rangle =  \langle D_w B^\ell J\chi_i, \chi_j \rangle\]
and the application of Lemma \ref{lem:hypertree_process_1}.

For the second, define
\[ f(g, o) = \mathbf{1}_{(g, o)_\ell \text{ is tangle-free}}\, \left(\sum_{\gamma \in \Gamma^\ell_o} w(\gamma) \sum_{x_{\ell+1} \in e_\ell \setminus \{x_\ell\}} \phi_i(\sigma(x_{\ell+1})) \right)^2. \]
The functional $f$ is $(t+1)$-local, and similarly to Lemma~\ref{lem:hypertree_process_2} one can check that whenever $G$ is tangle-free,
\begin{align} 
    \sum_{x \in [n]} f(G, x) &= \|SD_w B^\ell J\chi_i\|^2, \label{eq:square_f}\\
    |f(g, o)| &\leq 2 \|w\|_\infty^{2t+2} \left|(g, o)_{t+1}\right|^2.\notag
\end{align}
On the other hand, we have
\[ f(T, \rho) = \overline{f_{\phi_i, t+1}^2} = \mu_i^{2(t+1)} \bC_{ii}^{(\ell)}, \]
from the computations in Lemma~\ref{lem:hypertree_process_2}. The rest of the proof proceeds as in  \eqref{eq:norm_bound1}, noticing that
\[ \bC_{ii}^{(\ell)} = \frac{1 - \tau_{i}^{\ell+1}}{1 - \tau_{i}} + \tau_{i} \cdot \frac{1 - \tau_{i}^{\ell}}{1 - \tau_{i}} \phi_i^\star \Pi \bQ_{(q-2)w^{2}/\vartheta} \, \phi_i = \gamma_i + O(n^{-c}). \]

\subsection{Proof of Proposition~\ref{thm:IharaBass}}
\begin{proof}[Proof of Proposition~\ref{thm:IharaBass}]
From \cite[Lemma 2]{stephan2024sparse}, $J_k$ satisfies $J_k=J_k^\star $ and when restricted on $E_k$:
\begin{align}\label{eq:J_k_identity}
J_k^2=(q^{(k)}-2)J_k+(q^{(k)}-1) I_{E_k}.
\end{align}
Set
\[
S_w:=\sum_{k=1}^K w^{(k)} S_k,\qquad
J_w:=\sum_{k=1}^K w^{(k)} J_k.
\]
 We have from \eqref{eq:T_S_P_relations}, $B=T^\star S_w-J_w$, hence
\[
\lambda I-B=(\lambda I+J_w)-T^\star S_w.
\]
Whenever $\lambda I+J_w$ is invertible, the matrix determinant formula gives
\begin{equation}\label{eq:mdl-step}
\det(\lambda I-B)
=
\det(\lambda I+J_w)\;
\det\!\Big(I-S_w(\lambda I+J_w)^{-1}T^\star \Big).
\end{equation}

From Lemma~\ref{lem:J_eigenvalue}, $J_w$ has eigenvalues $-w^{(k)}$ with multiplicity $(q^{(k)}-1)m_k$ and $w^{(k)}(q^{(k)}-1)$ with multiplicity $m_k$, hence
\begin{align}\label{eq:charJ}
\det(\lambda I+J_w)
=
\prod_{k=1}^K
(\lambda-w^{(k)})^{(q^{(k)}-1)m_k}\;
(\lambda+w^{(k)}(q^{(k)}-1))^{m_k}.
\end{align}
Since $(S_k,T_k,J_k)$ have disjoint supports for different $k$,
\begin{align*}
S_w(\lambda I+J_w)^{-1}T^\star 
=\sum_{k=1}^K w^{(k)} S_k(\lambda I+w^{(k)}J_k)^{-1}T_k^\star .
\end{align*}
Define 
\[\Delta_k(\lambda):=(\lambda-w^{(k)})\bigl(\lambda+w^{(k)}(q^{(k)}-1)\bigr).\]
Using \eqref{eq:J_k_identity}, one checks
\[
(\lambda I+w^{(k)}J_k)^{-1}
=
\frac{(\lambda+w^{(k)}(q^{(k)}-2))I-w^{(k)}J_k}{\Delta_k(\lambda)}.
\]
With \eqref{eq:T_S_P_relations}, this yields
\[
w^{(k)} S_k(\lambda I+w^{(k)}J_k)^{-1}T_k^\star 
=
\frac{w^{(k)}\lambda}{\Delta_k(\lambda)}A_k
-\frac{{w^{(k)}}^2(q^{(k)}-1)}{\Delta_k(\lambda)}D_k.
\]
Therefore
\begin{align}\label{eq:matrix_det}
\det\!\Big(I-S_w(\lambda I+J_w)^{-1}T^\star \Big)
=
\det\!\Big(
I-\sum_{k=1}^K \frac{w^{(k)}\lambda}{\Delta_k(\lambda)}A_k
+\sum_{k=1}^K \frac{{w^{(k)}}^2(q^{(k)}-1)}{\Delta_k(\lambda)}D_k
\Big).
\end{align}

Rearrange the eigenvector $\tilde v\in \R^{2Kn}$ of $\tilde B$  as $(x,y)$ with
\[
x=(v_1^{\leftarrow},\dots,v_K^{\leftarrow})\in\mathbb R^{Kn},
\qquad
y=(v_1^{\rightarrow},\dots,v_K^{\rightarrow})\in\mathbb R^{Kn}.
\]
With respect to this decomposition, the matrix
$\lambda I-\tilde B$ has the block form
\[
\lambda I-\tilde B
=
\begin{pmatrix}
\lambda I & -M\\
-N & \lambda I-P
\end{pmatrix},
\]
where $M, N,P\in \R^{Kn\times Kn}$  with a $K\times K$ block structure satisfying
\[
M_{k\ell}=w_\ell(D_k-\delta_{k\ell}I),\qquad
N_{k\ell}=w_\ell(q^{(k)}-1)\delta_{k\ell}I,
\qquad
P_{k\ell}=w_\ell(A_k-(q^{(k)}-2)\delta_{k\ell}I).
\]
For $\lambda\neq 0$, take the Schur complement of $\lambda I$:
\begin{equation}\label{eq:schur-short}
\det(\lambda I-\tilde B)
=\lambda^{Kn}\det\!\Big(\lambda I-P-\lambda^{-1}NM\Big).
\end{equation}
Since $N$ is block-diagonal, $(NM)_{k\ell}=w^{(k)}(q^{(k)}-1)w_\ell(D_k-\delta_{k\ell}I)$.
A direct simplification gives, for each $(k,\ell)$-block,
\[
\big(\lambda I-P-\lambda^{-1}NM\big)_{k\ell}
=
\frac{\Delta_k(\lambda)}{\lambda}\,\delta_{k\ell}I
-\;w_\ell\Big(A_k+\frac{w^{(k)}(q^{(k)}-1)}{\lambda}D_k\Big).
\]
Now factor $(\Delta_k(\lambda)/\lambda)I$ from the $k$-th block row to obtain
\[
\lambda I-P-\lambda^{-1}NM
=
\operatorname{diag}\!\Big(\tfrac{\Delta_1(\lambda)}{\lambda}I,\dots,
\tfrac{\Delta_K(\lambda)}{\lambda}I\Big)\,
\Big(I-UV\Big),
\]
where $U\in\mathbb R^{Kn\times n}$ stacks the $n\times n$ blocks
$A_k+\frac{w^{(k)}(q^{(k)}-1)}{\lambda}D_k$, and
$V\in\mathbb R^{n\times Kn}$ has blocks
$V_\ell=\frac{w_\ell\lambda}{\Delta_\ell(\lambda)}I$.
Hence
\[
\det\!\Big(\lambda I-P-\lambda^{-1}NM\Big)
=
\prod_{k=1}^K\Big(\tfrac{\Delta_k(\lambda)}{\lambda}\Big)^n
\det(I-UV).
\]
By the matrix determinant formula, $\det(I-UV)=\det(I-VU)$, and we have
\[
VU
=
\sum_{k=1}^K \frac{w^{(k)}\lambda}{\Delta_k(\lambda)}A_k
+\sum_{k=1}^K \frac{{w^{(k)}}^2(q^{(k)}-1)}{\Delta_k(\lambda)}D_k.
\]
Plugging this into \eqref{eq:schur-short}  yields
\begin{equation}\label{eq:det-tildeB-short}
\det(\lambda I-\tilde B)
=
\Bigl(\prod_{k=1}^K \Delta_k(\lambda)^n\Bigr)
\det \Bigl(
I_n
-\sum_{k=1}^K \frac{w^{(k)}\lambda}{\Delta_k(\lambda)}A_k
+\sum_{k=1}^K \frac{{w^{(k)}}^2(q^{(k)}-1)}{\Delta_k(\lambda)} D_k
\Bigr).
\end{equation} Since both sides are rational functions in $\lambda$, the identity holds for all $\lambda \not\in \{w^{(k)},-w^{(k)}(q^{(k)}-1)\}$.
Combining \eqref{eq:det-tildeB-short} with \eqref{eq:mdl-step}, \eqref{eq:charJ}, and  \eqref{eq:matrix_det}  gives \eqref{eq:ihara-bass-lambda}.
\end{proof}

\subsection{Proof of Proposition~\ref{prop:Bethe_hessian_relation}}
\begin{proof}

The following matrix identities are  easy to check:
\begin{equation}\label{eq:T_S_P_relations}
\begin{aligned}
    S_kS_\ell^\star  &= D_k \delta_{kl},  \qquad  T_kT_\ell^\star =\left((q^{(k)}-2)A_k+(q^{(k)}-1)D_k\right)\delta_{kl}, \qquad S_kT_\ell^\star &= A_k \delta_{kl}, \\
    S_kJ_\ell &= T_k \delta_{kl}, \qquad J_k J_\ell= \left((q^{(k)}-2)J_k + (q^{(k)} - 1)I\right) \delta_{kl}, &\\
    \qquad B &= \left( \sum_k T_k \right)^\star \left(\sum_k w^{(k)} S_k\right)  - \sum_k w^{(k)} J_k . 
\end{aligned}
\end{equation}
Equation \eqref{eq:T_S_P_relations} emphasizes the deep connections between $B$ and the matrices $A_k$ and $D_k$.

    We begin from the relation $Bv=\lambda v$, which gives
    \begin{align} \label{eq:Bv}
    \left( \sum_k T_k \right)^\star \left(\sum_k w^{(k)} S_k v\right)  - \sum_k w^{(k)} J_k v = \lambda v. 
    \end{align}
    Multiplying on both sides by $S_k J_k^{-1}$, using the relations \eqref{eq:T_S_P_relations}, we find 
    \begin{align*} 
        \lambda v_k^{\leftarrow} &= S_k J_k^{-1} T_k^\star  \left( \sum_{\ell} w_\ell S_\ell \right) v - w^{(k)} S_k J_k^{-1} J_k v \\
        &= S_k S_k^\star  \left( \sum_{\ell} w_\ell S_\ell \right) v - w^{(k)}  v_k^{\rightarrow} \\
        &= D_k \sum_{\ell} w_\ell v_\ell^\rightarrow - w^{(k)} v_k^{\rightarrow}.
    \end{align*}
    Now, multiplying by $S_k$ in \eqref{eq:Bv} instead:
    \begin{align*}
        \lambda v_k^{\rightarrow} &= S_k T_k^\star \left( \sum_\ell w_\ell S_\ell \right) v - w^{(k)} S_k J_k  v \\
        & = A_k \left(\sum_{\ell} w_\ell S_\ell v \right) - w^{(k)} S_k \left( (q^{(k)} - 2)I + (q^{(k)} - 1)J_k^{-1} \right) v \\
        &= A_k \left( \sum_{\ell} w_\ell v_\ell^{\rightarrow} \right) - w^{(k)} (q^{(k)} - 2) v_k^{\rightarrow} - w^{(k)}(q^{(k)} - 1) v_k^{\leftarrow}.
    \end{align*}
    Those relations are equivalent to $\tilde B \tilde v = \lambda \tilde v$, which proves the first claim.

    For the second claim, we note that from the equations above,
    \begin{align*}
          \lambda v_k^{\leftarrow} &=D_ky-w^{(k)}v_k^{\rightarrow} \\
            \lambda v_k^{\rightarrow} &=A_k y- w^{(k)} (q^{(k)} - 2) v_k^{\rightarrow} - w^{(k)}(q^{(k)} - 1) v_k^{\leftarrow}.
    \end{align*}
    Eliminating $v_k^{\leftarrow}$ with the first equation, we obtain 
    \begin{align}
        v_k^{\rightarrow}= \Delta_k(\lambda) ^{-1} \left(\lambda A_ky-w^{(k)}(q^{(k)}-1) D_ky\right). \notag
    \end{align}
    Multiplying by $w^{(k)}$ and summing over $1\leq k\leq K$ gives exactly $H(\lambda) y=0$.
\end{proof}

\subsection{Proof of Theorem~\ref{thm:weakreconstruction}}

Note that after the weighting, since some $w^{(k)}$ might be negative, the first eigenvector of $\bQ$ is not necessarily $\1\in \R^k$. 
 Step 2 in Algorithm~\ref{alg:provable_algorithm} is used to detect the first non-trivial eigenspace of $\bQ$ orthogonal to the trivial eigenspace $\1$.

\begin{itemize}
    \item  Case 1: If  $|\langle \tilde y_1,\1\rangle |>\frac{1}{\log n}$ from part (1) in Theorem~\ref{th:main_reduced},  with high probability $\tilde y_1$ has a nonzero overlap with $\1$.
    \item Case 2:
  If  $|\langle \tilde y_1,\1\rangle |\leq \frac{1}{\log n}$, then  $\tilde y_1$ has a nonzero overlap with $\tilde \phi_1\not=\1$ and $  \phi_1$ is an eigenvector of $\bQ$ orthogonal to $\1$. 
\end{itemize}
According to  Step 2 in Algorithm~\ref{alg:provable_algorithm}, $u$ is an eigenvector of $\tilde B$ associated with the first  subspace of $\bQ$ orthogonal to $\1$. We assume Case 1 holds; the analysis for Case 2 is verbatim.


Equation~\eqref{eq:approx_y_SDu} in the  proof of Theorem \ref{th:main_reduced}  and \eqref{eq:def_u_v} imply that with high probability, 
\[ \|u -\bar u_2\| = O(n^{1/2-c}), \quad \text{with} \quad \bar u_2 = \frac{SD_wB^\ell J\chi_2}{\mu_2^{\ell+1}\sqrt{\gamma_2}},\]
 where $\gamma_2$ is defined in \eqref{eq:def_gamma_i} and $u$ is defined in Algorithm \ref{alg:provable_algorithm}.
Similar to  \eqref{eq:square_f},  when $G$ is $\ell$-tangle-free, which happens with  probability $1-O(n^{-c})$, \[\bar u_2(x) =  \frac{1}{\mu_2^{\ell+1}\sqrt{\gamma_2}} f(G, x),\]
where
\[ f(g, o) = \mathbf{1}_{(g, o)_\ell \text{ is tangle-free}}\, \sum_{\gamma \in \Gamma^\ell_o} w(\gamma) \sum_{x_{\ell+1} \in e_\ell \setminus \{x_\ell\}} \phi_2(\sigma(x_{\ell+1})), \]
where  $\Gamma_o^{\ell}$ is the set of all non-backtracking path of length $\ell$ starting at $o$.
The vector $\bar u_2$ then satisfies the weak convergence estimates of Proposition \ref{prop:concentration:gw_transform}. From \cite{stephan2024sparse}, the following holds:
\begin{lemma}[Lemma 31 in \cite{stephan2024non}]
    For any $i \in [r]$, there exists a random variable $X_i$ such that for any $K \geq 0$ that is a continuity point of $X_i$, with probability at least $1-n^{-c}$,
    \[\frac1n\sum_{x\in[n]} \ind_{\sigma(x)=i}\,  u(x) \ind_{| u(x)| \leq K} = \pi_i \,\E*{X_i \ind_{|X_i| \leq K}}+O(n^{-c}).\]
    Furthermore, we have
    \[ \sum_{i \in [r]} \pi_i \E{X_i} = 0, \quad  \sum_{i \in [r]} \pi_i \E{X_i^2} = 1 \quand \sum_{i \in [r]} \pi_i \E{X_i}^2 = \gamma_2^{-1}. \]
\end{lemma}

By the concentration bound in Proposition \ref{prop:concentration:gw_transform} and Algorithm~\ref{alg:provable_algorithm}, for all $i$, with probability at least $1-n^{-c}$, 
\begin{align} 
\frac1n \sum_{x \in [n]} \ind_{\sigma(x) = i} \ind_{x \in V^+} =\pi_i \left( \frac12 + \frac{\E*{X_i \ind_{|X_i| \leq K}}}{2K} \right) +O(n^{-c}). \notag
\end{align}
The rest of the proof follows from \cite[Proof of Theorem 3]{stephan2024sparse}. Whenever Case 2 holds, we repeat the same proof but with $u_1$ instead of $u_2$, which finishes the proof of Theorem~\ref{thm:weakreconstruction}.

\end{document}